\documentclass[11pt]{article}

\usepackage{fullpage}
\usepackage{natbib}
\usepackage{algorithm}
\usepackage{algorithmic}

\usepackage{comment} 
\usepackage{lipsum} 
\usepackage{amsmath}
\usepackage[utf8]{inputenc}
\usepackage[T2A]{fontenc}
\usepackage{amssymb}  
\usepackage{algorithm}
\usepackage{algorithmic}
\usepackage{graphicx}
\usepackage[dvipsnames]{xcolor}
\usepackage{nicefrac}
\usepackage{bm}

\usepackage{threeparttable}
\usepackage{makecell}
\newcommand{\makecellnew}[1]{{\renewcommand{\arraystretch}{0.8}\begin{tabular}{c} #1 \end{tabular}}}

\usepackage{multirow}
\usepackage{colortbl}
\definecolor{bgcolor}{rgb}{0.8,1,1}
\definecolor{bgcolor2}{rgb}{0.8,1,0.8}

\usepackage{graphicx} 
\graphicspath{{../}}

\usepackage{xspace}
\newcommand{\algname}[1]{{\sf\green#1}\xspace}
\newcommand{\dataname}[1]{{\tt\color{blue}#1}\xspace}

\newcommand{\eqdef}{\; { := }\;}

\newcommand{\R}{\mathbb{R}}

\newcommand{\N}{\mathbb{N}}

\def\<#1,#2>{\langle #1,#2\rangle}

\def\<{\left\langle}
\def\>{\right\rangle}
\def\[{\left[}
\def\]{\right]}
\def\({\left(}
\def\){\right)}

\usepackage{tcolorbox}
\usepackage{pifont}
\definecolor{mydarkgreen}{RGB}{39,130,67}
\definecolor{mydarkred}{RGB}{192,47,25}
\newcommand{\green}{\color{mydarkgreen}}

\newcommand{\cO}{\mathcal{O}}
\newcommand{\mA}{\mathbf{A}}
\newcommand{\mB}{\mathbf{B}}
\newcommand{\mC}{\mathbf{C}}
\newcommand{\mH}{\mathbf{H}}
\newcommand{\mE}{\mathbf{E}}
\newcommand{\mL}{\mathbf{L}}

\newcommand{\mQ}{\mathbf{Q}}
\newcommand{\mI}{\mathbf{I}}

\newcommand{\mV}{\mathbf{V}}

\newcommand{\mS}{\mathbf{S}}
\newcommand{\mX}{\mathbf{X}}

\newcommand{\cC}{{\mathcal{C}}}

\newcommand{\cQ}{{\mathcal{Q}}}

\usepackage{amsmath,amsfonts,amssymb,amsthm,array}

\usepackage{mdframed} 
\usepackage{thmtools}
\usepackage{textcomp}

\declaretheorem[within=section]{definition}
\declaretheorem[sibling=definition]{theorem}
\declaretheorem[sibling=definition]{proposition}
\declaretheorem[sibling=definition]{assumption}

\declaretheorem[sibling=definition]{lemma}
\declaretheorem[sibling=definition]{example}

\usepackage[colorinlistoftodos,bordercolor=orange,backgroundcolor=orange!20,linecolor=orange,textsize=scriptsize]{todonotes}

\usepackage{microtype}
\usepackage{subfigure}
\usepackage{booktabs} 
\usepackage{grffile}
\usepackage{hyperref}

\newcommand{\squeeze}{}

\title{\bf Basis Matters: Better Communication-Efficient \\ Second Order Methods for Federated Learning}

\author{Xun Qian\thanks{King Abdullah University of Science and Technology, Thuwal, Saudi Arabia.} \and Rustem Islamov\thanks{King Abdullah University of Science and Technology (KAUST), Thuwal, Saudi Arabia, and Institut Polytechnique de Paris (IP Paris), Palaiseau, France. This research was conducted while this author was an intern at KAUST and a master student at IP Paris.} \and Mher Safaryan\thanks{King Abdullah University of Science and Technology, Thuwal, Saudi Arabia.}   \and Peter Richt\'{a}rik\thanks{King Abdullah University of Science and Technology, Thuwal, Saudi Arabia.}}
\date{September 23, 2021 \\ (visually revised on October 27, 2021)}

\begin{document}

\maketitle

\begin{abstract}
 Recent advances in distributed optimization have shown that Newton-type methods with proper communication compression mechanisms can guarantee fast local rates and low communication cost compared to first order methods.
 We discover that the communication cost of these methods can be further reduced, sometimes dramatically so, with a surprisingly simple trick: {\em Basis Learn (BL)}. The idea is to transform the usual representation of the local Hessians via a change of basis in the space of matrices and apply compression tools to the new representation.
 To demonstrate the potential of using custom bases, we design a new Newton-type method (\algname{BL1}), which reduces communication cost via both {\em BL} technique and bidirectional compression mechanism. Furthermore, we present two alternative extensions (\algname{BL2} and \algname{BL3}) to partial participation to accommodate federated learning applications.
 We prove local linear and superlinear rates independent of the condition number.
 Finally, we support our claims with numerical experiments by comparing several first and second~order~methods.
\end{abstract}

\tableofcontents
 
\clearpage 
\section{Introduction}

We consider federated optimization problems of the form
\begin{equation}\label{problem-erm}
	\squeeze
	\min\limits_{x\in \mathbb{R}^d} f(x) \eqdef \frac{1}{n} \sum\limits_{i=1}^n f_i(x),
\end{equation}
where each function $f_i:\R^d\to\R$ represents the local loss associated with the data owned by device or client $i\in[n]\eqdef\{1,2,\dots,n\}$ only. This formulation aims to train a single machine learning model $x\in\R^d$ composed of $d$ parameters by minimizing empirical loss $f(x)$ using all $n$ clients' data. We assume $f$ is $\mu$-strongly convex and problem (\ref{problem-erm}) has the unique optimal solution $x^*$ throughout the paper. 

Due to the increasing size of the model and the amount of the training data, in practical deployments, methods of choice for solving the problem \eqref{problem-erm} have been {\em distributed first-order gradient methods} so far \citep{FOM-in-DS,grace}. Among other things, two key considerations in the design of efficient distributed optimization method are {\em iteration complexity}, which measures how many iterations the method should take to achieve some prescribed accuracy, and {\em communication cost} per iteration, which measures the number of bits that clients communicate to each other or some parameter server \citep{bekkerman2011scaling}. Expectedly, these two quantities are in a trade-off: reducing the size of communicated messages per iteration, potentially increases the total number of iterations. This trade-off then interacts with the problem structure (training data and model) and network properties (bandwidth and latency) to find the best configuration for final deployment.

However, despite their wide applicability, all first-order methods inevitably suffer from a dependence of suitably defined {\em condition number}. To overcome the curse of the condition number, Newton-type or second-order optimization methods have been gaining considerable attention recently since (at least local) convergence properties of these algorithms are not affected by the condition number of the problem \citep{quasiNewton1974,InexactNewton1982,Griewank1981,PN2006-cubic}. On the other hand, the caveat of this approach is that, although it greatly improves iteration complexity, the cost of naively communicating second-order information, such as Hessian matrices, is extremely high and infeasible in practice \citep{bekkerman2011scaling}.
In this work, we argue that with proper care of second-order information and for ill-conditioned problems, distributed second-order algorithms can offer essentially better trade-offs than first-order algorithms.

\subsection{Distributed second order methods.}
The straightforward implementation of classical Newton's method in the distributed environment includes communication of local Hessian matrices $\nabla^2 f_i(x^k)$ with $d^2$ entries in all iterations $k\ge0$. Consider this naive implementation of Newton's method as the baseline algorithm for distributed second-order optimization, just like the distributed gradient descent algorithm for first-order methods. Below we discuss the main algorithmic designs to reduce the quadratic dependence $d^2$ of the dimension in per-iteration communication cost and make second-order methods communication efficient for distributed optimization.

One stream of works avoids sending full Hessian matrices and uses second-order information locally to compute Hessian-vector products $\nabla^2 f_i(x^k)v_i^k\in\R^d$ for some (adaptively defined) vectors $v_i^k\in\R^d$. With this approach, second order information is imparted only through such products, which cost $d$ floats of communication instead of $d^2$. The computation side of this approach is also efficient since the Hessian matrices are not computed directly but Hessian-vector products only, which are as cheap to compute as gradients $\nabla f_i(x^k)$ \citep{HessianXvector1993}. Methods following this approach are DiSCO \citep{DiSCO2015} (see also \citep{Zhuang2015,Lin2014LargescaleLR,Newton-MR2019}), GIANT \citep{GIANT2018} (see also \citep{DANE,Reddi:2016aide}) and DINGO \citep{DINGO} (see also \citep{CompressesDINGO2020}). Inspired by the local first order methods \citep{Gorbunov2020localSGD,Stich-localSGD,localSGD-AISTATS2020,FEDLEARN}, local variant of Newton's method was studied in \citep{LocalNewton2021}.
Typically these methods either guarantee fast rates under stronger assumptions, such as quadratic problems or/and homogeneous data distribution or guarantee only linear rates attainable by first-order methods.



Alternatively, the high cost of Hessian communication can be reduced by compressing second-order information via lossy compression operators acting on matrices (such as low-rank approximations, randomly or greedily sparsifying entries). Again, this idea was originated from the first-order methods employing communication compression \citep{tonko,alistarh2018convergence,qsgd,terngrad}. Recently developed second-order methods DAN-LA \citep{DAN-LA2020}, Quantized Newton \citep{Alimisis2021QNewton}, NewtonLearn \citep{Islamov2021NewtonLearn} and FedNL \citep{FedNL2021} are based on this idea of properly incorporating compression strategies for second-order information. In contrast to the previous approach, this strategy relies on the computation of full Hessian matrices, which might be computationally intensive for some applications. However, the optimization problem these methods address is quite generic (general finite sum structure \eqref{problem-erm} over arbitrarily heterogeneous data), and theoretical guarantees (global linear with local superlinear convergence rates) are far beyond the reach of all first-order methods.

Motivated by these recent developments on distributed second-order methods with communication compression, we extend the results of FedNL \citep{FedNL2021} allowing even more aggressive compression for some applications.

\section{Motivation and Contributions}\label{sec:motiv-contr}

To motivate our approach properly and illustrate the potential of our technique, we discuss three different implementations of {\em classical Newton's method} for solving the problem \eqref{problem-erm}.


\subsection{Naive implementation} For general finite sums \eqref{problem-erm}, Newton's method requires each device $i\in[n]$ to compute gradient vector $\nabla f_i(x)$ and Hessian matrix $\nabla^2 f_i(x)$ at the current point and send them to the parameter server to do the model update. While the convergence of Newton's method is locally quadratic, ${\cal O}(d^2)$ communication costs are high due to quadratic dependence from the dimension $d$. However, we can devise more efficient implementations, given some prior knowledge of the problem or/and data structure.

\subsection{Utilizing the problem structure}
Suppose the problem \eqref{problem-erm} models the training of {\em Generalized Linear Model (GLM)}, such as ridge regression or logistic regression. Then each local loss function has the form 
\begin{equation}\label{avg-of-avg}
	f_i(x) = \frac{1}{m}\sum_{j=1}^m f_{ij}(x),
\end{equation}
where $f_{ij}(x) \eqdef \varphi_{ij}(a_{ij}^\top x)$ is the loss associated with $j^{th}$ data-vector $a_{ij}\in\R^d$ stored on $i^{th}$ device and $\varphi_{ij}\colon\R\to\R$ is the corresponding loss function.\footnote{For simplicity we assume that each device has the same number of local data points denoted by $m$. Generally, it could be different for different clients.} A better implementation of Newton's method taking advantage of the Hessian structure
\begin{eqnarray}
	\nabla^2 f_i(x) &=&\squeeze \frac{1}{m}\sum_{j=1}^m \varphi''_{ij}(a_{ij}^\top x) a_{ij}a_{ij}^\top \label{glm-hessian}
\end{eqnarray}
is described in  \citep[Section 2.2]{Islamov2021NewtonLearn}. In this implementation, the server is assumed to have {\em access to all training dataset} $\{a_{ij}\}_{ij}$. Then to communicate Hessian matrix of the form \eqref{glm-hessian}, it is enough to send $m$ coefficients $\{\varphi''_{ij}(a_{ij}^\top x) \colon j\in[m]\}$ instead of $d^2$ entries. Hence, in cases when $m \ll d^2$, we can run Newton's method much efficiently with ${\cal O}(m + d)$ communication cost.

However, there are two limitations here. First, this approach fails to benefit when local datasets are too big (i.e., $m > d^2$), which is often the case in practice. Second, all devices reveal their local training data to the server, making the approach inapplicable to federated learning applications, where data privacy is crucial.

\subsection{Utilizing the data structure}

We now describe a strategy that additionally takes advantage of the data structure and dramatically reduces communication costs regardless of the data size and without revealing any local data.

The imposed structural assumption on the data is that local data points $\{a_{ij} \colon j\in[m]\}$ of $i^{th}$ client belong to some linear subspace $G_i\subseteq\R^d$ of dimension $r\in[d]$.\footnote{To make notations simpler, $r$ is the same for all clients.} Note that this condition is trivially satisfied for $r = d$ for any data. However, in practice, training data points have much smaller intrinsic dimensionality $r\ll d$.
Notice that once we fix some basis $\{v_{it}\}_{t=1}^r$ of $G_i$, we can represent data points $a_{ij}$ as linear combinations
\begin{equation}\label{lin_comb}
	\squeeze
	a_{ij} = \sum_{t=1}^r \alpha_{ijt}v_{it}, \quad j\in[m].
\end{equation}

Instead of directly revealing actual data points $a_{ij}$, each device sends the basis $\{v_{it}\}_{t=1}^r$ to the server {\em only once} (before the training) with the cost of sending $r d$ floats. Based on the representations \eqref{glm-hessian} from the problem structure and \eqref{lin_comb} from the data structure, the Hessian of $f_{i}(x)$ can be transformed into

\begin{equation}\label{hessian-form}
	\squeeze
	\nabla^2f_{i}(x)
	\overset{\eqref{glm-hessian},\eqref{lin_comb}}{=} \frac{1}{m}\sum\limits_{j=1}^m \varphi^{\prime\prime}_{ij}(a_{ij}^\top x)\sum\limits_{t, l=1}^r \alpha_{ijt}\alpha_{ijl} v_{it} v_{il}^\top
	= \sum\limits_{t, l=1}^r \Bigg[\underbrace{\squeeze\frac{1}{m}\sum\limits_{j=1}^m \varphi^{\prime\prime}_{ij}(a_{ij}^\top x) \alpha_{ijt}\alpha_{ijl}}_{\gamma_{itl}(x)}\Bigg] \underbrace{v_{it} v_{il}^\top}_{\mV_{itl}},
\end{equation}



where outer products $\mV_{itl} \eqdef v_{it} v_{il}^\top$ are linearly independent matrices (see Lemma~\ref{lem:outer-prod-ind}  in the Appendix) and $\gamma_{itl}(x)$ are coefficients in the brackets. The key observation from \eqref{hessian-form} is that to communicate $\nabla^2 f_i(x)$ we need to send only $r^2$ coefficients $\{\gamma_{itl}(x) \colon t,l\in[r]\}$ instead of $d^2$ entries as the server already knows matrices $\mV_{itl}$ through the basis $\{v_{it}\}_{t=1}^{r}$.
The takeaway from this observations is that the standard basis of $\R^{d\times d}$ is not always optimal. Indeed, in this case, any basis of $\R^{d\times d}$ containing $r^2$ matrices $\{\mV_{itl}\}_{t,l=1}^r$ is better choice for encoding Hessians $\nabla^2 f_i(x)$ without any loss in precision. Thus, ${\cal O}(d^2)$ communication cost is reduced to ${\cal O}(r^2 + d)$. In case of $r = {\cal O}(\sqrt{d})$, we get Newton's method with $\cO(d)$ communication cost and local quadratic convergence.

Analogous to \eqref{hessian-form}, similar representation can be derived for gradients too, namely $\nabla f_i(x) \in G_i$. Hence, we can send $\nabla f_i(x)$ by its $r$ basis coefficients instead of $d$ coordinates. This way we further reduce communication cost up to ${\cal O}(r^2)$ (see Table \ref{table:newtons-impl} for the summary). In the extreme case of $r = {\cal O}(1)$, we run Newton's method with $\cO(1)$ cost per iteration!

\begin{table*}[t!]
    \centering
    \caption{Key features of different implementation of classical Newton's method in distributed systems. Here $m$ is the number of local training data, $r$ is the intrinsic dimensionality of local data vectors.}\label{table:newtons-impl}
    \renewcommand{\arraystretch}{1.7}
    \begin{tabular}{|c||c|c|c|}
        \hline
        \makecellnew{Implementation of \\ Newton's method} & Standard/Naive & \citep{Islamov2021NewtonLearn} & {\bf Ours} \\
        \hline
        Problem & \makecellnew{General \\ Finite Sum} & \makecellnew{General \\ Finite Sum} & \makecellnew{Generalized \\ Linear Model}  \\
        \hline
        Data & Arbitrary & Arbitrary & \makecellnew{Intrinsically \\ Low-Dimensional}  \\
        \hline
        \makecellnew{Gradient communication \\ cost per iteration (floats)} & $d$ & $\min(m,d)$ & $r$  \\
        \hline
        \makecellnew{Hessian communication \\ cost per iteration (floats)} & $d^2$ & $\min(m,d^2)$ & $r^2$  \\
        \hline
        \makecellnew{Initial communication \\ cost (floats)} & -- & $m d$ & $r d$  \\
        \hline
        Reveals local training data ? & No & Yes & No  \\
        \hline
    \end{tabular}       
\end{table*}

Note that \eqref{hessian-form} is a special case of more general Hessian representation $\nabla^2 f_i(x) = \mQ_i \Lambda_i(x) \mQ_i^\top$, where $\mQ_i$ is a fixed invertible matrix (known to the server) and $\Lambda_i(x)$ is a sparse matrix with much less than $d^2$ (e.g., $r^2$ for \eqref{hessian-form}) non-zero entries. Changing the standard basis of $\R^{d\times d}$ via the transition matrix $\mQ_i$, we transform potentially dense Hessian $\nabla^2 f_i(x)$ (in the standard basis) into sparse $\Lambda_i(x)$ in the new basis.

\begin{quote}
{\em Thus, we save in communication for free just by changing the basis in the beginning of the training.} Motivated by this idea, we propose a new approach: {\em Basis Learn}.
\end{quote}

\subsection{Contributions}
Our goal is to further investigate the benefits and possible pitfalls of using custom bases in second-order optimization for general finite sums \eqref{problem-erm} with arbitrarily heterogeneous data.
As, by choosing a suitable basis, we can transform the Hessian into a sparser matrix in a lossless way, we propose and design three new methods, which apply lossy compression strategies afterwards to get even better performance in terms of communication complexity.

\paragraph{(1) Basis learn with bidirectional compression.} Our first contribution is the new method \algname{BL1}, which successfully integrates bidirectional compression with any predefined basis for Hessians. In \algname{BL1}, both client-to-server and server-to-client communications are compressed via careful application of compression operators. We allow both unbiased compressors, such as random sparsification (Rand-$K$) or random dithering, and contractive compressors, such as greedy sparsification (Top-$K$) or low-rank approximations (Rank-$R$).
In the special case of choosing the standard basis, our method recovers FedNL \citep{FedNL2021}. Thus, basis learn can be viewed as a generalization of FedNL. 

\paragraph{(2) Extensions to partial device participation.} For massively distributed trainings, such as in federated learning, with too many clients, we propose two extensions, \algname{BL2} and \algname{BL3}, to accommodate partial participation of devices. Thus, we unify bidirectional compression and partial participation under basis learn.
Furthermore, within these two extensions we propose two options to guarantee the positive definiteness of accumulated Hessian estimator at the server avoiding matrix projection steps of \algname{BL1}: first option (implemented in \algname{BL2}) is based on compression error trick of \citep{FedNL2021}, while the other option (realized in \algname{BL3}) is to choose bases with positive semidefinite matrices in the symmetric matrix space.

\paragraph{(3) Fast local rates.} For all our methods we prove local linear and superlinear rates independent of the condition number and the size of local dataset.

\paragraph{(4) Experiments.} By composing low-rank approximation and unbiased compression operators, we propose more efficient compressors for matrices leading to better performance in the experiments.

\section{Matrix Compression}
\label{sec:3}

Here we adopt two classes of vector compressor operators to matrices. A (possibly) randomized map $\cC:\R^{d\times d}\to \R^{d\times d}$ is called a {\em  contraction compressor} if there exists a constant $0< \delta \leq 1$ such that  
\begin{equation}\label{eq:contractor}
	\mathbb{E} \left[\|\mA - \cC(\mA) \|_{\rm F}^2\right] \leq (1-\delta)\|\mA\|_{\rm F}^2, \quad \forall \mA \in \R^{d\times d}. 
\end{equation}
Further, we say that a randomized map $\cC:\R^{d\times d}\to \R^{d\times d}$ is  an {\em unbiased compressor} if  there exists a constant $\omega \geq 0$ (``variance parameter'') such that
\begin{equation} \label{eq:unbiased}
	\mathbb{E}\left[ \cC(\mA) \right] = \mA \quad {\rm and} \quad \mathbb{E}\left[\| \cC(\mA)\|_{\rm F}^2\right] \leq (\omega + 1)\|\mA \|_{\rm F}^2, \quad \forall \mA \in \R^{d\times d}. 
\end{equation}

The contraction compressor and unbiased compressor on $\R^d$ can be defined in the same way where the Frobenius norm $\|\cdot\|_{\rm F}$ is replaced by the standard Euclidean norm $\|\cdot\|$. For more examples of contraction and unbiased compressors, we refer the reader to \citep{FedNL2021, biased2020}. On the other hand, the compressor on $\R^{d\times d}$ can be regarded as a compressor on $\R^{d^2}$. Hence, compressors on the vector space $\R^{d^2}$ can be applied to the matrix in $\R^{d\times d}$. 
One can combine two compressors from different classes to get new ones \citep{EC-LSVRG}. In particular, we consider composition of Rank-$R$ \citep{FedNL2021} and unbiased compressors below.

Suppose ${\cal Q}_1^i$ and ${\cal Q}_2^i$, $i\in [d]$ are unbiased compressors on $\R^d$ with variance parameter $\omega_1$ and $\omega_2$, respectively. For any $\mA \in \R^{d\times d}$, let $\mA = \sum_{i=1}^d \sigma_i u_iv_i^\top$ be the singular value decomposition of $\mA$ with singular values $\sigma_1\geq \sigma_2\geq \dots\geq \sigma_d\geq0$. For $R\leq d$, define 
$$
\squeeze
{\cal C}_1(\mA) \eqdef \sum_{i=1}^R  \frac{\sigma_i {\cal Q}_1^i(a_i u_i) {\cal Q}_2^i(b_i v_i)^\top}{a_ib_i(\omega_1 +1)(\omega_2+1)}, 
$$
where $a_i,b_i>0$ are som constants chosen independently of ${\cal Q}_1^i$ and ${\cal Q}_2^i$ for $1\leq i\leq R$. For example, we can set $a_i = b_i = 1$ for all $i$, or $a_i = b_i = \sqrt{\sigma_i}$ for all $i$. Notice that if $\mA$ is symmetric, ${\cal C}_1(\mA)$ is not necessarily symmetric. However, we can symmetrize the output matrix by defining
$$
{\cal C}_2(\mA) \eqdef
\left\{\begin{smallmatrix}
	{\cal C}_1(\mA)   &&\text{if}\; \mA \;\text{is not symmetric}   \\
	\tfrac{{\cal C}_1(\mA) + {\cal C}_1(\mA)^\top}{2}   &&\text{if}\; \mA \;\text{is symmetric} \phantom{~~~~}
\end{smallmatrix}\right. .
$$

The following lemma says that this symmetrization process preserves contractiveness. 

\begin{lemma}\label{lm:ABineq}
\begin{itemize}
\item[(i)] For any $\mA, \mB\in \R^{d\times d}$, if $\mA$ is symmetric, then 
	$
	\left\|\nicefrac{(\mB + \mB^\top)}{2} - \mA \right\|_{\rm F} \leq \|\mB - \mA\|_{\rm F}. 
	$
\item[(ii)] For any contraction compressor $\cC:\R^{d\times d}\to \R^{d\times d}$ with contraction parameter $\delta$, the mapping ${\tilde {\cal C}}:\R^{d\times d}\to \R^{d\times d}$ defined by
	$$
	{\tilde {\cal C}}(\mA) \eqdef
	\left\{\begin{smallmatrix}
		{\cal C}(\mA)   &&\text{if}\; \mA \;\text{is not symmetric}   \\
		\tfrac{{\cal C}(\mA) + {\cal C}(\mA)^\top}{2}   &&\text{if}\; \mA \;\text{is symmetric} \phantom{~~~~}
	\end{smallmatrix}\right.
	$$ 
	is also a contraction compressor with contraction parameter $\delta$. 
\end{itemize}	
\end{lemma}

\begin{proposition}\label{th:C1C2}
	${\cal C}_1$ and ${\cal C}_2$ are both contraction compressors with parameter $\tfrac{R}{d(\omega_1+1)(\omega_2+1)}$. 
\end{proposition}


\section{Basis Learning in $\R^{d\times d}$}

Let $\{ \mB_i^{jl} \ | \  j, l\in [d] \}$ be a basis in $\R^{d\times d}$ and $N \eqdef d^2$ be the number of matrices in the basis for any $i\in [n]$. Then for any matrix $\mA$ in $\R^{d\times d}$, it can be uniquely represented as 
$
\mA = \sum_{j, l} h^i_{jl}(\mA) \mB_i^{jl}, 
$
where $h^i_{jl}(\mA) \in \R$ is the coefficient corresponding to $\mB_i^{jl}$. 

Define $h^i(\mA) \in \R^{d\times d}$ such that $h^i(\mA)_{jl} \eqdef h^i_{jl}(\mA)$ for $j, l \in [d]$. For any matrix $\mA \in \R^{d\times d}$, let $vec: \R^{d\times d} \to \R^{d^2}$ be defined as
$$
vec(\mA) = (\mA_{11}, ..., \mA_{d1}, ..., \mA_{1i}, ..., \mA_{di}, ..., \mA_{1d}, ..., \mA_{dd})^\top.
$$
Define ${\cal M}_i \eqdef (\mB_i^{11}, ..., \mB_i^{d1}, ..., \mB_i^{1j}, ..., \mB_i^{dj}, ..., \mB_i^{1d}, ..., \mB_i^{dd})$. Then we have $\mA =  {\cal M}_i vec(h^i(\mA))$, which is equivalent to 
\begin{equation}\label{eq:Arep-BL1}
	vec(\mA) = {\cal B}_i \cdot vec(h^i(\mA)), 
\end{equation}
for any matrix $\mA \in \R^{d\times d}$, where 
\begin{align*}
	{\cal B}_i &\eqdef (vec(\mB_i^{11}), ..., vec(\mB_i^{d1}), ..., vec(\mB_i^{1j}), ..., vec(\mB_i^{dj}), ..., vec(\mB_i^{1d}), ..., vec(\mB_i^{dd})) \in \R^{N \times N}. 
\end{align*}
Since the representation (\ref{eq:Arep-BL1}) is unique, we know ${\cal B}_i$ is invertible, and thus 
\begin{equation}\label{eq:hmA}
	vec(h^i(\mA)) = {\cal B}_i^{-1} vec(\mA). 
\end{equation}

\begin{example}
	The $(j,l)^{th}$ entry of $\mB_i^{jl} $ is $1$ and the others are $0$ for $j, l \in [d]$. Then $\mA = h^i(\mA)$ for any $\mA \in \R^{d\times d}$. 
\end{example}

\begin{example}
	The $(j,l)^{th}$ and $(l,j)^{th}$ entries of $\mB_i^{jl} $ are $1$ and the others are $0$ for $d\geq j \geq l \geq 1$. The $(j,l)^{th}$ entry of $\mB_i^{jl} $ is $1$, the $(l,j)^{th}$ entry of $\mB_i^{jl} $ is $-1$, and the others are $0$ for $1\leq j<l\leq d$. Then for any symmetric matrix $\mA \in \R^{d\times d}$, $h^i(\mA)$ is the lower triangular part of $\mA$.
\end{example}

We have Algorithm~\ref{alg:BL1} (\algname{BL1}) as an extension of \algname{FedNL-BC} in \citep{FedNL2021}. \algname{BL1} mainly has two differences from \algname{FedNL-BC}: (i) We use $\mL_i^k$ to learn the coefficient matrix $h^i(\nabla^2 f_i(z^k))$ rather than the Hessian; (ii) When $\xi^k=0$, we use $\[\mH^k\]_{\mu}$ rather than $\mH^k$ to construct the gradient estimator $g^k$, where $\[\cdot\]_\mu$ represents the projection on the set $\{  \mA \in \R^{d\times d} | \ \mA=\mA^\top, \mA \succeq \mu \mI  \}$.

\begin{algorithm}[h]
	\caption{\algname{BL1} (Basis Learn with {\color{blue}Bidirectional Compression)}}
	\label{alg:BL1}
	\begin{algorithmic}[1]
		\STATE \textbf{Parameters:} Hessian learning rate $\alpha\ge0$; {\color{blue}model learning rate $\eta\ge0$}; {\color{blue}gradient compression probability $p\in(0,1]$}; compression operators $\{\cC_1^k,\dots,\cC_n^k\}$ and {\color{blue}$\cQ^k$}
		\STATE \textbf{Initialization:} $x^0\;{\color{blue}=w^0=z^0}\in\R^d$; $\mL_i^0 \in \R^{d\times d}$, $\mH_i^0 = \sum_{jl} (\mL_i^0)_{jl} \mB_i^{jl}$, and $\mH^0 \eqdef \tfrac{1}{n}\sum_{i=1}^n \mH_i^0$; {\color{blue}$\xi^0 = 1$}
		\FOR{each device $i = 1, \dots, n$ in parallel}
		\STATE \textbf{if} {\color{blue}$\xi^k = 1$}
		\STATE ~~ $ w^{k+1} = z^{k}$, compute local gradient $\nabla f_i(z^k)$ and send to the server
		\STATE \textbf{if} {\color{blue}$\xi^k = 0$}
		\STATE ~ $w^{k+1} = w^k$
		\STATE Compute local Hessian $\nabla^2 f_i(z^k)$ and send $\mS_i^k \eqdef \cC_i^k(h^i(\nabla^2 f_i(z^k)) - \mL_i^k)$ to the server
		\STATE Update local Hessian shifts $\mL_i^{k+1} = \mL_i^k + \alpha \mS_i^k$,  $\mH_i^{k+1} = \mH_i^k + \alpha \sum_{jl} (\mS_i^k)_{jl} \mB^{jl}_i$
		\ENDFOR 
		\STATE \textbf{on} server
		\STATE \quad \textbf{if} {\color{blue}$\xi^k = 1$}
		\STATE \quad ~~ $w^{k+1} = z^k,\; g^k = \nabla f(z^k)$
		\STATE \quad \textbf{if} {\color{blue}$\xi^k = 0$}
		\STATE \quad ~~ $w^{k+1} = w^k,\; g^k = \[\mH^k\]_{\mu}(z^k-w^k) + \nabla f(w^k)$ 
		\STATE \quad $x^{k+1} = z^k - \[\mH^k\]_{\mu}^{-1} g^k$
		\STATE \quad $\mH^{k+1} = \mH^k + \frac{\alpha}{n} \sum_{i=1}^n \sum_{jl} (\mS_i^k)_{jl} \mB_i^{jl}$
		\STATE \quad {\color{blue}Send $v^k \eqdef \cQ^k (x^{k+1}-z^k)$ to all devices $i\in[n]$}
		\STATE \quad {\color{blue}Update the model $z^{k+1} = z^k+\eta v^k$}
		\STATE \quad {\color{blue}Send $\xi^{k+1} \sim \text{Bernoulli}(p)$ to all devices $i\in[n]$}
		\FOR{each device $i = 1, \dots, n$ in parallel}
		\STATE {\color{blue}Update the model $z^{k+1}=z^k+\eta v^k$}
		\ENDFOR
	\end{algorithmic}
\end{algorithm}

\begin{assumption}\label{as:Qunbiasedcomp-BL1}
	(i) ${\cal Q}^k$ (${\cal Q}_i^k$) is an unbiased compressor with parameter $\omega_{\rm M}$ and $0<\eta \leq \nicefrac{1}{(\omega_{\rm M}+1)}$.
	(ii) For all $j\in[d]$, $(z^k)_j$ in Algorithm~\ref{alg:BL1} ($(z_i^k)_j$ in Algorithm~\ref{alg:BL2} or Algorithm~\ref{alg:BL3} ) is a convex combination of $\{  (x^t)_j  \}_{t=0}^k$ for $k\geq 0$. 
\end{assumption}

\begin{assumption}\label{as:Qcontractioncomp-BL1}
	(i) ${\cal Q}^k$ (${\cal Q}_i^k$) is a contraction compressor with parameter $\delta_{\rm M}$ and $\eta=1$. (ii) ${\cal Q}^k$ (${\cal Q}_i^k$) is deterministic, i.e., $\mathbb{E}[{\cal Q}^k(x)] = {\cal Q}^k(x)$ for any $x\in\R^d$. 
\end{assumption}

\begin{assumption}\label{as:Cunbiasedcomp-BL1}
	(i) $\cC_i^k$ is an unbiased compressor with parameter $\omega$ and $0<\alpha \leq \nicefrac{1}{(\omega+1)}$.\\ (ii) For all $i\in[n]$ and $j,l\in[d]$, $(\mL^k_i)_{jl}$ is a convex combination of $\{  h^i(\nabla^2 f_i(z^t))_{jl}  \}_{t=0}^k$ in Algorithm~\ref{alg:BL1} ($\{h^i(\nabla^2 f_i(z_i^t))_{jl}  \}_{t=0}^k$ in Algorithm~\ref{alg:BL2} or $\{  {\tilde h}^i(\nabla^2 f_i(z_i^t))_{jl}  \}_{t=0}^k$ in Algorithm~\ref{alg:BL3} ) for $k\geq 0$. 
\end{assumption}

\begin{assumption}\label{as:Ccontractioncomp-BL1}
	(i) $\cC_i^k$ is a contraction compressor with parameter $\delta$ and $\alpha=1$. (ii) $\cC_i^k$ is deterministic, i.e., $\mathbb{E}[\cC_i^k(\mA)] = \cC_i^k(\mA)$ for any $\mA\in\R^{d\times d}$. 
\end{assumption}

\begin{assumption}\label{as:BL1}
	We have
	$\|\nabla^2 f_i(x)-\nabla^2 f_i(y)\| \leq H\|x-y\|$,
	$\|\nabla^2 f_i(x)-\nabla^2 f_i(y)\|_{\rm F} \leq H_1\|x-y\|$,
	$\|h^i(\nabla^2 f_i(x)) - h^i(\nabla^2 f_i(y))\|_{\rm F} \leq M_1 \|x-y\|$,
	$\max_{jl} \{ |h^i(\nabla^2 f_i(x))_{jl} - h^i(\nabla^2 f_i(y))_{jl}| \} \leq M_2 \|x-y\|$,
	$\max_{jl}\{\|\mB_i^{jl}\|_{\rm F}\}\leq R$
	for any $x,y\in \R^d$ and $i\in[n]$. For Algorithm~\ref{alg:BL2}, we assume each $f_i$ is $\mu$-strongly convex. 
\end{assumption}

We estimate $M_1$ and $M_2$ in Assumption \ref{as:BL1} in the following lemma. 
\begin{lemma}\label{lm:EstM-1}
	Assume $\|\nabla^2 f_i(x) - \nabla^2 f_i(y)\|_{\rm F} \leq H_1\|x-y\|$ and $\max_{jl} \{ |(\nabla^2 f_i(x))_{jl} - (\nabla^2 f_i(y))_{jl} | \} \leq \nu \|x-y\|$ for any $x, y\in \R^d$, and $i \in [n]$. Then we have $M_1 \leq   \max_i \{ \| {\cal B}_i^{-1}\| \} H_1$ and $M_2 \leq \nu \max_{i} \{\| {\cal B}_i^{-1}\|_{\infty}\}$. 
\end{lemma}

To present our theory in a unified manner, we define
\begin{equation}\label{eq:N_B-BL1}
	N_{\rm B} \eqdef
	\left\{\begin{smallmatrix}
		1   &&\text{if the bases}\; \{\mB_i^{jl} \}_{j,l\in[d]} \; \text{are all orthogonal}   \\
		d^2   &&\text{otherwise} \phantom{~~~~~~~~~~~~~~~~~~~~~~~~~~~~~~~~~~}
	\end{smallmatrix}\right.
\end{equation}

\begin{equation}\label{eq:AB_M-BL1}
	(A_{\rm M}, B_{\rm M}) \eqdef
	\left\{\begin{smallmatrix}
		(\eta, \eta)   && \text{if Asm.}\; \ref{as:Qunbiasedcomp-BL1} \text{(i) holds}    \\
		\(\tfrac{\delta_{\rm M}}{4}, \tfrac{6}{\delta_{\rm M}} - \tfrac{7}{2}\)   && \text{if Asm.}\; \ref{as:Qcontractioncomp-BL1} \text{(i) holds}   
	\end{smallmatrix}\right.
\end{equation}

\begin{equation}\label{eq:AB-BL1}
	(A, B) \eqdef
	\left\{\begin{smallmatrix}
		(\alpha, \alpha)   && \text{if Asm.}\; \ref{as:Cunbiasedcomp-BL1} \text{(i) holds}    \\
		\(\tfrac{\delta}{4}, \tfrac{6}{\delta} - \tfrac{7}{2}\)   && \text{if Asm.}\; \ref{as:Ccontractioncomp-BL1} \text{(i) holds}   
	\end{smallmatrix}\right.
\end{equation}

and ${\cal H}^k \eqdef \tfrac{1}{n} \sum_{i=1}^n \|\mL_i^k - \mL_i^*\|^2_{\rm F}$, $\Phi^k_1 \eqdef \|z^k-x^*\|^2 + \tfrac{A_{\rm M}(1-p)}{2p}\|w^k-x^*\|^2$, where $\mL_i^* \eqdef h^i(\nabla^2 f_i(x^*))$, for $k\geq 0$.

\begin{theorem}[Linear convergence of \algname{BL1}]\label{th:linear-BL1}
	Let Assumption \ref{as:BL1} hold. Let Assumption \ref{as:Qunbiasedcomp-BL1} (i) or Assumption \ref{as:Qcontractioncomp-BL1} (i) hold. Assume $\|z^k-x^*\|^2 \leq \tfrac{A_{\rm M}\mu^2}{4H^2B_{\rm M}}$ and ${\cal H}^k \leq \tfrac{A_{\rm M}\mu^2}{16N_{\rm B}R^2 B_{\rm M}}$ for $k\geq 0$. Then we have 
	$$
	\mathbb{E}[\Phi^k_1] \leq \left(  1 - \frac{\min\{  A_{\rm M}, p  \}}{2}  \right)^k \Phi^0_1, 
	$$
	for $k\geq 0$. 
\end{theorem}

Define $\Phi_2^k \eqdef {\cal H}^k + \tfrac{4BM_1^2}{A_{\rm M}}\|x^k-x^*\|^2 $ for $k\geq 0$. 

\begin{theorem}[Superlinear convergence of \algname{BL1}]\label{th:superlinear-BL1}
	Let $\eta=1$, $\xi^k \equiv 1$ and ${\cal Q}^k(x) \equiv x$ for any $x\in \R^d$ and $k\geq 0$. Let Assumption \ref{as:BL1} hold. Let Assumption \ref{as:Cunbiasedcomp-BL1} (i) or Assumption \ref{as:Ccontractioncomp-BL1} (i) hold. Assume $\|z^k-x^*\|^2 \leq \tfrac{A_{\rm M}\mu^2}{4H^2B_{\rm M}}$ and ${\cal H}^k \leq \tfrac{A_{\rm M}\mu^2}{16N_{\rm B}R^2 B_{\rm M}}$ for $k\geq 0$. Then we have 
	$$
	\mathbb{E}[\Phi_2^k] \leq \theta_1^k \Phi_2^0, 
	$$
	and 
	$$
	\mathbb{E} \left[   \frac{\|x^{k+1}-x^*\|^2}{\|x^k-x^*\|^2} \right] \leq \theta_1^k \left(  \frac{A_{\rm M}H^2}{8BM_1^2\mu^2} + \frac{2N_{\rm B}R^2}{\mu^2}  \right) \Phi_2^0, 
	$$
	for $k\geq 0$, where $\theta_1 \eqdef \left(  1 - \frac{\min\{  4A, A_{\rm M}  \} }{4} \right)$. 
\end{theorem}

Next, we explore under what conditions we can guarantee the boundness of $\|z^k-x^*\|$ and ${\cal H}^k$. 
\begin{theorem}\label{th:nbor-BL1}
	Let Assumption \ref{as:BL1} hold. Then we have the following results. 
\begin{itemize}
\item[(i)] Let Assumption \ref{as:Qunbiasedcomp-BL1} and Assumption \ref{as:Cunbiasedcomp-BL1} (ii) hold. If $\|x^0-x^*\|^2 \leq {\tilde c}_1 \eqdef \min \left\{  \tfrac{\mu^2}{4d^2H^2}, \tfrac{\mu^2}{16d^4 N_{\rm B}R^2M_2^2} \right\}$, then $\|z^k-x^*\|^2 \leq d{\tilde c}$ and ${\cal H}^k \leq \tfrac{\mu^2}{16dN_{\rm B}R^2 }$ for $k\geq 0$.  
\item[(ii)] Let Assumption \ref{as:Qcontractioncomp-BL1} and Assumption \ref{as:Ccontractioncomp-BL1} hold. If $\|z^0 - x^*\|^2 \leq {\tilde c}_2 \eqdef \min\left\{  \tfrac{A_{\rm M}\mu^2}{4H^2 B_{\rm M}},  \tfrac{AA_{\rm M}\mu^2}{16N_{\rm B}R^2 B_{\rm M}BM_1^2}  \right\}$ and ${\cal H}^0 \leq \tfrac{A_{\rm M}\mu^2}{16N_{\rm B}R^2B_{\rm M}}$, then $\|z^k - x^*\|^2 \leq {\tilde c}_2$ and ${\cal H}^k \leq \tfrac{A_{\rm M}\mu^2}{16N_{\rm B}R^2B_{\rm M}}$ for $k\geq 0$. 
\end{itemize}
\end{theorem}

We can also unify the bidirectional compression and partial participation to have \algname{BL2} (Algorithm~\ref{alg:BL2}), where $[\cdot]_s$ represents an operator on $\R^{d\times d}$ such that $[\mA]_s = \nicefrac{(\mA + \mA^\top)}{2}$ for any $\mA\in \R^{d\times d}$. Since each node has a local model $w_i^k$, we introduce $z_i^k$ to apply the bidirectional compression, and $\mL_i^k$ is expected to learn $h^i(\nabla^2 f_i(z_i^k))$ iteratively. Like in \algname{FedNL-PP} \citep{FedNL2021}, the key relation 
\begin{equation}\label{eq:gik-BL2}
	g_i^k = ([\mH_i^k]_s + l_i^k\mI) w_i^k - \nabla f_i(w_i^k)
\end{equation}
need to be maintained in the design of \algname{BL2}. The reason to keep relation (\ref{eq:gik-BL2}) is that the update of $x^k$ follows the structure of Stochastic Newton \citep{SN2019}, where $g_i^k$ is supposed to be $\nabla^2 f_i(w_i^k) - \nabla f_i(w_i^k)$, and naturally $\nabla^2 f_i(w_i^k)$ is replaced by the Hessian estimator $[\mH_i^k]_s + l_i^k\mI$. Here we use $l_i^k = \|[\mH_i^k]_s-\nabla^2 f_i(z_i^k)\|_{\rm F}$ to guarantee the positive definiteness of $[\mH_i^k]_s+l_i^k \mI$ like in \citep{FedNL2021}. From (\ref{eq:gik-BL2}), it is easy to see that on server, for the $\xi_i^k=0$ case, $g_i^{k+1} - g_i^k = ([\mH_i^{k+1}]_s - [\mH_i^k]_s + l_i^{k+1}\mI - l_i^k\mI) w_i^{k+1}$ since $w_i^{k+1} = w_i^k$. We give the convergence results of \algname{BL2} in the following two theorems.

\begin{algorithm}[h!]
	\caption{\algname{BL2} (Basis Learn with {\color{blue}Bidirectional Compression and Partial Participation)}}
	\label{alg:BL2}
	\begin{algorithmic}[1]
		\STATE {\bfseries Parameters:} $\alpha>0$; $\eta>0$; matrix compression operators $\{\cC_1^k, \dots,\cC_n^k\}$; $p\in(0, 1]$; $0<\tau \leq n$
		\STATE {\bfseries Initialization:}
		$w^0_i = z^0_i = x^0 \in \R^d$; $\mL_i^0 \in \R^{d\times d}$; $\mH_i^0= \sum_{jl} (\mL_i^0)_{jl} \mB_i^{jl}$; $l_i^0 = \|[\mH_i^{0}]_s - \nabla^2 f_i(w_i^{0})\|_{\rm F}$; $g_i^0 = ([\mH_i^{0}]_s + l_i^{0} \mI)w_i^{0} - \nabla f_i(w_i^{0})$; Moreover: $\mH^0 = \tfrac{1}{n} \sum_{i=1}^n \mH_i^0$; $l^0 = \tfrac{1}{n} \sum_{i=1}^n l_i^0$; $g^0 = \tfrac{1}{n} \sum_{i=1}^n g_i^0$
		\STATE \textbf{on} server
		\STATE ~~~$x^{k+1} = \left(  [\mH^k]_s + l^k\mI  \right)^{-1} g^k$, {\color{blue} choose a subset $S^{k} \subseteq [n]$ such that $\mathbb{P}[ i \in S^k] = \nicefrac{\tau}{n}$ for all $i\in [n]$}
		\STATE ~~~$v_i^k = \cQ_i^k(x^{k+1} - z_i^k)$, $z_i^{k+1} = z_i^k + \eta v_i^k$ for $i \in S^k$ 
		\STATE ~~~$z_i^{k+1} = z_i^k$, \quad $w_i^{k+1} = w_i^k$ for $i \notin S^k$ 
		\STATE ~~~Send $v_i^k$ to {\color{blue} the selected devices $i\in S^k$} 
		\FOR{each device $i = 1, \dots, n$ in parallel}
		\STATE {\color{blue} {\bf for participating devices} $i \in S^k$ {\bf do} }
		\STATE $z_i^{k+1} = z_i^k + \eta v_i^k$, $\mS_i^k \eqdef \cC_i^k(h^i(\nabla^2 f_i(z_i^{k+1})) - \mL_i^k)$
		\STATE $\mL_i^{k+1} = \mL_i^k + \alpha \mS_i^k$, $\mH_i^{k+1} = \mH_i^k + \alpha \sum_{jl} (\mS_i^k)_{jl} \mB_i^{jl} $ 
		\STATE $l_i^{k+1} = \|[\mH_i^{k+1}]_s - \nabla^2 f_i(z_i^{k+1})\|_{\rm F}$ 
		\STATE Sample $\xi_i^{k+1} \sim \text{Bernoulli}(p)$
		\STATE {\color{blue}\textbf{if} $\xi_i^k=1$ }
		\STATE ~~~$w_i^{k+1} = z_i^{k+1}$, $g_i^{k+1} = ([\mH_i^{k+1}]_s + l_i^{k+1} \mI)w_i^{k+1} - \nabla f_i(w_i^{k+1})$, send $g_i^{k+1}-g_i^k$ to server 
		\STATE {\color{blue}\textbf{if} $\xi_i^k=0$ }
		\STATE ~~~$w_i^{k+1} = w_i^k$, $g_i^{k+1} = ([\mH_i^{k+1}]_s + l_i^{k+1} \mI)w_i^{k+1} - \nabla f_i(w_i^{k+1})$ 
		\STATE Send $\mS_i^k$, $l_i^{k+1} - l_i^k$, and $\xi_i^k$ to server 
		\STATE {\color{blue} {\bf for non-participating devices} $i \notin S^k$ {\bf do} }
		\STATE $z_i^{k+1} = z_i^k$, $w_i^{k+1} = w_i^k$, $\mL_i^{k+1} = \mL_i^k$, $\mH_i^{k+1} = \mH_i^k$, $l_i^{k+1} = l_i^k$, $g_i^{k+1} = g_i^k$ 
		\ENDFOR
		
		\STATE \textbf{on} server
		\STATE ~~~{\color{blue}\textbf{if} $\xi_i^k=1$ }
		\STATE \quad \quad $w_i^{k+1} = z_i^{k+1}$, receive $g_i^{k+1}-g_i^k$
		\STATE ~~~{\color{blue}\textbf{if} $\xi_i^k=0$ } 
		\STATE \quad \quad $w_i^{k+1} = w_i^k$, $g_i^{k+1}-g_i^k = \alpha \left[\sum_{jl} (\mS_i^k)_{jl} \mB_i^{jl} \right]_s w_i^{k+1} +  (l_i^{k+1} - l_i^k) w_i^{k+1}$ 
		\STATE ~~~$g^{k+1} = g^k + \tfrac{1}{n}\sum_{i\in S^k} \left(  g_i^{k+1} - g_i^k  \right)$  
		\STATE ~~~$\mH^{k+1} = \mH^k + \tfrac{\alpha}{n}\sum_{i\in S^k} \sum_{jl} (\mS_i^k)_{jl} \mB_i^{jl}$   
		\STATE ~~~$l^{k+1} = l^k + \tfrac{1}{n}\sum_{i\in S^k} \left(  l_i^{k+1} - l_i^k  \right)$ 
	\end{algorithmic}
\end{algorithm}

Let $\Phi_3^k \eqdef {\cal W}^k + \tfrac{2p}{A_{\rm M}}\left(  1 - \tfrac{\tau A_{\rm M}}{n}  \right) {\cal Z}^k$, where ${\cal Z}^k \eqdef \tfrac{1}{n} \sum_{i=1}^n \|z_i^k-x^*\|^2$, for $k\geq 0$. 

\begin{theorem}[Linear convergence of \algname{BL2}]\label{th:BL2}
	Let Assumption \ref{as:BL1} hold. Let Assumption \ref{as:Qunbiasedcomp-BL1} (i) or Assumption \ref{as:Qcontractioncomp-BL1} (i) hold. Assume $\|z_i^k-x^*\|^2 \leq \tfrac{A_{\rm M}\mu^2}{(6H^2 + 24H_1^2)B_{\rm M}}$ and ${\cal H}^k \leq \tfrac{A_{\rm M}\mu^2}{96N_{\rm B}R^2 B_{\rm M}}$ for all $i\in[n]$ and $k\geq 0$. Then we have 
	$$
	\mathbb{E}[\Phi_3^k] \leq \left(  1 - \frac{\tau \min\{  p, A_{\rm M}  \}}{2n}  \right)^k \Phi_3^0,
	$$
	for $k\geq 0$. 
\end{theorem}

Define $\Phi_4^k \eqdef {\cal H}^k + \tfrac{4BM_1^2}{A_{\rm M}}\|x^k-x^*\|^2 $ for $k\geq 0$. 

\begin{theorem}[Superlinear convergence of \algname{BL2}]\label{th:superlinear-BL2}
	Let $\eta=1$, $\xi^k \equiv 1$, $S^k \equiv [n]$, and ${\cal Q}_i^k(x) \equiv x$ for any $x\in \R^d$ and $k\geq 0$. Let Assumption \ref{as:BL1} hold. Let Assumption \ref{as:Cunbiasedcomp-BL1} (i) or Assumption \ref{as:Ccontractioncomp-BL1} (i) hold. Assume $\|z_i^k-x^*\|^2 \leq \tfrac{A_{\rm M}\mu^2}{(6H^2 + 24H_1^2)B_{\rm M}}$ and ${\cal H}^k \leq \tfrac{A_{\rm M}\mu^2}{96N_{\rm B}R^2 B_{\rm M}}$ for all $i\in[n]$ and $k\geq 0$. Then we have 
	$$
	\mathbb{E}[\Phi_4^k] \leq \theta_2^k \Phi_4^0, 
	$$
	and 
	$$
	\mathbb{E} \left[   \frac{\|x^{k+1}-x^*\|^2}{\|x^k-x^*\|^2} \right] \leq \theta_2^k \left(  \frac{A_{\rm M}(3H^2+12H_1^2)}{16BM_1^2\mu^2} + \frac{12N_{\rm B}R^2}{\mu^2}  \right) \Phi_4^0, 
	$$
	for $k\geq 0$, where $\theta_2 \eqdef  1 - \tfrac{\min\{  2A, A_{\rm M}  \} }{2} $. 
\end{theorem}

Now, we explore under what conditions we can guarantee the boundedness of $\|z_i^k-x^*\|$ and ${\cal H}^k$. 

\begin{theorem}\label{th:nbor-BL2}
	Let Assumption \ref{as:BL1} hold. Then we have the following results. 
\begin{itemize}	 
\item[(i)] Let Assumption \ref{as:Qunbiasedcomp-BL1} and Assumption \ref{as:Cunbiasedcomp-BL1} (ii) hold. If $$\|x^0-x^*\|^2 \leq {\tilde c}_3 \eqdef \min \left\{  \frac{\mu^2}{d^2(6H^2+24H_1^2)}, \frac{\mu^2}{96d^4 N_{\rm B}R^2M_2^2} \right\},$$ then $\|z_i^k-x^*\|^2 \leq d{\tilde c}_3$ and ${\cal H}^k \leq \tfrac{\mu^2}{96dN_{\rm B}R^2 }$ for $i\in[n]$ and $k\geq 0$. 
\item[(ii)] Let Assumption \ref{as:Qcontractioncomp-BL1} and Assumption \ref{as:Ccontractioncomp-BL1} hold. If $\|z_i^0 - x^*\|^2 \leq {\tilde c}_4$, where $$ {\tilde c}_4 \eqdef \min\left\{  \frac{A_{\rm M}\mu^2}{B_{\rm M}(6H^2+24H_1^2)},  \frac{AA_{\rm M}\mu^2}{96N_{\rm B}R^2 B_{\rm M}BM_1^2}  \right\},$$ and $\|\mL_i^0 - \mL_i^*\|^2_{\rm F} \leq \tfrac{A_{\rm M}\mu^2}{96N_{\rm B}R^2B_{\rm M}}$ for all $i\in[n]$, then $\|z_i^k - x^*\|^2 \leq {\tilde c}_4$ and $\|\mL_i^k - \mL_i^*\|^2_{\rm F} \leq \tfrac{A_{\rm M}\mu^2}{96N_{\rm B}R^2B_{\rm M}}$ for all $i\in[n]$ and $k\geq 0$. 
\end{itemize}	
\end{theorem}

\section{Basis Learning in ${\cal S}^d$}
Let $\{ \mB_i^{jl} \ | \  j, l \in [d], j\geq l \}$ be a basis in the symmetric subspace ${\cal S}^d$ of $\R^{d \times d}$ that consists of all the symmetric matrices for $i\in [n]$. In this case, the number of symmetric matrices in the basis is ${\tilde N} \eqdef \tfrac{d(d+1)}{2}$. Then for any symmetric matrix $\mA$ in $\R^{d\times d}$, it can be uniquely represented as 
$$
\mA = \sum_{j\geq l} {\tilde h}^i_{jl}(\mA) \mB^{jl}, 
$$
where ${\tilde h}^i_{jl}(\mA) \in \R$ is the coefficient corresponding to $\mB_i^{jl}$. Let $\mB_i^{lj} \eqdef \mB_i^{jl}$ for $j >l$ and define ${\tilde h}^i(\mA)$ as a symmetric matrix in $\R^{d\times d}$ such that ${\tilde h}^i(\mA)_{jl} \eqdef \tfrac{1}{2}{\tilde h}^i_{jl}$ for $j>l$ and ${\tilde h}^i(\mA)_{jl} \eqdef {\tilde h}^i_{jl}$ for $j=l$. Let $svec: {\cal S}^d \to \R^{{\tilde N}}$ be defined as 
$$ 
svec(\mA) \eqdef
(\mA_{11}, 2\mA_{21}..., 2\mA_{d1}, ..., \mA_{jj}, ..., 2\mA_{dj}, ..., \mA_{dd})^\top, 
$$
and ${\tilde {\cal M}}_i \eqdef (\mB_i^{11}, ..., \mB_i^{d1}, ..., \mB_i^{jj}, ..., \mB_i^{dj}, ..., \mB_i^{dd})$. Then we have $\mA = {\tilde {\cal M}}_i svec({\tilde h}^i(\mA))$, which is equivalent to 
\begin{equation}\label{eq:SArep}
	svec(\mA) = {\tilde {\cal B}}_i \cdot svec({\tilde h}^i(\mA)), 
\end{equation}
for any symmetric matrix $\mA$, where 
\begin{align*}
	{ \tilde {\cal B}}_i & \eqdef (svec(\mB_i^{11}), ..., svec(\mB_i^{d1}), ..., svec(\mB_i^{jj}), ..., svec(\mB_i^{dj}), ..., svec(\mB_i^{dd})) \in \R^{{\tilde N} \times {\tilde N}}. 
\end{align*}
Since the representation (\ref{eq:SArep}) is unique, we know ${\tilde {\cal B}}_i$ is invertible, and thus 
\begin{equation}\label{eq:hmAsym}
	svec({\tilde h}^i(\mA)) = ({\tilde {\cal B}}_i)^{-1} svec(\mA). 
\end{equation}

\begin{example}
	We choose $\mB_i^{jl} \in {\cal S}^d$ such that for $j\neq l$, $(\mB_i^{jl})_{jl} = (\mB_i^{jl})_{lj} = (\mB_i^{jl})_{jj} = (\mB_i^{jl})_{ll} =1$ and the other entries are $0$; for $j=l$, $(\mB_i^{jj})_{jj}=1$ and the other entries are $0$. It is easy to verify it is a basis in ${\cal S}^d$, and we also have $\mB_i^{jl} \succeq 0$.
\end{example}

\begin{algorithm}[h!]
	\caption{\algname{BL3} }
	\label{alg:BL3}
	\begin{algorithmic}
		\STATE {\bfseries Parameters:} learning rate $\alpha>0$, positive constant $c>0$, minibatch size $\tau \in \{1,2,\dots,n\}$ 
		\STATE {\bfseries Initialization:}
		$\mB_i^{jl} \succeq 0$; $w^0_i=z_i^0=x^0$ for $i\in [n]$; $\mL_i^0 \in \R^{d\times d}$; $\gamma_i^0 = \max_{jl}\{  c, |(\mL_i^0)_{jl}|  \} $; $\mA^0_i = \sum_{jl}((\mL_i^0)_{jl} + 2\gamma_i^0 ) \mB_i^{jl} $; $\mC_i^0 = \sum_{jl} 2 \gamma_i^0 \mB_i^{jl}$; $\mA^0 = \tfrac{1}{n} \sum_{i=1}^n \mA_i^0$; $\mC^0 = \tfrac{1}{n} \sum_{i=1}^n \mC_i^0$; $\beta_i^{0} = \max_{jl} \tfrac{{\tilde h}^i(\nabla^2 f_i(w_i^0))_{jl} + 2 \gamma_i^{0}}{(\mL_i^{0})_{jl} + 2 \gamma_i^{0} }$; $\beta^0 = \max_i \{  \beta^0_i  \}$; $g_{i,1}^{0} = \mA_i^{0} w_i^{0}$; $g_{i,2}^{0} = \mC_i^{0} w_i^{0} + \nabla f_i(w_i^{0})$;  
		$g^0_1 = \tfrac{1}{n} \sum_{i=1}^n g_{i,1}^0$; $g^0_2 = \tfrac{1}{n} \sum_{i=1}^n g_{i,2}^0$; $\mH^0 = \beta^0 \mA^0 - \mC^0$; $g^0 = \beta^0g_1^0 - g_2^0$
		
		\STATE \textbf{on} server
		\STATE ~~~ $x^{k+1} = \left(  \mH^k   \right)^{-1} g^k$ \hfill { \scriptsize Main step: Update the global model}
		\STATE ~~~ {\color{blue} Choose a subset $S^{k} \subseteq \{1,\dots, n\}$ such that $\mathbb{P}[ i \in S^k] = \nicefrac{\tau}{n}$ for all $i\in [n]$}
		\STATE ~~~ $v_i^k = \cQ_i^k(x^{k+1} - z_i^k)$, \quad $z_i^{k+1} = z_i^k + \eta v_i^k$ for $i \in S^k$ 
		\STATE ~~~ $z_i^{k+1} = z_i^k$, \quad  $w_i^{k+1}=w_i^k$ for $i \notin S^k$ 
		\STATE ~~~ Send $v_i^k$ to {\color{blue} the selected devices $i\in S^k$} \hfill { \scriptsize Communicate to selected clients}
		\FOR{each node $i = 1, \dots, n$} 
		\STATE {\color{blue} {\bf for participating devices} $i \in S^k$ {\bf do} }
		\STATE $z_i^{k+1} = z_i^k + \eta v_i^k$, $\mL_i^{k+1} = \mL_i^k + \alpha \cC_i^k \left(  {\tilde h}^i(\nabla^2 f_i(z_i^{k+1}) )  - \mL_i^k  \right)$, \quad $\gamma_i^{k+1} = \max_{jl} \{  c, |(\mL_i^{k+1})_{jl}| \}$ 
		\STATE {\color{blue}\textit{Option 1:}} $\beta_i^{{k+1}} = \max_{jl} \tfrac{{\tilde h}^i(\nabla^2 f_i(z_i^k))_{jl} + 2 \gamma_i^{k+1}}{(\mL_i^{k+1})_{jl} + 2 \gamma_i^{k+1} }$ \quad  {\color{blue}\textit{Option 2:}} $\beta_i^{{k+1}} = \max_{jl} \tfrac{{\tilde h}^i(\nabla^2 f_i(z_i^{k+1}))_{jl} + 2 \gamma_i^{k+1}}{(\mL_i^{k+1})_{jl} + 2 \gamma_i^{k+1} }$ 
		\STATE $\mA_i^{k+1} = \mA_i^k + \sum_{jl}\left(  (\mL_i^{k+1})_{jl} - (\mL_i^k)_{jl} + 2 \gamma_i^{k+1} - 2 \gamma_i^k \right) \mB_i^{jl}$
		\STATE $\mC_i^{k+1} = \mC_i^k + \sum_{jl}\left( 2 \gamma_i^{k+1} - 2 \gamma_i^k  \right) \mB_i^{jl}$, \quad Sample $\xi_i^{k+1} \sim \text{Bernoulli}(p)$
		\STATE {\color{blue}\textbf{if} $\xi_i^k=1$ }
		\STATE ~~~ $w_i^{k+1} = z_i^{k+1}$, $g_{i,1}^{k+1} = \mA_i^{k+1} w_i^{k+1}$, $g_{i,2}^{k+1} = \mC_i^{k+1} w_i^{k+1} + \nabla f_i(w_i^{k+1})$ 
		\STATE ~~~ Send $g_{i,1}^{k+1} - g_{i,1}^k$, $g_{i,2}^{k+1} - g_{i,2}^k$ to server 
		\STATE {\color{blue}\textbf{if} $\xi_i^k=0$ }
		\STATE ~~~ $w_i^{k+1} = w_i^k$, $g_{i,1}^{k+1} = \mA_i^{k+1} w_i^{k+1}$, $g_{i,2}^{k+1} = \mC_i^{k+1} w_i^{k+1} + \nabla f_i(w_i^{k+1})$  
		\STATE Send $\mL_i^{k+1} - \mL_i^k$, $\beta_i^{k+1}$, $\xi_i^k$, $\gamma_i^{k+1} - \gamma_i^k$ to server 
		
		\STATE {\color{blue} {\bf for non-participating devices} $i \notin S^k$ {\bf do} }
		\STATE $z_i^{k+1} = z_i^k$, $w_i^{k+1} = w_i^k$, $\mL_i^{k+1} = \mL_i^k$, $\gamma_i^{k+1} = \gamma_i^k$, $\beta_i^{k+1} = \beta_i^k$,  $\mA_i^{k+1} = \mA_i^k$, $\mC_i^{k+1} = \mC_i^k$, $g_{i,1}^{k+1} = g_{i ,1}^k$, $g_{i,2}^{k+1} = g_{i,2}^k$
		
		\ENDFOR
		\STATE \textbf{on} server
		\STATE ~~~ {\color{blue}\textbf{if} $\xi_i^k=1$ }
		\STATE ~~~~~~ $w_i^{k+1} = z_i^{k+1}$, Receive $g_{i,1}^{k+1}-g_{i,1}^k$, $g_{i,2}^{k+1}-g_{i,2}^k$, 
		\STATE ~~~ {\color{blue}\textbf{if} $\xi_i^k=0$ } 
		\STATE ~~~~~~ $w_i^{k+1} = w_i^k$, $g_{i,1}^{k+1}-g_{i,1}^k = \sum_{jl}(\mL_i^{k+1} - \mL_i^k)_{jl} \mB_i^{jl} w_i^{k+1} + 2(\gamma_i^{k+1} - \gamma_i^k)w_i^{k+1}$
		\STATE ~~~~~~  $g_{i,2}^{k+1}-g_{i,2}^k = \sum_{jl} 2(\gamma_i^{k+1} - 2\gamma_i^k)\mB_i^{jl}w_i^{k+1}$ 
		\STATE ~~~ $g_1^{k+1} = g_1^k + \tfrac{1}{n} \sum_{i\in S^k} \left(  g_{i, 1}^{k+1} - g_{i,1}^k  \right) $, $g_2^{k+1} = g_2^k + \tfrac{1}{n} \sum_{i\in S^k} \left(  g_{i, 2}^{k+1} - g_{i,2}^k  \right)$, \ $\beta^{k+1} = \max_{i} \{  \beta_i^{k+1}  \}$
		\STATE ~~~ $g^{k+1} = \beta^{k+1} g_1^{k+1} - g_2^{k+1}$, \ $\mA^{k+1} = \mA^k + \tfrac{1}{n}\sum_{i\in S^k} \sum_{jl}\left(  (\mL_i^{k+1})_{jl} - (\mL_i^k)_{jl} + 2 \gamma_i^{k+1} - 2 \gamma_i^k  \right) \mB_i^{jl}$
		\STATE ~~~ $\mC^{k+1} = \mC^k + \tfrac{1}{n} \sum_{i \in S^k} \sum_{jl}\left( 2 \gamma_i^{k+1} - 2 \gamma_i^k  \right) \mB_i^{jl}$, \quad $\mH^{k+1} = \beta^{k+1}\mA^{k+1}  - \mC^{k+1}$ 
	\end{algorithmic}
\end{algorithm}

We choose a basis $\{\mB_i^{jl}\}$ in ${\cal S}^d$ such that $\mB_i^{jl} \succeq 0$ for \algname{BL3} (Algorithm~\ref{alg:BL3}). The way to guarantee the positive definiteness of the Hessian estimator is similar to the approach of \citet{Islamov2021NewtonLearn}. From the definition of $\gamma_i^k$, we know $(\mL_i^k)_{jl} + 2\gamma_i^k\geq c>0$. Noticing that $\nabla^2 f_i(z_i^k)$ can be expressed in the form 
$$
\sum_{jl} \left(  \frac{{\tilde h}^i(\nabla^2 f_i(z_i^k))_{jl} + 2\gamma_i^k}{(\mL_i^k)_{jl} + 2\gamma_i^k} \cdot ((\mL_i^k)_{jl} + 2\gamma_i^k) -2\gamma_i^k  \right) \mB_i^{jl}, 
$$
for $\beta_i^k$ in Option 2, we have the inequality 
\begin{align*}
	& \sum_{jl} \left( \beta^k  ((\mL_i^k)_{jl} + 2\gamma_i^k)  -2\gamma_i^k  \right) \mB_i^{jl} - \nabla^2 f_i(z_i^k)  = \sum_{jl} \left( \beta^k - \tfrac{{\tilde h}^i(\nabla^2 f_i(z_i^k))_{jl} + 2\gamma_i^k}{(\mL_i^k)_{jl} + 2\gamma_i^k} \right) \cdot ((\mL_i^k)_{jl} + 2\gamma_i^k)   \mB_i^{jl}  \succeq \mathbf{0}. 
\end{align*}
Thus, if we can maintain the Hessian estimator in the form $\mH_i^k \eqdef \sum_{jl} \left( \beta^k  ((\mL_i^k)_{jl} + 2\gamma_i^k)  -2\gamma_i^k  \right) \mB_i^{jl}$, then $\mH_i^k \succeq \nabla^2 f_i(z_i^k)$ (we can get $\mH_i^k \succeq \nabla^2 f_i(z_i^{k-1})$ for Option 1 similarly). To achieve this goal, we use two auxiliary matrices $\mA_i^k$, $\mC_i^k$, and maintain $\mA_i^k = \sum_{jl}   ((\mL_i^k)_{jl} + 2\gamma_i^k)   \mB_i^{jl}$, $\mC_i^k = \sum_{jl} 2\gamma_i^k \mB_i^{jl}$, and $\mH_i^k = \beta^k \mA_i^k - \mC_i^k$. \algname{BL3} follows the same structure of \algname{BL2}, thus we also need to keep the relation $g_i^k = \mH_i^k w_i^k - \nabla f_i(w_i^k)$, which is actually $g_i^k = \beta^k \mA_i^k w_i^k - \mC_i^k w_i^k - \nabla f_i(w_i^k)$. Since for non-participating devices, $\beta^k$ usually changes at each step, we split $g_i^k$ to two parts by using two auxiliary vectors $g_{i,1}^k$, $g_{i,2}^k$, and keeping $g_{i,1}^k = \mA_i^kw_i^k$, $g_{i,2}^k=\mC_i^kw_i^k - \nabla f_i(w_i^k)$, and $g_i^k = \beta^k g_{i,1}^k - g_{i,2}^k$. The rest of \algname{BL3} is the same as \algname{BL2}. We need the following assumption for \algname{BL3}.

\begin{assumption}\label{as:BL3}
	Assume $\|\nabla^2 f_i(x) - \nabla^2 f_i(y)\| \leq H\|x-y\|$ for any $x,y\in \R^d$ and $i\in[n]$. Assume $\max_{jl}\{|(\mL_i^k)_{jl}|\} \leq M_3$ for all $i\in[n]$ and $k\geq 0$. Assume $\|{\tilde h}^i(\nabla^2 f_i(x)) - {\tilde h}^i(\nabla^2 f_i(y))\|_{\rm F} \leq M_4\|x-y\|$, $\max_{jl} \{|{\tilde h}^i(\nabla^2 f_i(x))_{jl} - {\tilde h}^i(\nabla^2 f_i(y))_{jl} | \} \leq M_5 \|x-y\|$ for any $x,y\in \R^d$ and $i\in[n]$, and $\max_{jl}\{\|\mB_i^{jl}\|_{\rm F}\}\leq R$ for $i\in[n]$. Assume each $f_i$ is $\mu$-strongly convex. 
\end{assumption}

We estimate $M_3$, $M_4$, and $M_5$ in Assumption \ref{as:BL3} in the following lemma. 
\begin{lemma}\label{lm:EstM-2} We have the following bounds on $M_3$, $M_4$, and $M_5$:
\begin{itemize} 
\item[(i)] Assume $\|\nabla^2 f_i(x) - \nabla^2 f_i(y)\|_{\rm F} \leq H_1\|x-y\|$ and $\max_{jl} \{ |(\nabla^2 f_i(x))_{jl} - (\nabla^2 f_i(y))_{jl} | \} \leq \nu \|x-y\|$ for any $x, y\in \R^d$, and $i \in [n]$. Then we have $M_4 \leq  \sqrt{2} \max_i \{ \|({\tilde {\cal B}}_i)^{-1}\| \} H_1$ and $M_5 \leq 2\nu \max_{i} \{\|({\tilde {\cal B}}_i)^{-1}\|_{\infty}\}$.  

\item[(ii)] Assume $\max_{jl} \{|(\nabla^2 f_i(x))_{jl}| \} \leq \gamma$ for any $x\in \R^d$ and $i \in [n]$. If Assumption \ref{as:Cunbiasedcomp-BL1} (ii) holds, then $M_3 \leq 2\gamma \max_{i} \{\|({\tilde {\cal B}}_i)^{-1}\|_{\infty} \}$. 

\item[(iii)] Assume $\|\nabla^2 f_i(x)\|_{\rm F} \leq {\tilde \gamma}$ for any $x\in \R^d$ and $i \in [n]$. If Assumption \ref{as:Ccontractioncomp-BL1} holds and $\|\mL_i^0\|_{\rm F} \leq \frac{\sqrt{2B}}{\sqrt{A}} \|({\tilde {\cal B}}_i)^{-1}\| {\tilde \gamma}$ for all $i \in [n]$, then we have $\|\mL_i^k\|_{\rm F} \leq \frac{\sqrt{2B}}{\sqrt{A}} \|({\tilde {\cal B}}_i)^{-1}\| {\tilde \gamma}$ for $k\geq 0$ and $i\in [n]$, and $M_3 \leq \frac{\sqrt{2B} {\tilde \gamma}}{\sqrt{A}} \max_{i} \{ \|({\tilde {\cal B}}_i)^{-1}\| \}$. 
\end{itemize}
\end{lemma}

Let $\Phi_5^k \eqdef {\cal W}^k + \tfrac{2p}{A_{\rm M}}\left(  1 - \tfrac{\tau A_{\rm M}}{n}  \right) {\cal Z}^k$, where ${\cal Z}^k \eqdef \tfrac{1}{n} \sum_{i=1}^n \|z_i^k-x^*\|^2$, for $k\geq 0$. 

\begin{theorem}[Linear convergence of \algname{BL3}]\label{th:BL3}
	Let Assumption \ref{as:BL3} hold. Let Assumption \ref{as:Qunbiasedcomp-BL1} (i) or Assumption \ref{as:Qcontractioncomp-BL1} (i) hold. Assume $\|z_i^k-x^*\|^2 \leq \tfrac{A_{\rm M}\mu^2}{4(H^2 + 4c_1)B_{\rm M}}$ and ${\cal H}^k \leq \tfrac{A_{\rm M}\mu^2}{16c_2 B_{\rm M}}$ for all $i\in[n]$ and $k\geq 0$, where $c_1 \eqdef \tfrac{4N^2 R^2 M_5^2 (M_3 + 2\max\{c, M_3\})^2}{c^2}$ and $c_2 \eqdef {2NR^2} \left(  1 + \tfrac{2N(M_3 + 2\max\{c, M_3\})^2}{c^2}  \right)$. Then we have 
	$$
	\mathbb{E}[\Phi_5^k] \leq \left(  1 - \frac{\tau \min\{  p, A_{\rm M}  \}}{2n}  \right)^k \Phi_5^0,
	$$
	for $k\geq 0$. 
\end{theorem}

Define $\Phi_6^k \eqdef {\cal H}^k + \tfrac{4BM_4^2}{A_{\rm M}}\|x^k-x^*\|^2 $ for $k\geq 0$. 

\begin{theorem}[Superlinear convergence of \algname{BL3}]\label{th:superlinear-BL3}
	Let $\eta=1$, $\xi^k \equiv 1$, $S^k \equiv [n]$, and ${\cal Q}_i^k(x) \equiv x$ for any $x\in \R^d$ and $k\geq 0$. Let Assumption \ref{as:BL3} hold. Let Assumption \ref{as:Cunbiasedcomp-BL1} (i) or Assumption \ref{as:Ccontractioncomp-BL1} (i) hold. Assume $\|z_i^k-x^*\|^2 \leq \tfrac{A_{\rm M}\mu^2}{4(H^2 + 4c_1)B_{\rm M}}$ and ${\cal H}^k \leq \tfrac{A_{\rm M}\mu^2}{16c_2 B_{\rm M}}$ for all $i\in[n]$ and $k\geq 0$. Then we have 
	$$
	\mathbb{E}[\Phi_6^k] \leq \theta_3^k \Phi_6^0, 
	$$
	for $k\geq 0$, where $\theta_3 \eqdef \left(  1 - \tfrac{\min\{  2A, A_{\rm M}  \} }{2} \right)$. Moreover, for Option 1, we have 
	$$
	\mathbb{E} \left[   \tfrac{\|x^{k+1}-x^*\|^2}{\|x^k-x^*\|^2} \right] \leq \theta_3^k \left(  \frac{A_{\rm M}(H^2 \theta_3+4c_1)}{8BM_4^2\mu^2 \theta_3} + \frac{2c_2}{\mu^2}  \right) \Phi_6^0, 
	$$
	and for Option 2, we have 
	$$
	\mathbb{E} \left[   \frac{\|x^{k+1}-x^*\|^2}{\|x^k-x^*\|^2} \right] \leq \theta_3^k \left(  \frac{A_{\rm M}(H^2 + 4c_1)}{8BM_4^2\mu^2} + \frac{2c_2}{\mu^2}  \right) \Phi_6^0, 
	$$
	for $k\geq 0$. 
\end{theorem}

Next, we explore under what conditions we can guarantee the boundedness of $\|z_i^k-x^*\|^2$ and ${\cal H}^k$. 

\begin{theorem}\label{th:nbor-BL3}
	Let Assumption \ref{as:BL3} hold. Then we have the following results. 
\begin{itemize}	
\item[(i)] Let Assumption \ref{as:Qunbiasedcomp-BL1} and Assumption \ref{as:Cunbiasedcomp-BL1} (ii) hold. If $$\|x^0-x^*\|^2 \leq \min \left\{  \frac{\mu^2}{4d^2(H^2+4c_1)}, \frac{\mu^2}{16d^4 c_2M_5^2} \right\},$$ then $\|z_i^k-x^*\|^2 \leq \min \left\{ \tfrac{\mu^2}{4d(H^2+4c_1)},  \tfrac{\mu^2}{16d^3 c_2M_5^2}  \right\}$ and ${\cal H}^k \leq \tfrac{\mu^2}{16dc_2 }$ for $i\in[n]$ and $k\geq 0$. 
\item[(ii)] Let Assumption \ref{as:Qcontractioncomp-BL1} and Assumption \ref{as:Ccontractioncomp-BL1} hold. If $\|z_i^0 - x^*\|^2 \leq \min\left\{  \tfrac{A_{\rm M}\mu^2}{4B_{\rm M}(H^2+4c_1)},  \tfrac{AA_{\rm M}\mu^2}{16c_2 B_{\rm M}BM_4^2}  \right\}$ and $\|\mL_i^0 - \mL_i^*\|^2_{\rm F} \leq \tfrac{A_{\rm M}\mu^2}{16c_2B_{\rm M}}$ for all $i\in[n]$, then $\|z_i^k - x^*\|^2 \leq \min\left\{  \tfrac{A_{\rm M}\mu^2}{4B_{\rm M}(H^2+4c_1)},  \tfrac{AA_{\rm M}\mu^2}{16c_2 B_{\rm M}BM_4^2}  \right\}$ and $\|\mL_i^k - \mL_i^*\|^2_{\rm F} \leq \tfrac{A_{\rm M}\mu^2}{16c_2B_{\rm M}}$ for all $i\in[n]$ and $k\geq 0$. 
\end{itemize}
\end{theorem}

\section{Experiments}

We conduct numerical experiments to compare the performance of \algname{BL} methods with various efficient methods in federated learning. We consider regularized logistic regression problem
\begin{equation}
	\squeeze
	\min\limits_{x\in \mathbb{R}^d}\left\{f(x)\eqdef \frac{1}{n}\sum\limits_{i=1}^nf_i(x) + \frac{\lambda}{2}\|x\|^2\right\},
\end{equation}
where $$f_i(x) \eqdef \frac{1}{m}\sum_{j=1}^m\log\(1+\exp\(-b_{ij}a_{ij}^\top x\)\).$$ Here $\{a_{ij}, b_{ij}\}_{j\in [m]}$ are data points stored on $i$-th device. In our experiments, we used the following datasets from LibSVM \citep{chang2011libsvm}: \dataname{a1a}, \dataname{a9a}, \dataname{phishing}, \dataname{w2a}, \dataname{w8a}, \dataname{covtype}, \dataname{madelon}. We use two values of the regularization parameter: $\lambda \in \{10^{-3}, 10^{-4}\}$. In the figures we plot the
relation of the optimality gap $f(x^k) - f(x^*)$
and the number of communicated bits per node. The optimal value $f(x^*)$ is chosen as the function value at the
$20$-th iterate of standard Newton’s method.

\subsection{Basis computation for \algname{BL}}
One of the most popular types of data preprocessing in classical machine learning is dimension reduction. One of such techniques is based on SVD of the feature matrix. We want to point out that SVD could also be used to find a basis for each client. In our experiments, we use linalg.orth function from SciPy module \citep{scipy}. In other words, such data preprocessing could be used not only for the stability of certain machine learning models but also for improving the optimization process.


\subsection{Comparison with second-order methods}
We compare the performance of \algname{BL1} with \algname{DINGO} \citep{DINGO}, \algname{FedNL} \citep{FedNL2021}, \algname{NL1} \citep{Islamov2021NewtonLearn}, \algname{N0} \citep{FedNL2021} in terms of communication complexity. For \algname{FedNL}, \algname{NL1}, and \algname{BL1} we use $\nabla^2 f_i(x^0)$ as the initialization of $\mH_i^0$. Besides, the stepsize $\alpha=1$, Rank-$1$ compression for matrices, and option $1$ (projection) were used for \algname{FedNL}. For \algname{NL1} compression mechanism is Rand-$1$ with stepsize $\alpha=\nicefrac{1}{(\omega+1)}.$ Backtracking linesearch for \algname{DINGO} selects the largest stepsize from $\{1, 2^{-1}, 2^{-2}, \dotsc, 2^{-10}\}$. We set the authors' choice for other parameters of the method: $\theta=10^{-4}, \phi = 10^{-6}, \rho=10^{-4}$. Compression operator $\cC_i^k$ in \algname{BL1} is Top-$K$, where $K=r$ ($r$ is the dimension of the local data). We set $p=1$ and use identity compression for $\cQ^k$ with stepsize $\eta=1$ for models (backside compression is not used). According the results in Figure~\ref{fig_main} ($1^{st}$ row), \algname{BL1} is the most efficient method in all cases.


\begin{figure*}[ht]
    \begin{center}
        \begin{tabular}{cccc}
            \includegraphics[width=0.23\linewidth]{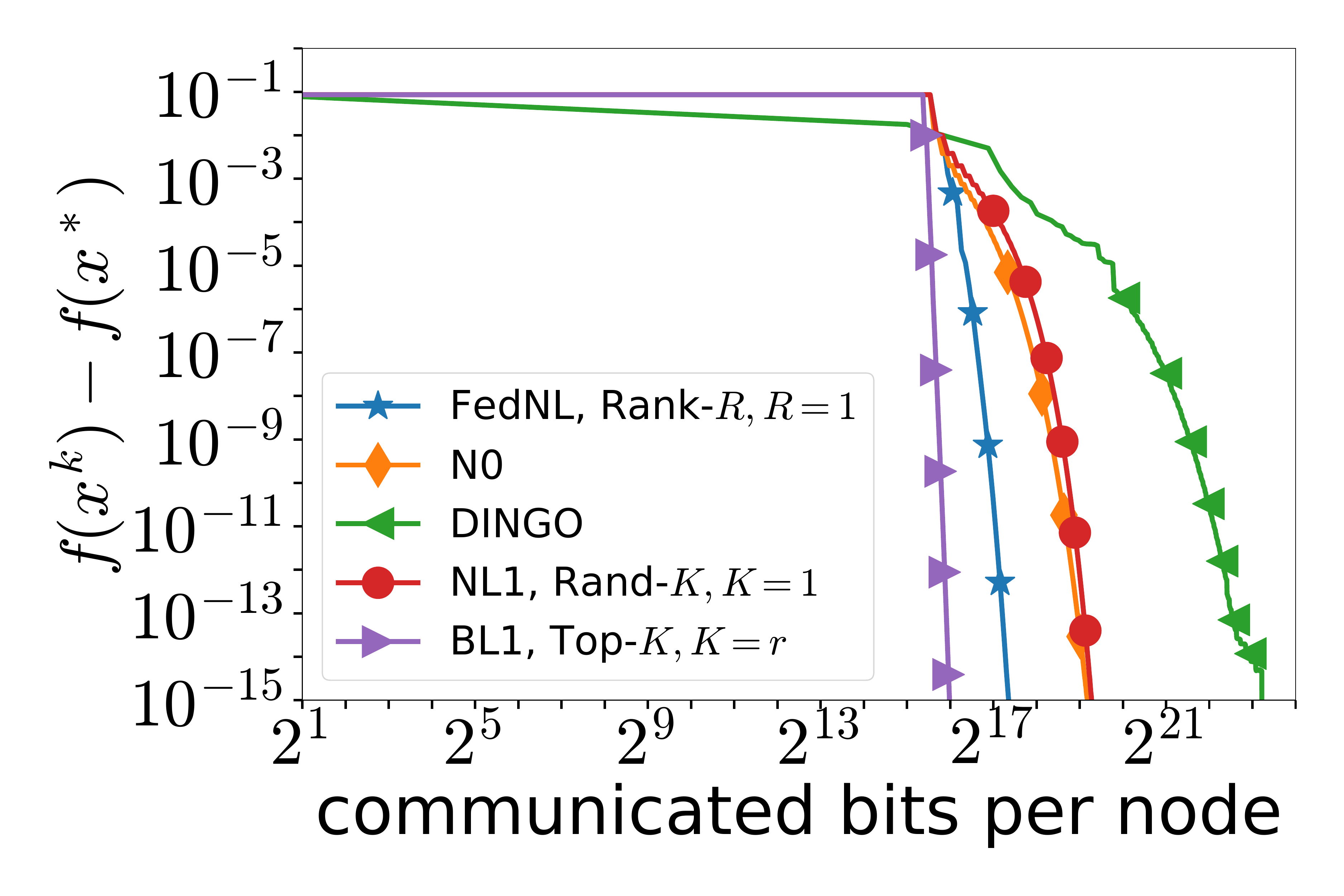} & 
            \includegraphics[width=0.23\linewidth]{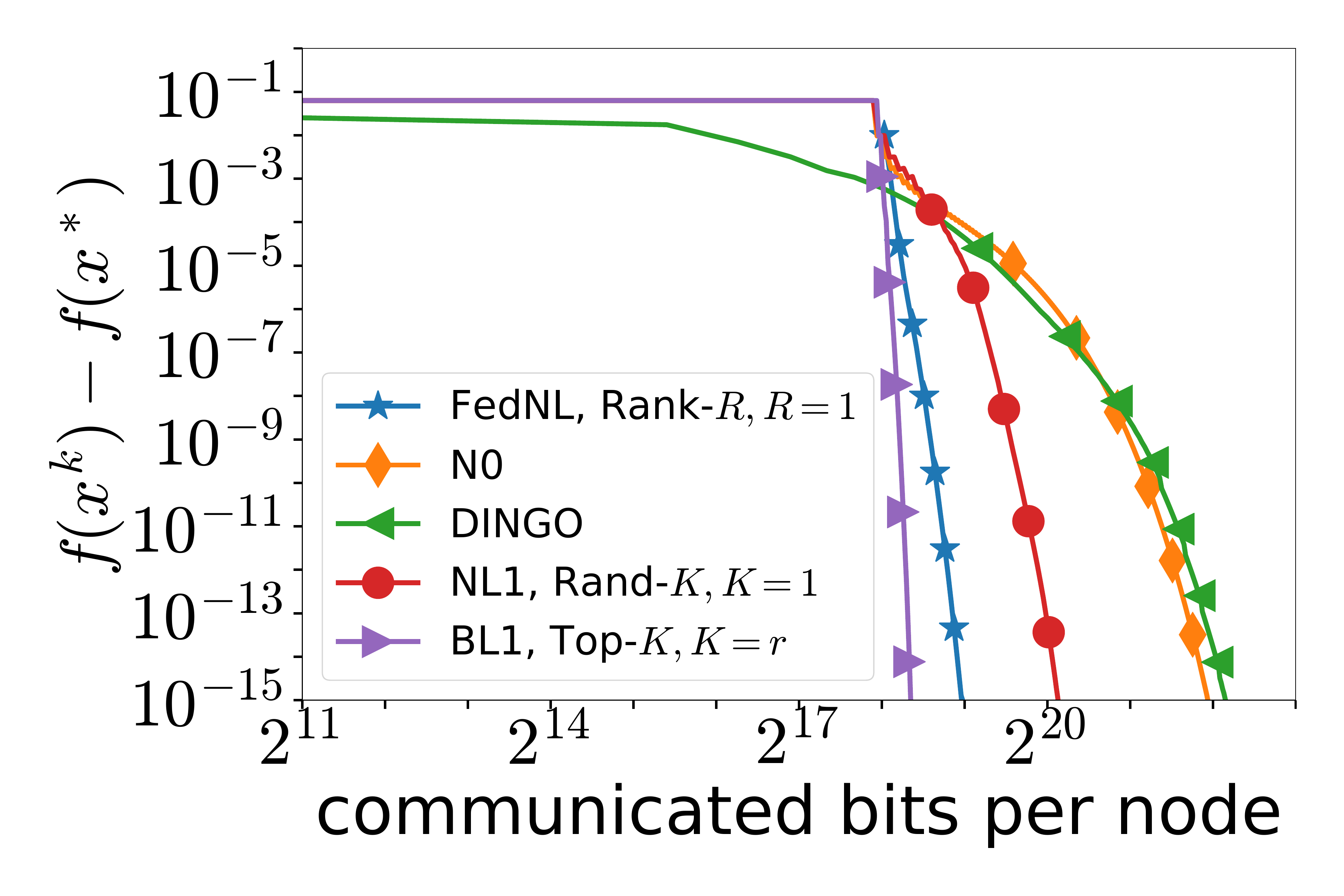} &
            \includegraphics[width=0.23\linewidth]{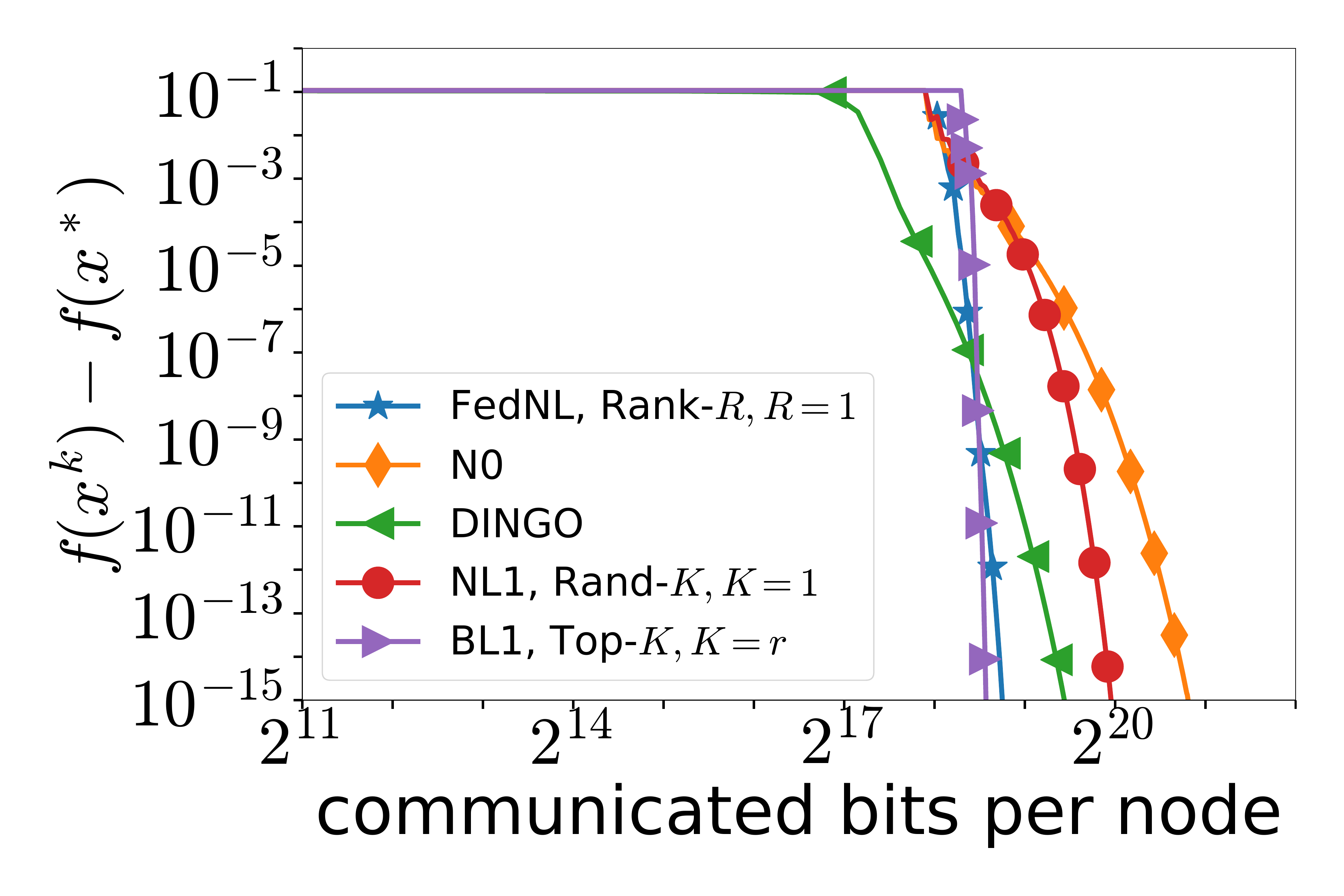} &
            \includegraphics[width=0.23\linewidth]{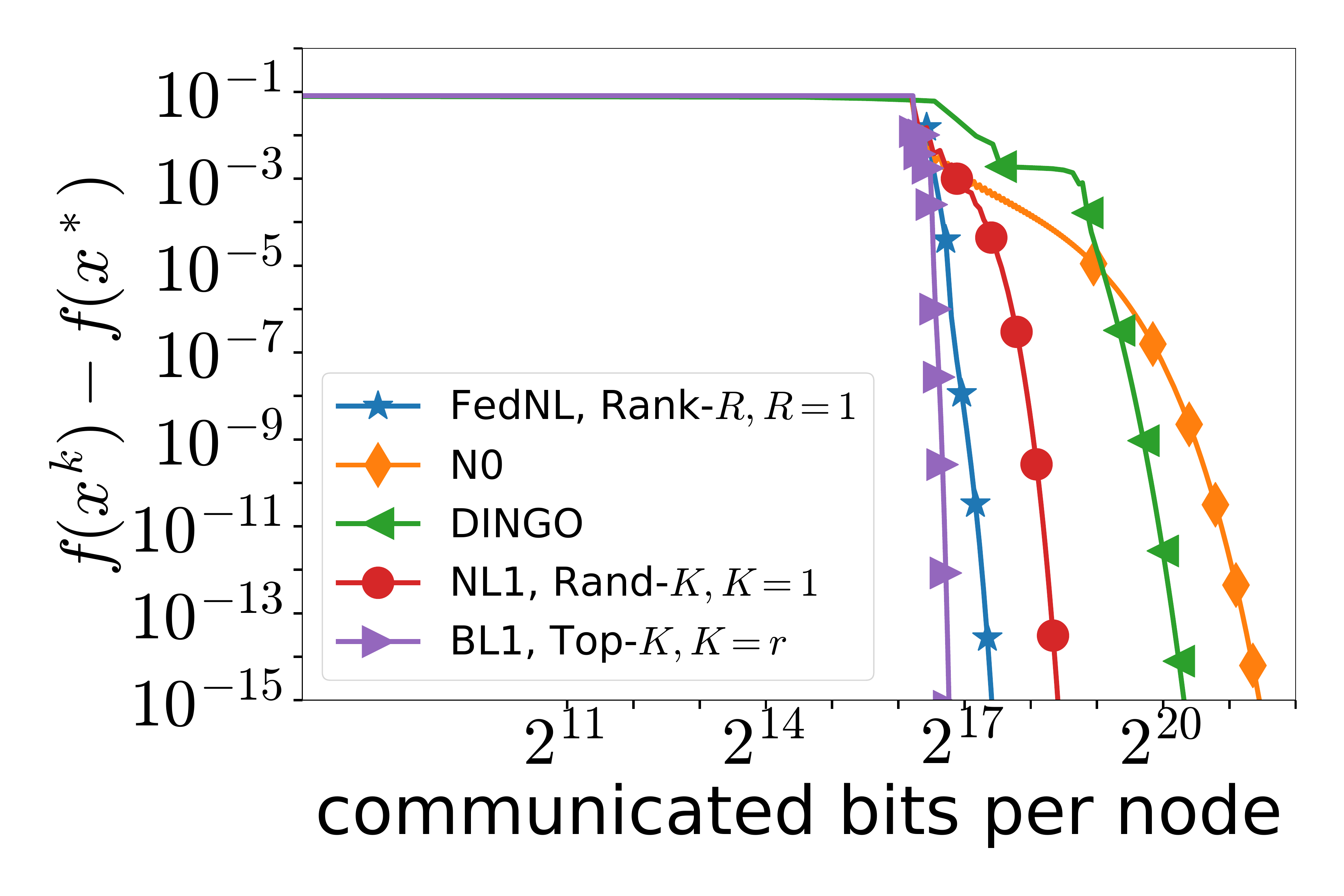}\\
            (a) \dataname{covtype}, {\scriptsize  $\lambda=10^{-3}$} &
            (b) \dataname{a1a}, {\scriptsize  $\lambda=10^{-4}$} &
            (c) \dataname{a9a}, {\scriptsize  $\lambda=10^{-3}$} &
            (d) \dataname{phishing}, {\scriptsize  $\lambda=10^{-4}$}
        \end{tabular}
        \\
        \begin{tabular}{cccc}
            \includegraphics[width=0.23\linewidth]{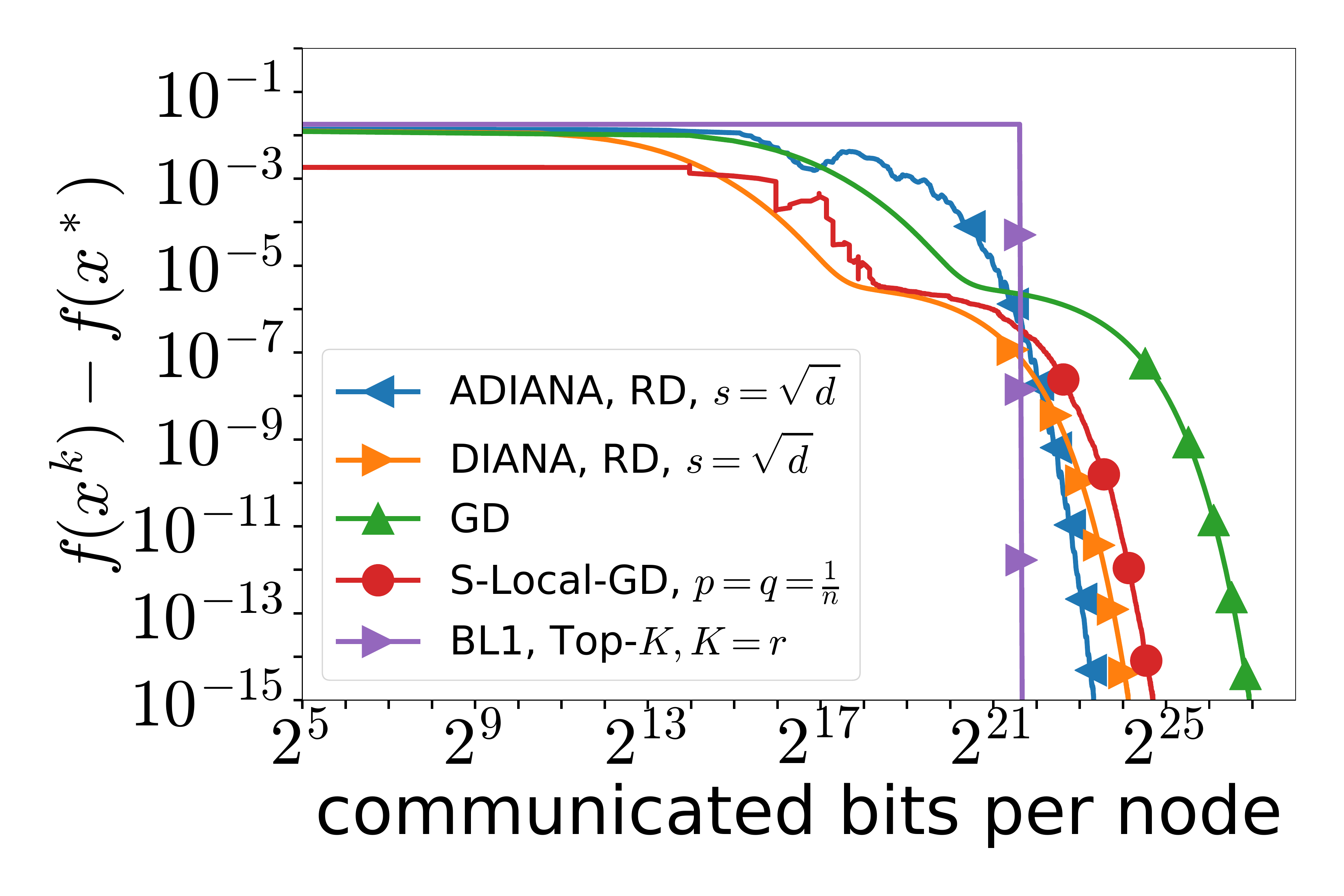} & 
            \includegraphics[width=0.23\linewidth]{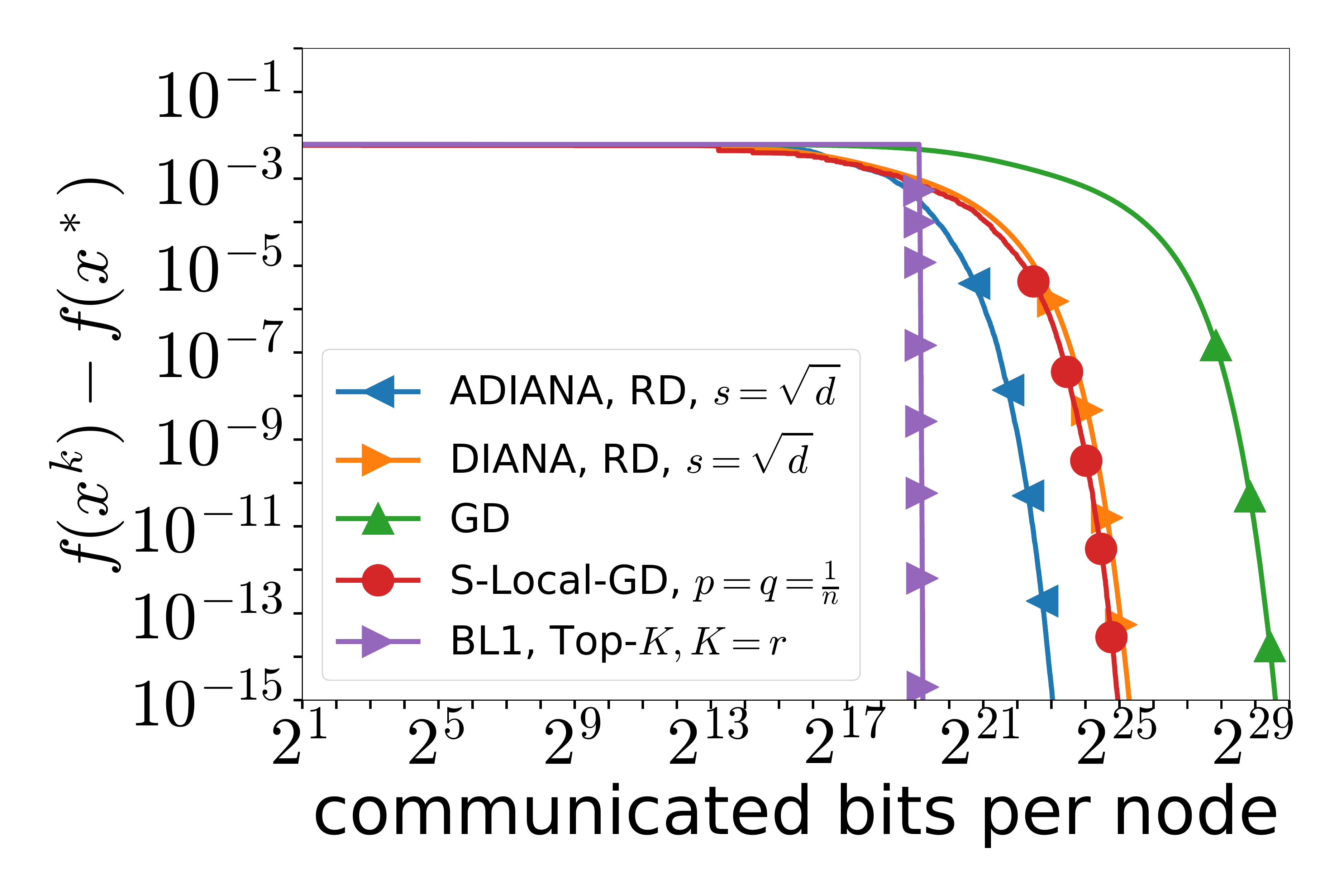} &
            \includegraphics[width=0.23\linewidth]{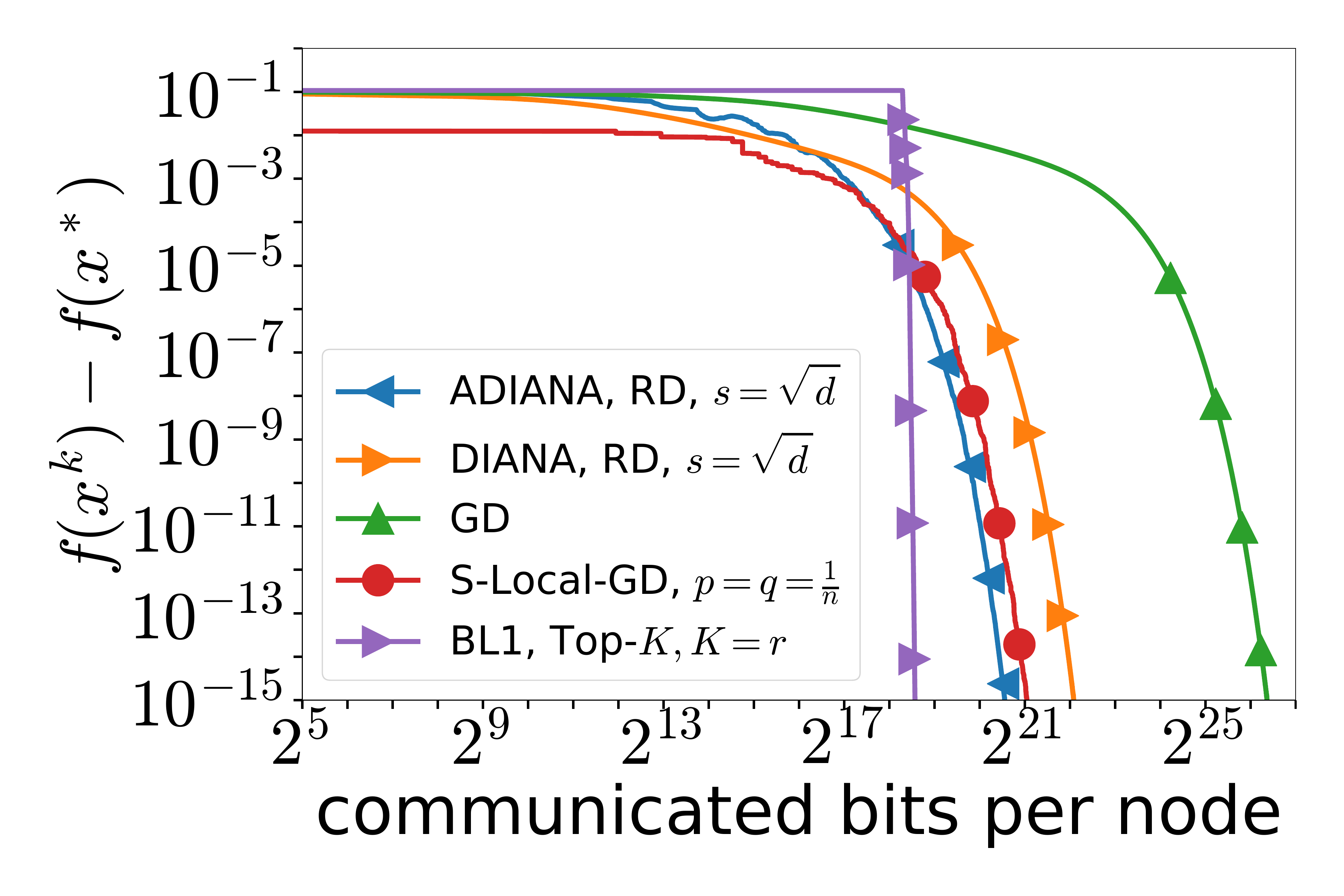} &
            \includegraphics[width=0.23\linewidth]{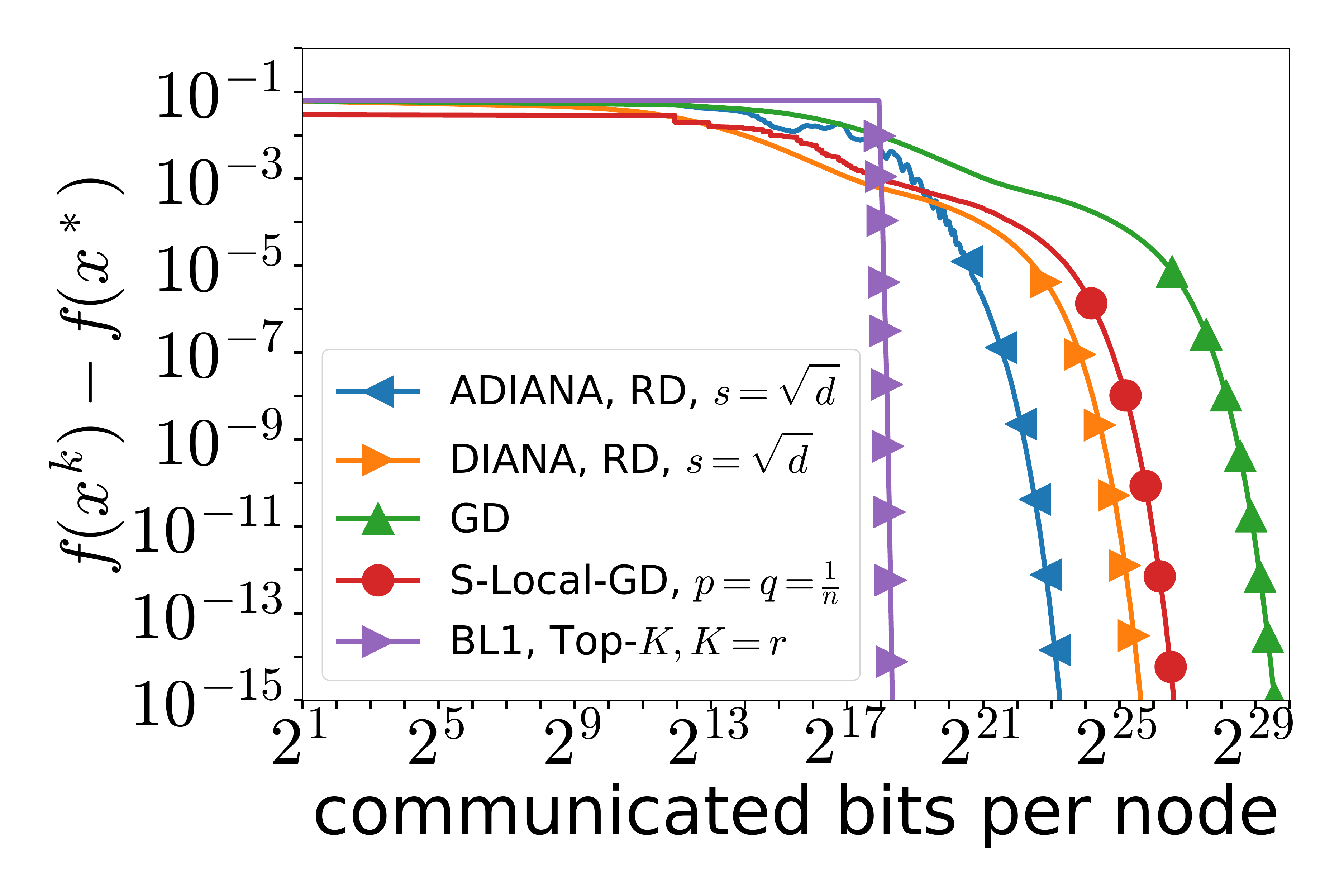} \\
            (a) \dataname{madelon}, {\scriptsize  $\lambda=10^{-3}$} &
            (b) \dataname{w2a}, {\scriptsize  $\lambda=10^{-4}$} &
            (c) \dataname{a9a}, {\scriptsize$ \lambda=10^{-3}$} &
            (d) \dataname{a1a}, {\scriptsize$ \lambda=10^{-4}$}
        \end{tabular}
        \\
        \begin{tabular}{cccc}
            \includegraphics[width=0.23\linewidth]{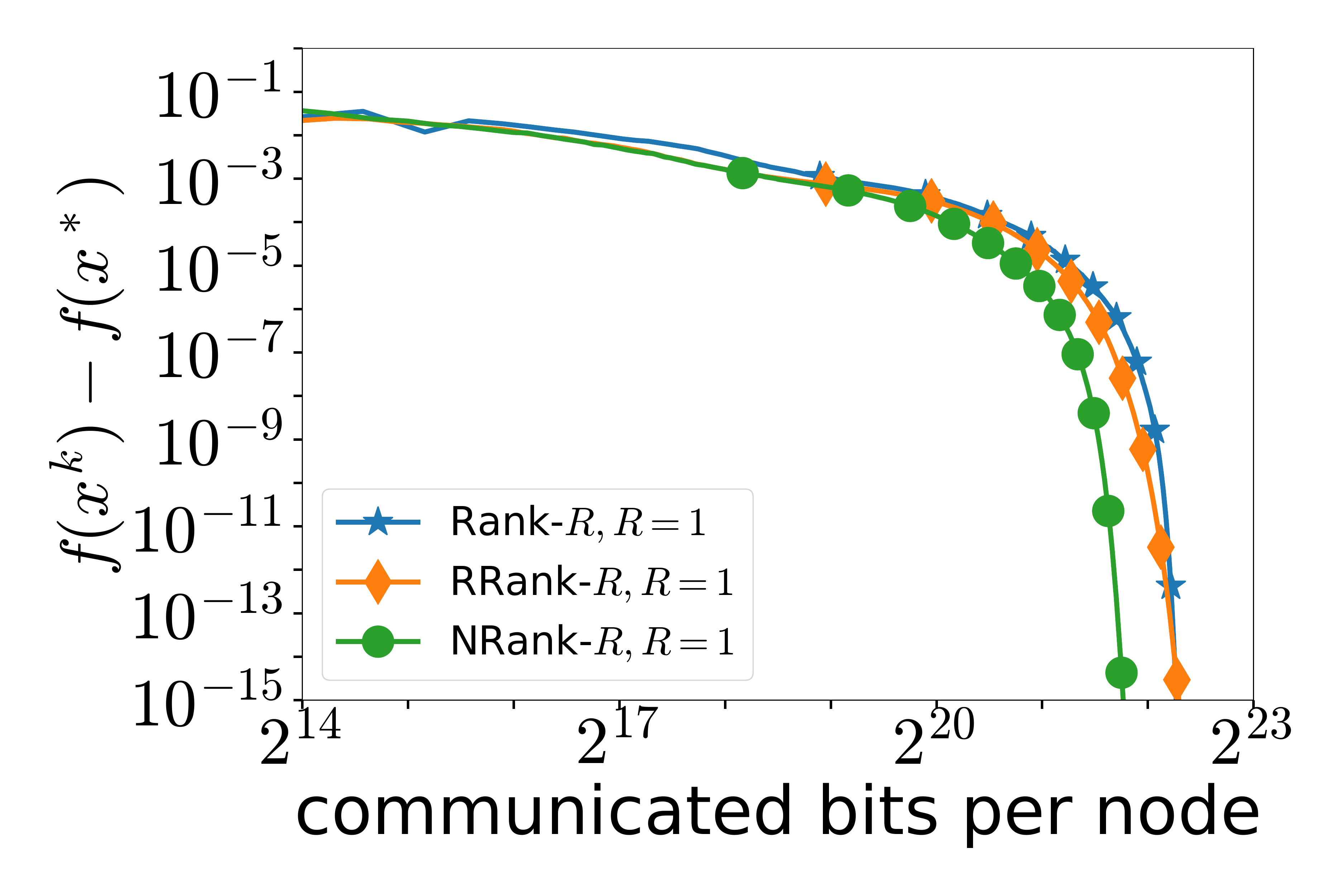} & 
            \includegraphics[width=0.23\linewidth]{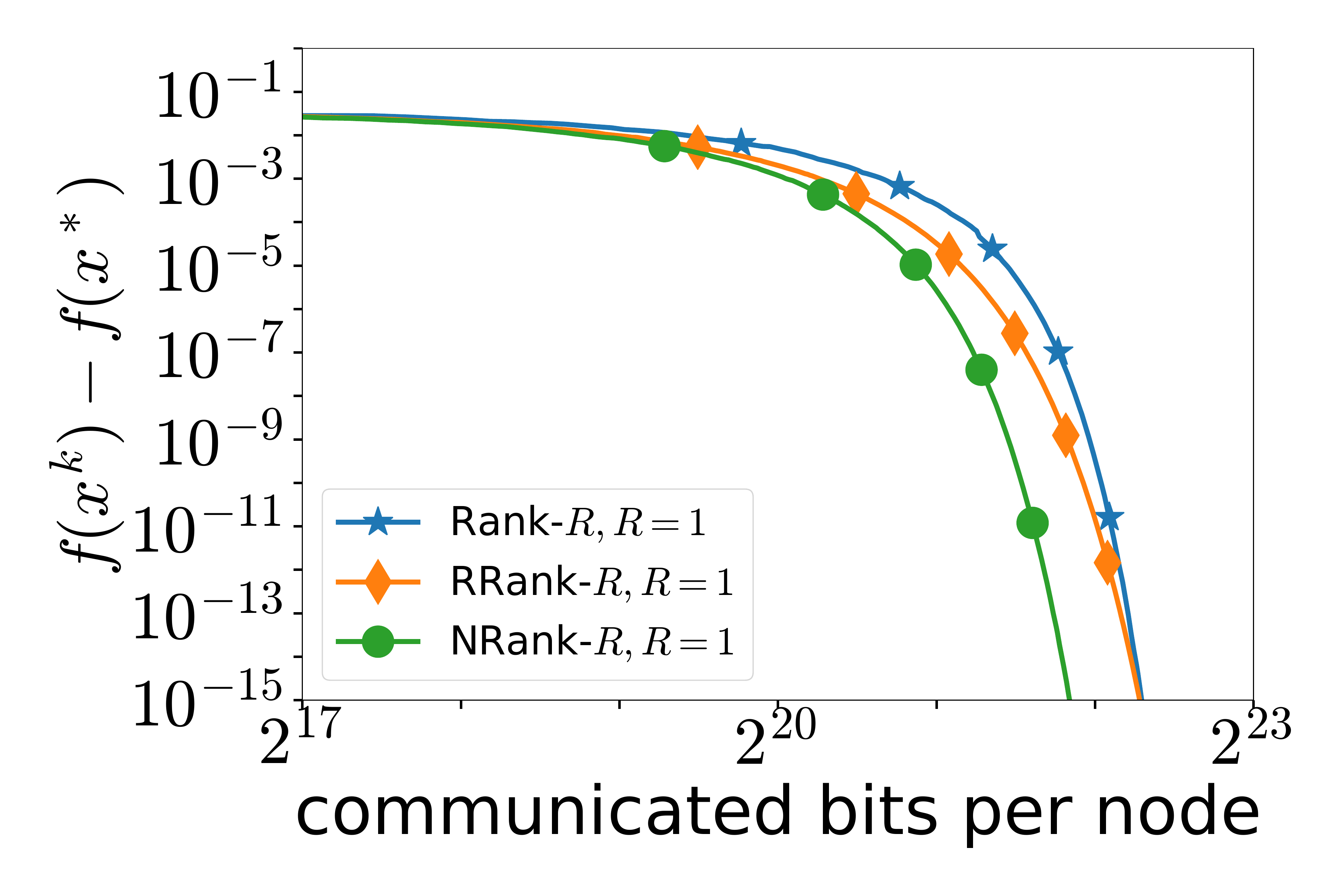} &
            \includegraphics[width=0.23\linewidth]{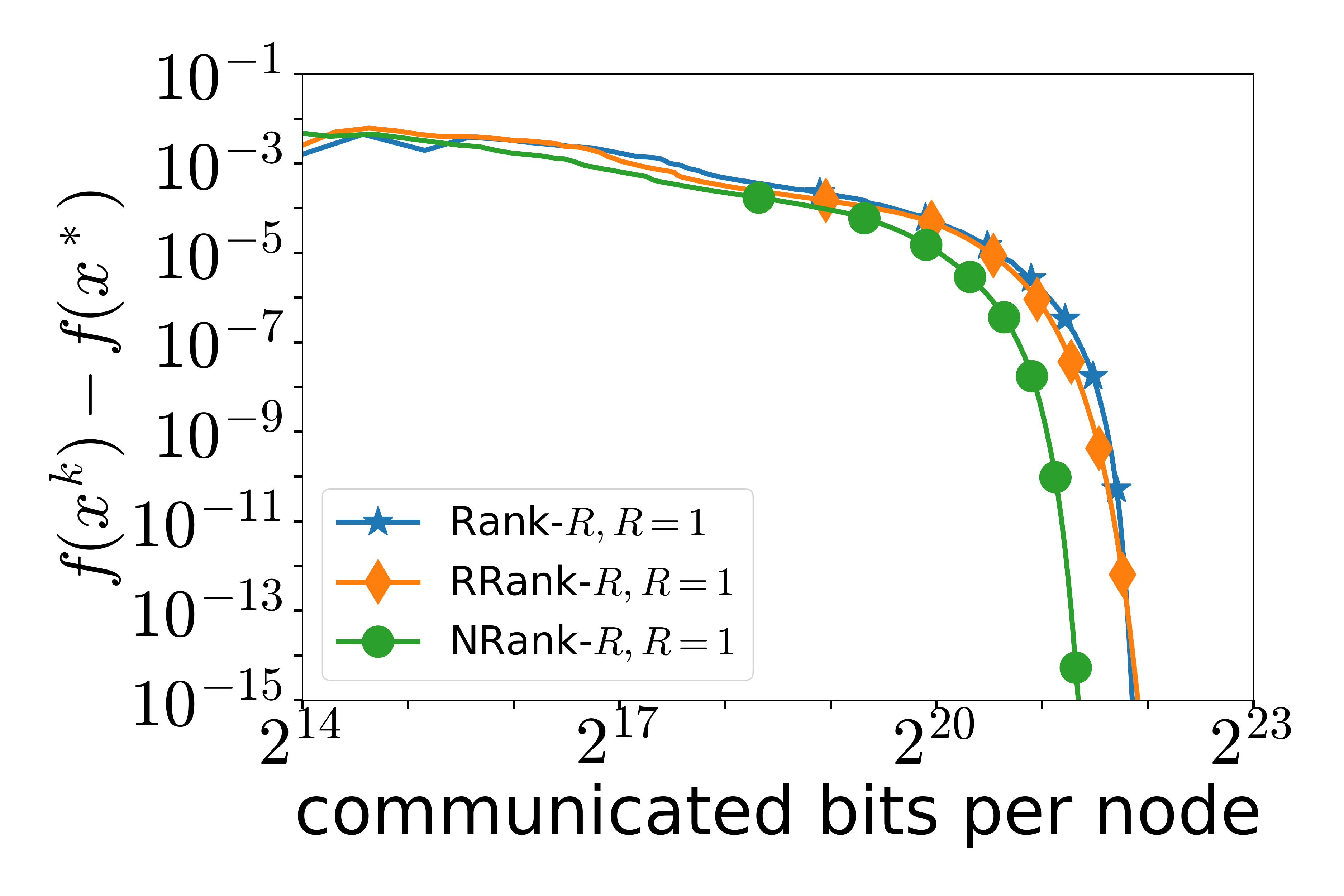} &
            \includegraphics[width=0.23\linewidth]{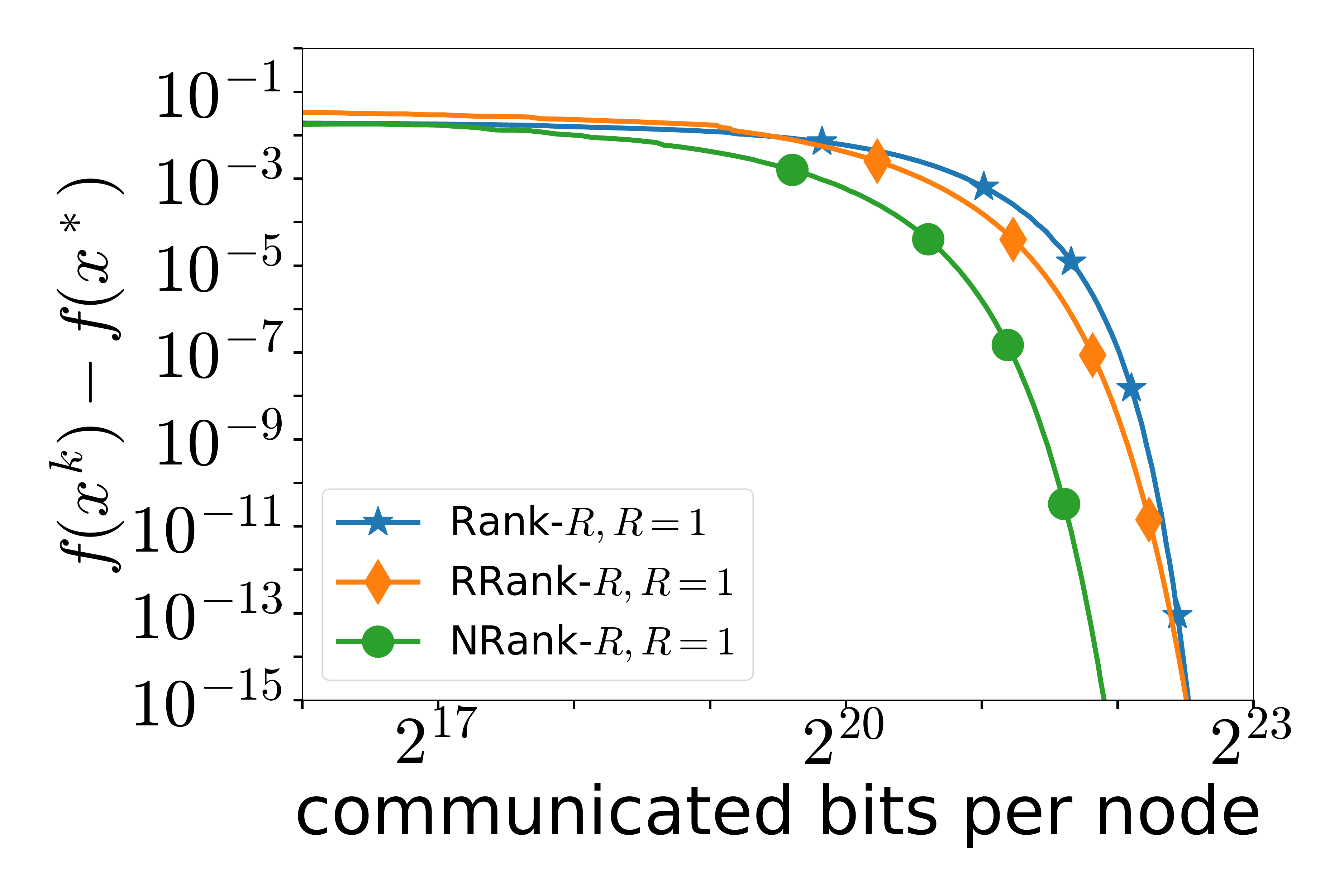} \\
            (a) \dataname{a9a}, {\scriptsize$ \lambda=10^{-4}$} &
            (b) \dataname{w8a}, {\scriptsize $\lambda=10^{-3}$} &
            (c) \dataname{a1a}, {\scriptsize$ \lambda=10^{-4}$} &
            (d) \dataname{w2a}, {\scriptsize$ \lambda=10^{-3}$} 
        \end{tabular}               
    \end{center}
    \caption{Comparison of \algname{BL1} with \algname{N0}, \algname{FedNl}, \algname{NL1}, \algname{DINGO} ({\bf first row}), \algname{DIANA}, \algname{ADIANA}, \algname{GD}, \algname{S-Local-GD} ({\bf second row}) and the performance of \algname{BL2} with compressors Rank-$R$, RRank-$R$, and NRank-$R$ ({\bf third row}) in terms of communication complexity.}
    \label{fig_main}
\end{figure*}

\subsection{Comparison with first-order methods}

Next we compare the performance of \algname{BL1} with vanilla gradient descent (\algname{GD}), \algname{DIANA} \citep{DIANA}, \algname{ADIANA} \citep{ADIANA}, and shifted local gradient descent (\algname{S-Local-GD}) \citep{Gorbunov2020localSGD} in terms of communication complexity. Theoretical stepsizes were chosen for first-order methods. For \algname{DIANA} and \algname{ADIANA} we use random dithering compression \citep{qsgd, DIANA-VR} with $s=\sqrt{d}$ levels. Probabilities $p$ and $q$ are equal to $\nicefrac{1}{n}$ for \algname{S-Local-GD}. Parameters of \algname{BL1} are the same as in the previous section. We clearly see in Figure~\ref{fig_main} ($2^{nd}$ row) that \algname{BL1} is more communication efficient than all gradient type methods by {\it several orders in magnitude.}


\subsection{Composition of compressors}
In our next experiment we analyse the composition of Rank-$R$ and unbiased compression operators; see Section~\ref{sec:3} for more details. We consider \algname{BL2} with $3$ compression mechanisms: Rank-$R$, RRank-$R$ (composition of Rank-$R$ and random dithering with $s=\sqrt{d}$ levels), and NRank-$R$ (composition of Rank-$R$ and natural compression). For all three compressors $R=1$, and initializaion is $\mH^0=\nabla^2 f(x^0)$. Besides, the parameters of \algname{BL2} are the following: $\tau=n$, $p=\frac{1}{10}$. Finally, we use Top-$K$ with $K=\lfloor \frac{d}{10}\rfloor$ for $\cQ_i^k$. In this experiment we use standard basis in the space of matrices which means that \algname{BL2} turns to be \algname{FedNL}. According to numerical results presented in Figure~\ref{fig_main} ($3^{rd}$ row), composition is indeed useful.

\section{Extensions}

In this paper, we consider the basis in $\R^{d\times d}$ and ${\cal S}^d$. It is actually possible to extend Basis Learn to the case where $\{\mB_i^{jl}\}$ is not necessarily a basis in some space. More precisely, if there exist a set $\{\mB_i^j\}_{j\in S^i}$ and a map $h^i : \R^d \to \R^{|S^i|}$ such that for any $x\in\R^d$, $\nabla^2 f_i(x)$ can be represented by $\sum_{j} h^i(x)_j \mB_i^j$ and $h^i$ is $L$-Lipschitz continuous, i.e., $\|h^i(x)-h^i(y)\| \leq L \|x-y\|$ for any $x,y\in \R^d$, then we can get the corresponding algorithm and convergence results in the same way.

\bibliography{references}
\bibliographystyle{plainnat}

\newpage
\appendix

\part*{Appendix}

\section{Extra experiments}

In this section we demonstrate additional numerical experiments comparing \algname{BL} with relevant benchmarks and with state-of-the-art methods. We consider regularized logistic regression problem 

\begin{equation*}
	\min\limits_{x\in \R^d}\left\{\frac{1}{n}\sum\limits_{i=1}^nf_i(x) + \frac{\lambda}{2}\|x\|^2\right\}, \quad \text{where} \quad f_i(x) = \frac{1}{m}\sum\limits_{j=1}^m\log\(1+\exp(-b_{ij}a_{ij}^\top x)\),
\end{equation*}
and $\{a_{ij},b_{ij}\}_{j\in [m]}$ are data samples belonging to the  $i$-th node. 

\subsection{Parameters setting and data sets}

The data sets were taken from LibSVM library \citep{chang2011libsvm}: \dataname{a1a}, \dataname{a9a}, \dataname{phishing}, \dataname{covtype}, \dataname{madelon}, \dataname{w2a}, \dataname{w8a}. Each data set was partitioned across several nodes to cover a variety of scenarios. See Table \ref{tab:datasets} for more detailed description.  

\begin{table}[h]
	\caption{Data sets used in the experiments with the number of worker nodes $n$ used in each case.}
	\label{tab:datasets}
	\centering
	\begin{tabular}{|l|c|c|c|c|}
		\toprule
		{\bf data set} & {\bf \# workers} $n$ & {\bf \# data points} ($=nm$) & {\bf \# features} $d$  &  {\bf average}               \\
		& & & & {\bf dimension $r$} \\
		\midrule
		\dataname{a1a} & $16$ & $1600$ & $123$ & $64$\\ \hline
		\dataname{a9a} & $80$ & $32560$ & $123$ & $82$\\ \hline
		\dataname{phishing} & $100$ & $110$ & $68$ & $35$\\ \hline
		\dataname{covtype} & $200$ & $581000$ & $54$& $24$\\ \hline
		\dataname{madelon} & $10$ & $2000$ & $500$ & $200$\\ \hline
		\dataname{w2a} & $50$ & $3450$ & $300$ & $59$\\ \hline
		\dataname{w8a} & $142$ & $49700$ & $300$ & $133$\\
		\bottomrule
	\end{tabular}
\end{table}

Theoretical parameters were used for gradient type methods: vanilla gradient descent (\algname{GD}), \algname{DIANA} \citep{DIANA}, \algname{ADIANA} \citep{ADIANA}, and local gradient descent (\algname{Local-GD}). The parameter constants for \algname{DINGO} \citep{DINGO} were chosen following authors' choice: $\theta = 10^{-4}, \phi = 10^{-6}, \rho=10^{-4}$. Backtracking line search was used for \algname{DINGO} to find the largest stepsize from $\{1,2^{-1},\dotsc, 2^{-10}\}$. The initialization of $\mH^0_i$ for \algname{NL1} \citep{Islamov2021NewtonLearn} and vanilla \algname{FedNL} \citep{FedNL2021} is $\nabla^2f_i(x^0)$. Besides, for \algname{NL1} we use Rand-$K$ compressor with $K=1$ and the stepsize $\alpha = \frac{1}{\omega+1}$, where $\omega=\frac{m}{K}-1$. For \algname{FedNL} we use option $1$ to make the Hessian approximation to be positive definite (projection onto the cone of positive definite matrices), stepsize $\alpha=1$, and compression operator Rank-$R$ with $R=1$. For \algname{BL3}, we use option 2. 

We carry out experiments for two values of regularization parameter $\lambda \in \{10^{-3}, 10^{-4}, 10^{-5}\}$. In the
figures we plot the optimality gap $f(x^k) - f(x^*)$ versus the number of communicated
bits per node. The optimal value $f(x^*)$ is chosen as the
function value at the 20-th iterate of standard Newton’s method. 

\subsection{Compression operators}

\paragraph{Unbiased compression operator: random dithering.}

In all experiments with \algname{ADIANA} and \algname{DIANA} the compression operator applied on gradient differences is random dithering \citep{qsgd, DIANA-VR}. This compressor has the parameter $s$ (number of levels) and can be defined via the formula

\begin{equation}
	\cC(x) \eqdef \text{sign}(x) \cdot \|x\|_q \cdot \frac{\xi_s}{s},
\end{equation}
where $\|x\|_q \eqdef \(\sum_i |x_i|^q\)^{1/q}$ and $\xi_s \in \R^d$ is a random vector whose $i$-th entire defind as follows
\begin{equation}
	(\xi_s)_i = \begin{cases}
		l+1 & \text{with probability } \frac{|x_i|}{\|x\|_q}s-l,\\
		l & \text{otherwise}.
	\end{cases}
\end{equation}
Here $s \in \N_+$ denotes the levels of rounding, and $l$ satisfies $\frac{|x_i|}{\|x\|_q} \in \[\frac{l}{s}, \frac{l+1}{s}\]$. This compressor has  variance parameter satsfying $\omega \leq 2+\frac{d^{1/2}+d^{1/q}}{s}$ \citep{DIANA-VR}. However, for Euclidean norm ($q=2$) one can improve the bound to $\omega \leq \min\left\{\frac{d}{s^2}, \frac{\sqrt{d}}{s}\right\}$ \citep{qsgd}.

\paragraph{Examples of contractive compression operators for matrices.} 

One of the examples of contractive compression operators is low-rank approximation or Rank-$R$ compressor. This compression operator is based on singular value decomposition of the matrix and belongs to the class of contractive compressors with $\delta=\frac{R}{d}$ \citep{FedNL2021}. Let $\mX \in \R^{d\times d}$ and singular value decomposition of $\mX$ is 
\begin{equation}\label{eq:b97gs_9098fd}
	\mathbf{X} = \sum\limits_{i=1}^d \sigma_i u_iv_i^\top,
\end{equation}
where the singular values $\sigma_i$ are sorted in non-increasing order: $\sigma_1\geq \sigma_2 \geq \cdots \geq \sigma_d$. Then, the Rank-$R$ compressor, for $R \le d$, is  defined by
\begin{equation}\label{eq:n988fdgfd}
	\cC(\mathbf{X}) \eqdef \sum\limits_{i=1}^{R} \sigma_i u_iv_i^\top.
\end{equation}
Note that if the input of Rank-$R$ compressor is a symmetric matrix, then its output is automatically symmetric matrix.

Another popular choice of contractive compressors in practice is Top-$K$. This compressor applied on matrices sorts the entires of input in non-increasing order by magnitude, and then selects $K$ maximal elements. Top-$K$ compressor belongs to the class of contractive compressors with $\delta =\frac{d^2}{K}$. For arbitrary matrix $\mathbf{X} \in \R^{d\times d}$ let sort its entires in non-increasing order by magnitude, i.e., $X_{i_k j_k}$ is the $k$-th maximal element of $\mathbf{X}$ by magnitude. Let  $\{\mathbf{B}_{ij}\}_{i,j=1}^d$ be a standard basis in the space of matrices. Then, the Top-$K$ compression operator can be defined via
\begin{equation}
	\cC(\mathbf{X}) \eqdef \sum\limits_{k=1}^K X_{i_k j_k}\cdot \mathbf{B}_{i_k j_k}.
\end{equation}
One way how to make the output of this compressor to be a symmetric matrix is to apply Top-$K$ on upper triangular part of the input. 

\subsection{Example of unbiased compression operators for matrices}

The simplest example of unbiased compressor which could be applied on matrices is random sparsification operator or Rand-$K$. This compressor belongs to the class of unbiased compressors with $\omega =\frac{d^2}{K}-1$. For the input matrix $\mathbf{X} \in \R^{d\times d}$ we choose a set $\mathcal{S}_K$ of indexes $(i,j)$  of cardinality $K$ uniformly at random. Then Rand-$K$ compressor can be defined via
\begin{equation}
	\cC\(\mathbf{X}\)_{ij} \eqdef \begin{cases}\frac{d^2}{K}X_{ij} & \text{if } (i,j) \in \mathcal{S}_K,\\
		0 & \text{if } (i,j) \notin \mathcal{S}_K. \end{cases}
\end{equation}
The way how to make the output of Rand-$K$ to be a symmetric matrix is exaclty the same as for Top-$K$.

\subsection{The performance of Newton's method in different basis}

First, we investigate how the performance of Newton's method is influenced by the choice of the basis. We compare the efficiency of Newton's method on two bases: the one that was described in Section 2.3 and the standard one. The results are presented in Figure~\ref{fig_apx:newton_diff_basis}. We clearly see that Newton's method in the specific basis is approximately $4$ times more communication-efficient than in standard one.

\begin{figure}[ht]
	\begin{center}%
		\centerline{\begin{tabular}{cccc}
				\includegraphics[width = 0.22 \textwidth]{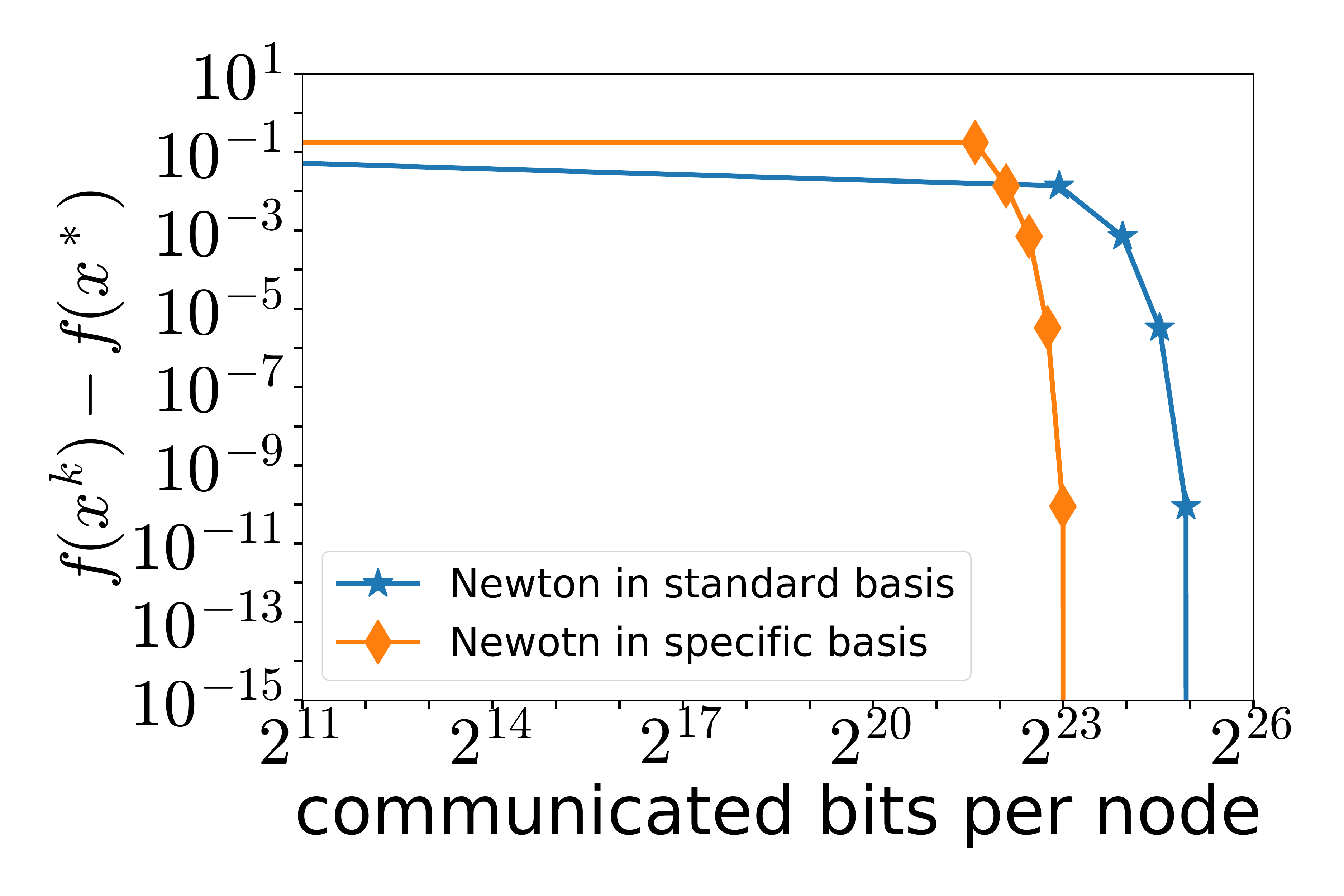} &
				\includegraphics[width = 0.22 \textwidth]{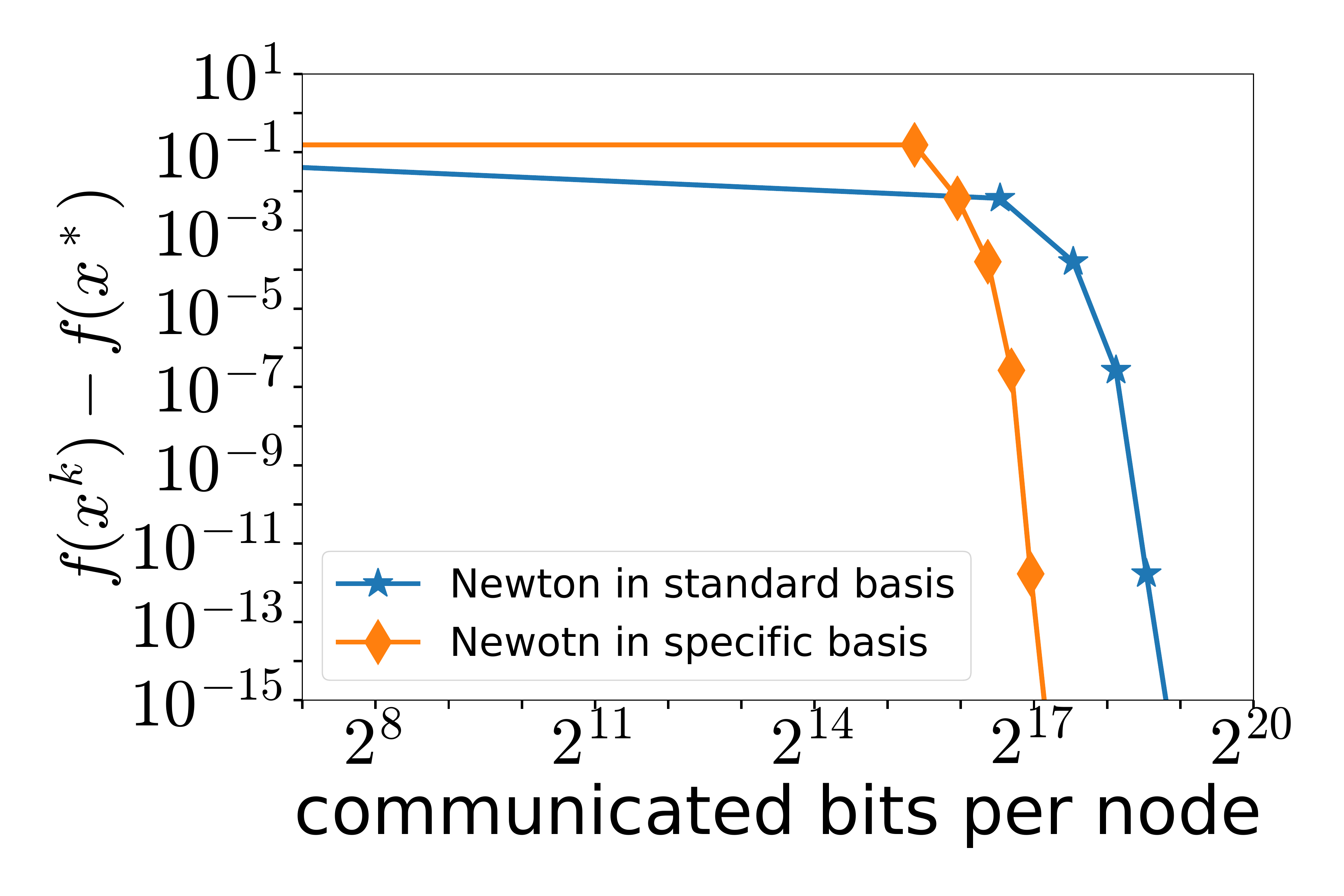}&
				\includegraphics[width = 0.22 \textwidth]{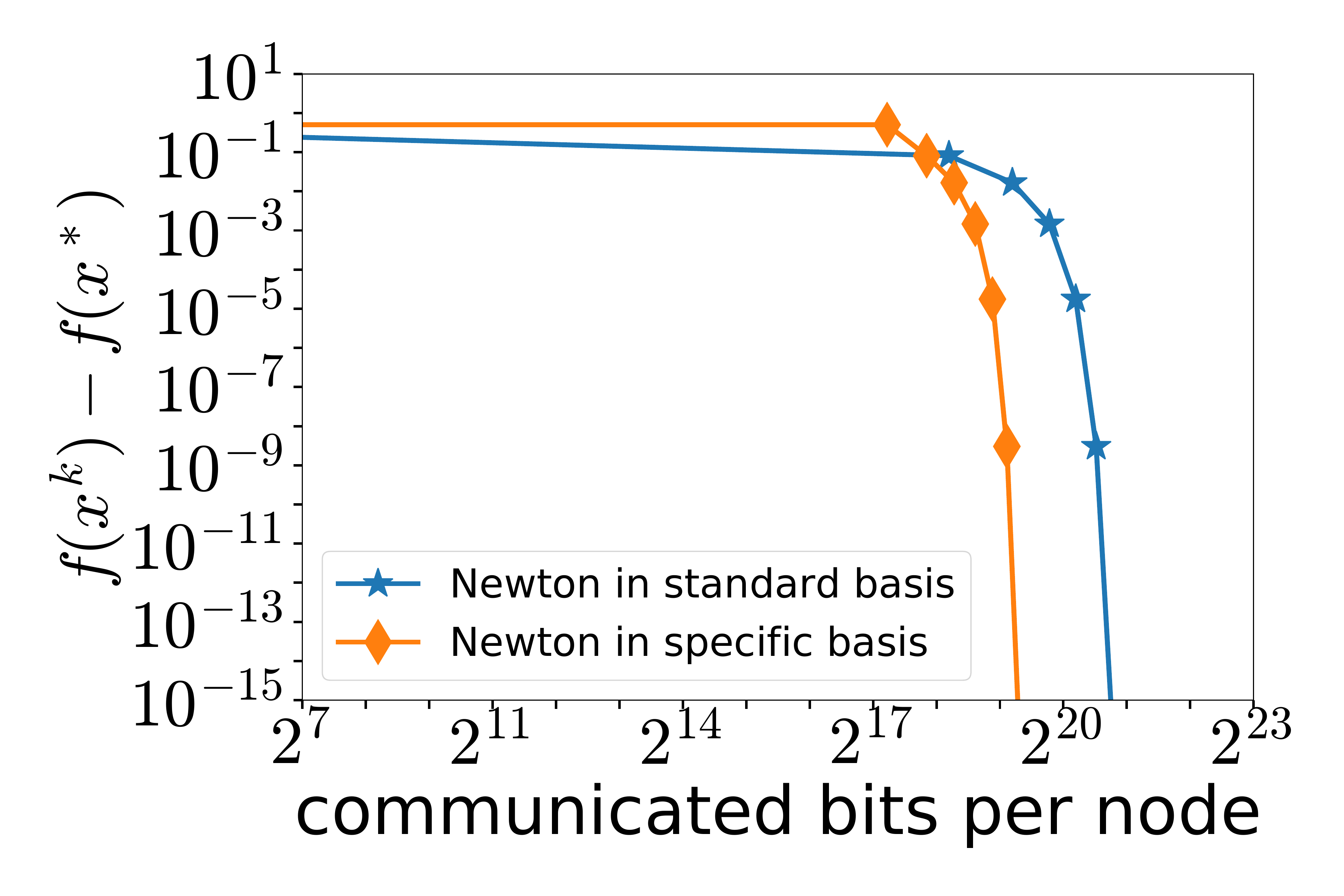} &
				\includegraphics[width = 0.22 \textwidth]{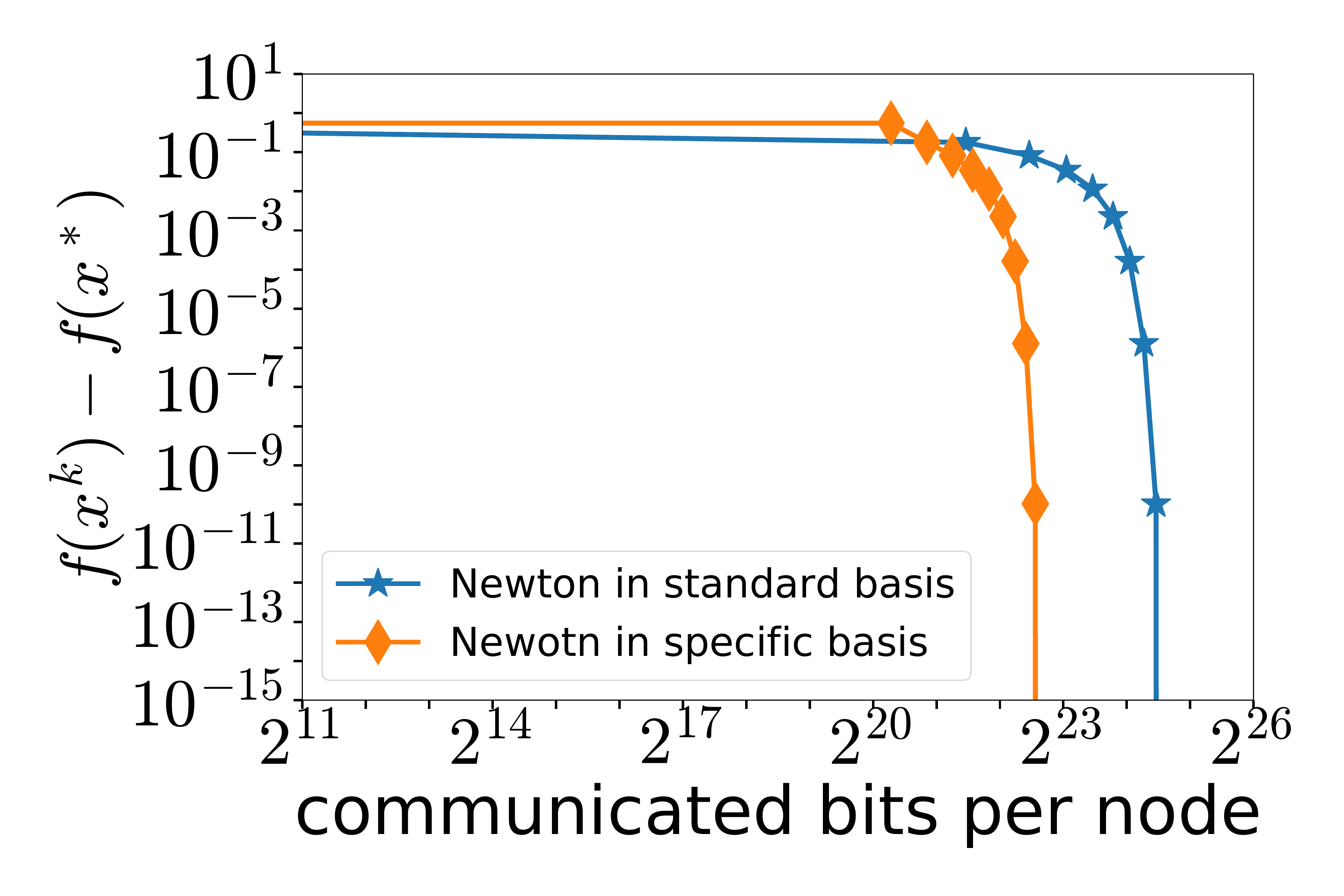}\\
				\dataname{madelon}, $\lambda=10^{-3}$ &
				\dataname{covtype}, $\lambda=10^{-3}$&
				\dataname{phishing}, $\lambda=10^{-4}$ &
				\dataname{w8a}, $\lambda=10^{-4}$\\
		\end{tabular}}
	\end{center}
	\caption{The performance of Newton's method in different basis in terms of communication complexity.}
	\label{fig_apx:newton_diff_basis}
\end{figure}

\subsection{Composition of Top-$K$ and unbiased compressor}

\begin{figure}[ht]
	\begin{center}
		\centerline{\begin{tabular}{cccc}
				\includegraphics[width = 0.22 \textwidth]{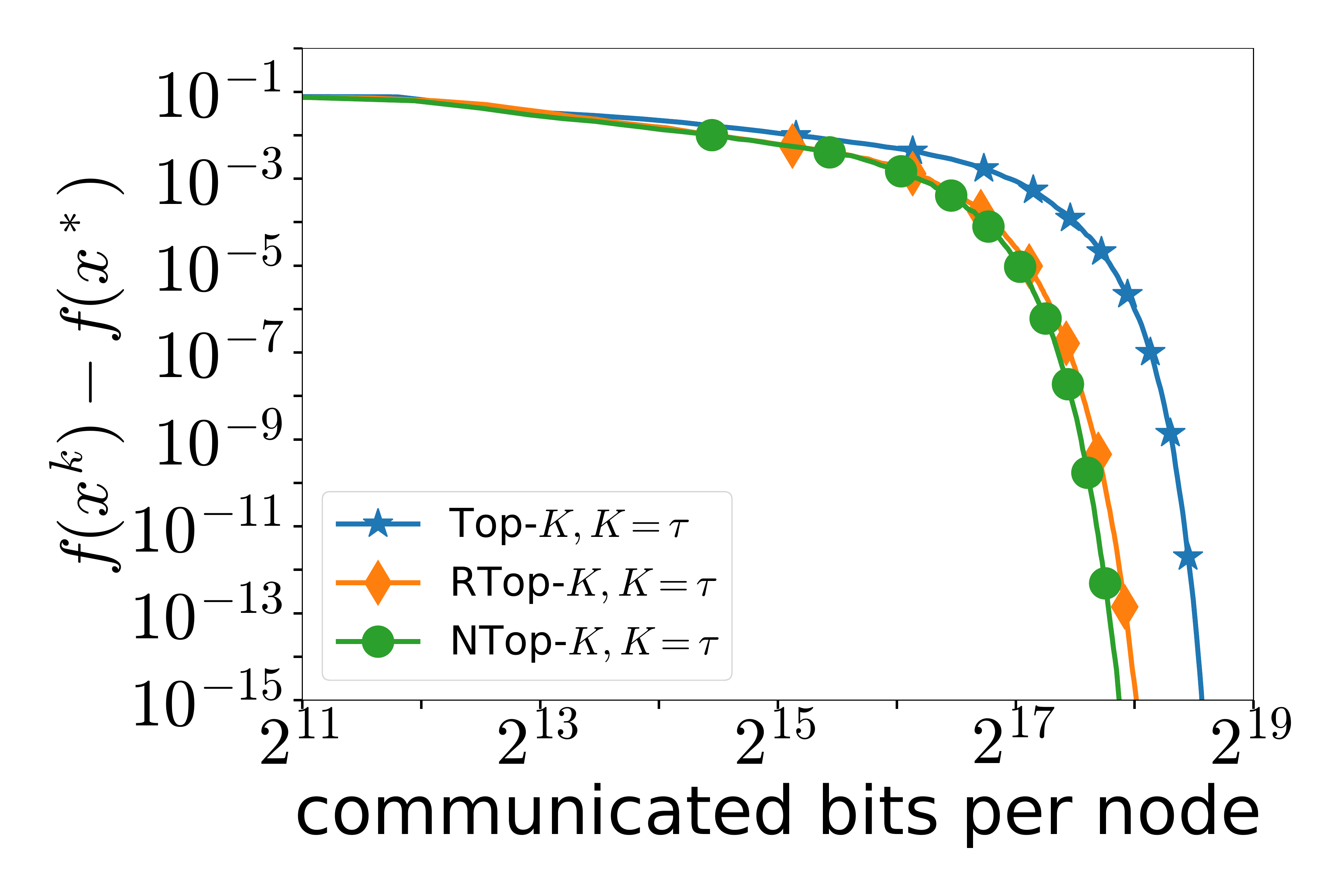} &
				\includegraphics[width = 0.22 \textwidth]{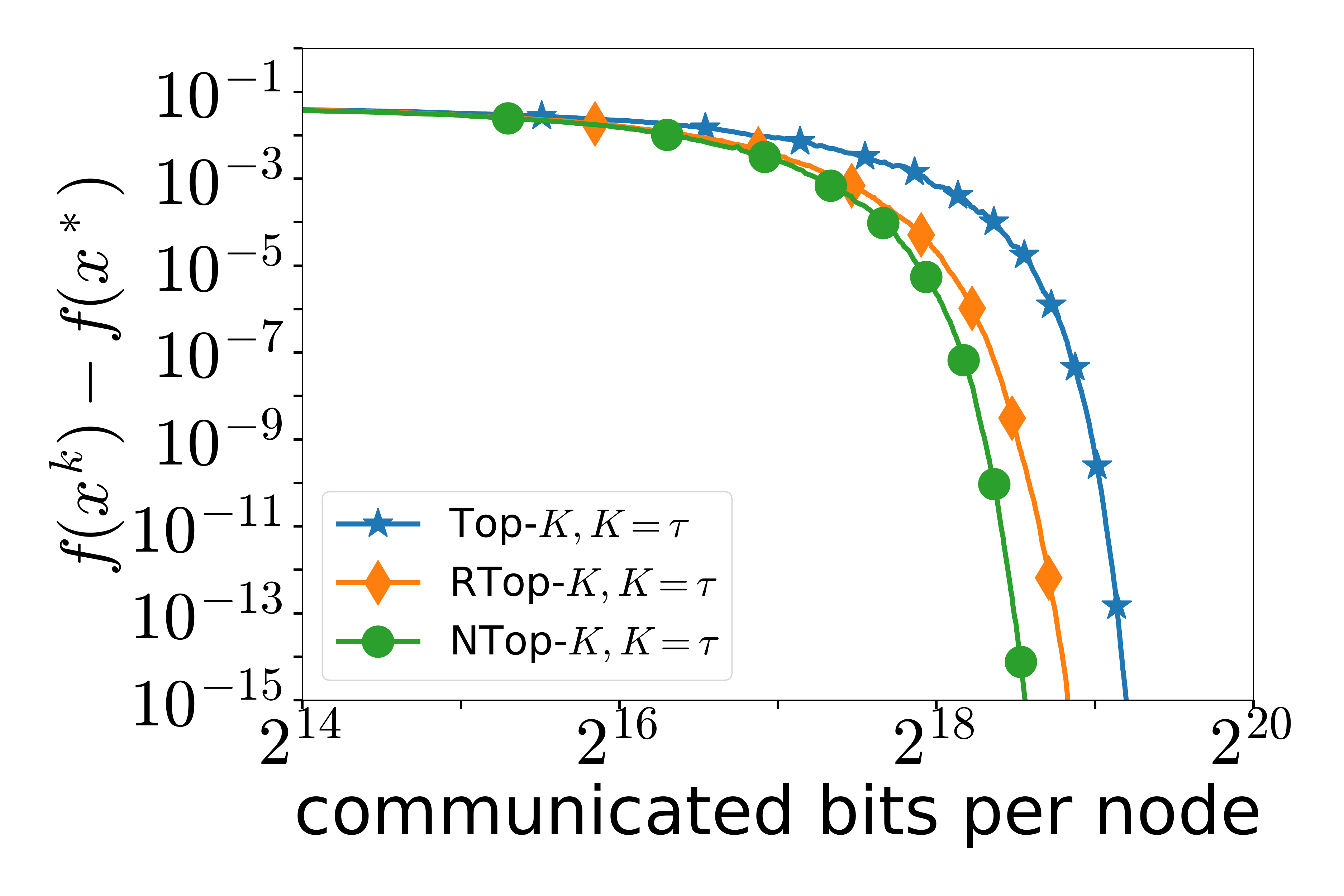}&
				\includegraphics[width = 0.22 \textwidth]{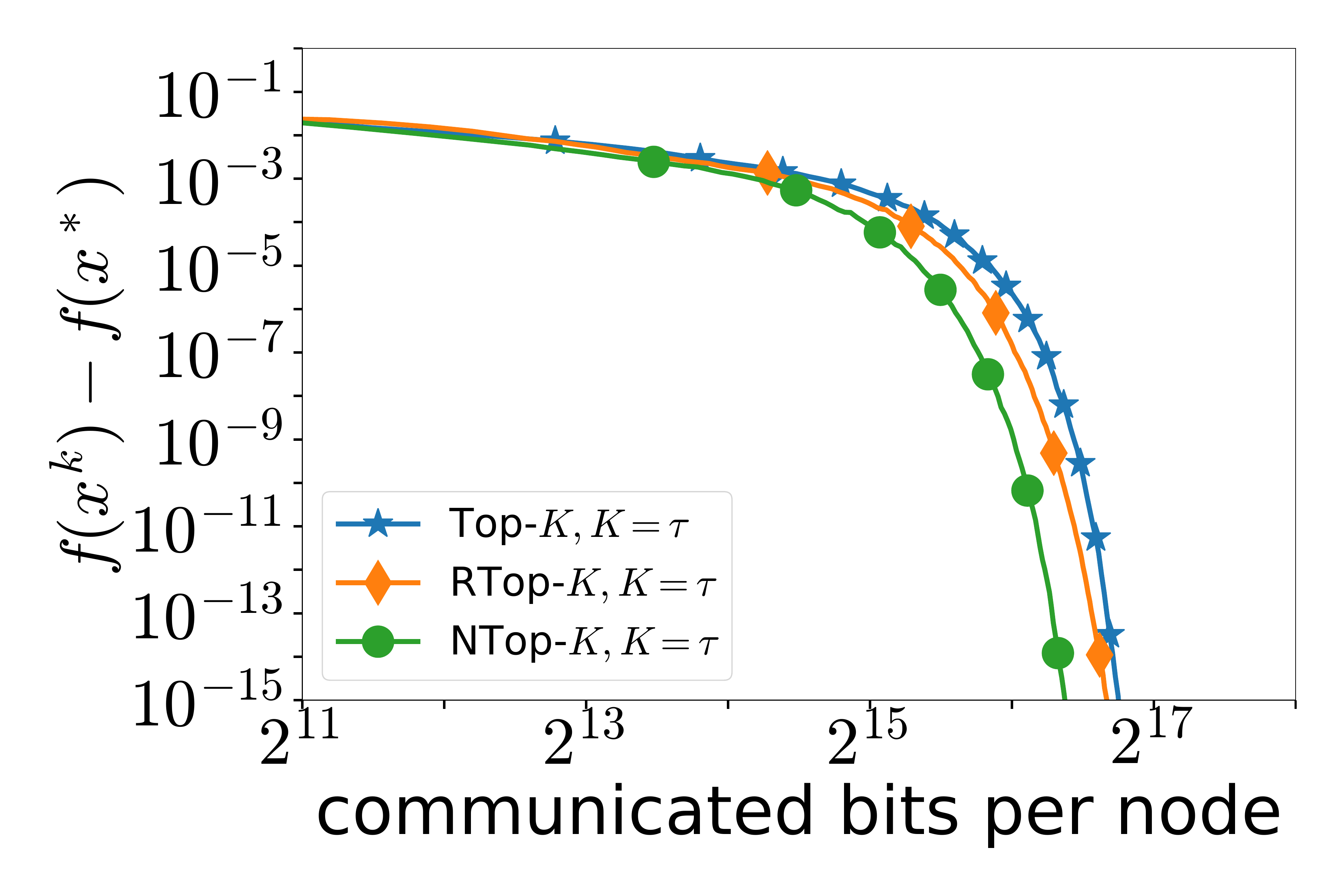} &
				\includegraphics[width = 0.22 \textwidth]{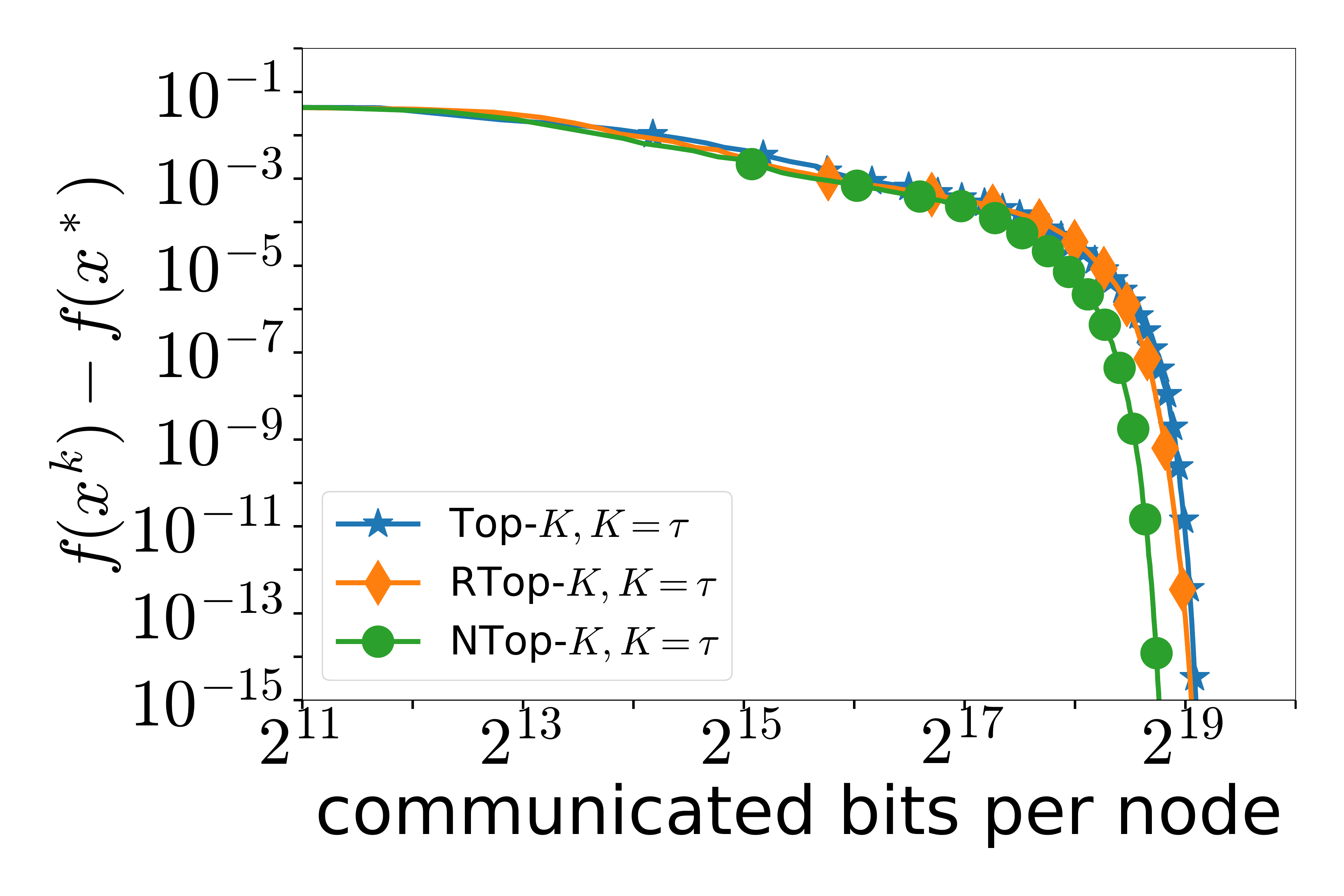}\\
				\dataname{a9a}, $\lambda=10^{-3}$ &
				\dataname{w2a}, $\lambda=10^{-3}$&
				\dataname{phishing}, $\lambda=10^{-4}$ &
				\dataname{a1a}, $\lambda=10^{-4}$\\
		\end{tabular}}
	\end{center}
	\caption{The performance of \algname{BL2} with different types of compression operators: Top-$K$, RTop-$K$ (composition of Top-$K$ and random dithering with $s=\sqrt{K}$), and NTop-$K$ (composition of Top-$K$ and natrual compression).}
	\label{fig_apx:topk_comp}
\end{figure}

Next, we study other type of composition of compression operators. We investigate how composition of Top-$K$ and unbiased compression operator \citep{EC-LSVRG} influences the performance of \algname{BL2}. We compare the performance of \algname{BL2} with Top-$K$ $(K=r)$, RTop-$K$ $(K=r)$ (composition of Top-$K$ and random dithering with $s=\sqrt{K}$), and NTop-$K$ $(K=r)$ (composition of Top-$K$ and natural compression). The initialization of $\mH^0$ is $\nabla^2f(x^0)$. Besides, we use the basis that was decribed in Section~2.3. We set the following parameters for \algname{BL2}: $p=\frac{r}{2d}$, $\tau=n$, and Top-$K$  with $K=\lfloor\frac{r}{2}\rfloor$ for models in the experiments on \dataname{w2a}, \dataname{a1a} data sets. In the experiments on \dataname{a9a}, \dataname{phishing} data sets, these parameters are $p=\frac{r}{4d}$, $\tau=n$, and Top-$K$, ($K=\lfloor\frac{r}{4}\rfloor$) compressor for models. The results are presented in Figure~\ref{fig_apx:topk_comp}.  According to numerical results, we can conclude that composition of Top-$K$ and natural compression is the most efficient compressor in all cases. However, RTop-$K$ have almost the same performance as Top-$K$ on data sets \dataname{a1a}, \dataname{a9a}.

\subsection{The effect of partial participation}

\begin{figure}[ht]
	\begin{center}
		\centerline{\begin{tabular}{cccc}
				\includegraphics[width = 0.15 \textwidth]{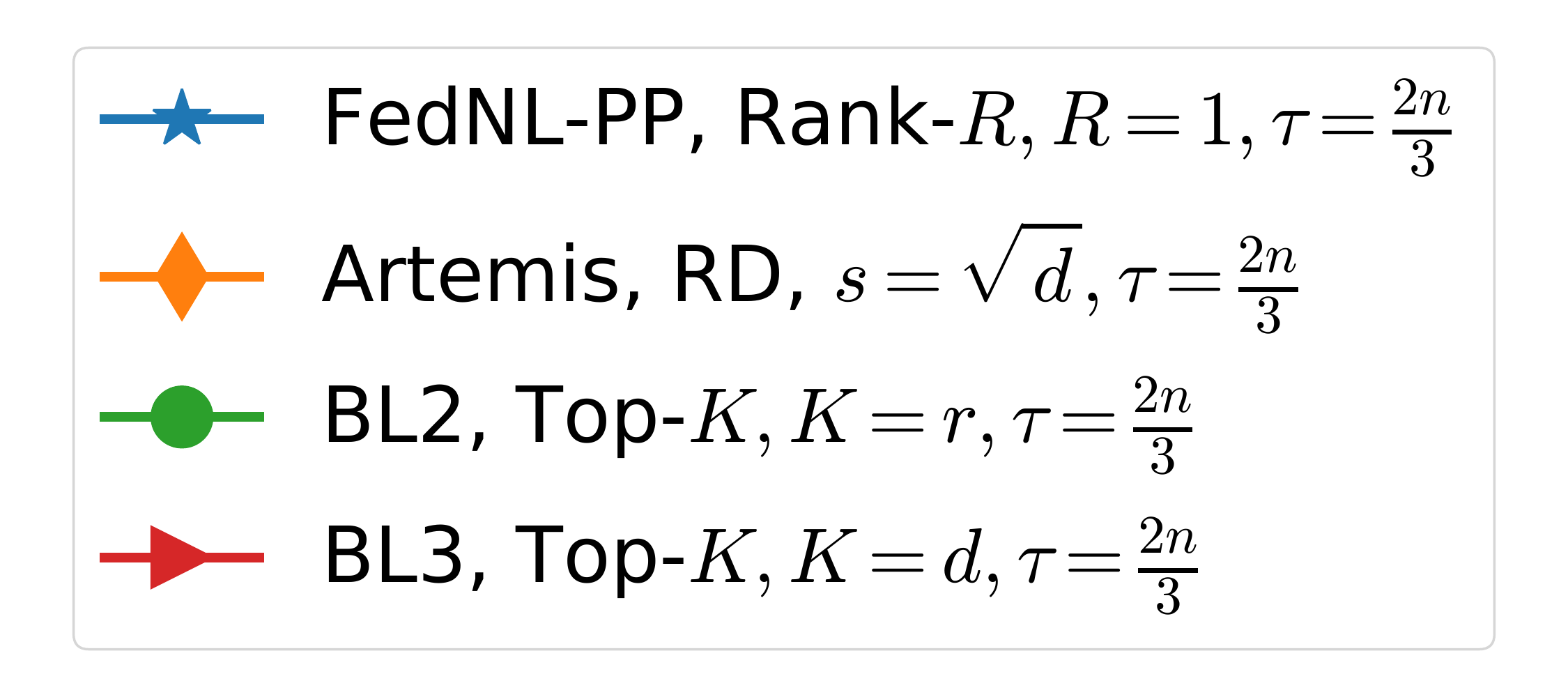}&
				\includegraphics[width = 0.15 \textwidth]{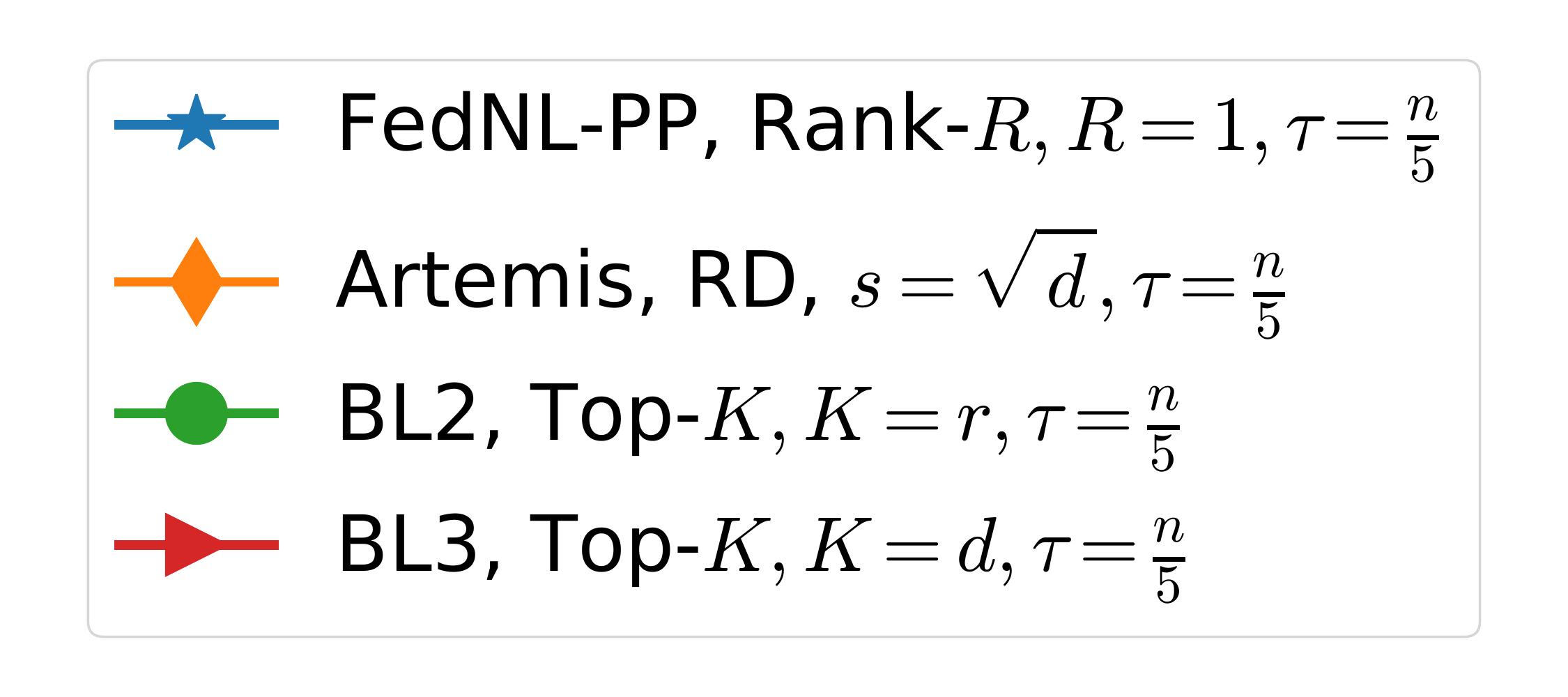} &
				\includegraphics[width = 0.15 \textwidth]{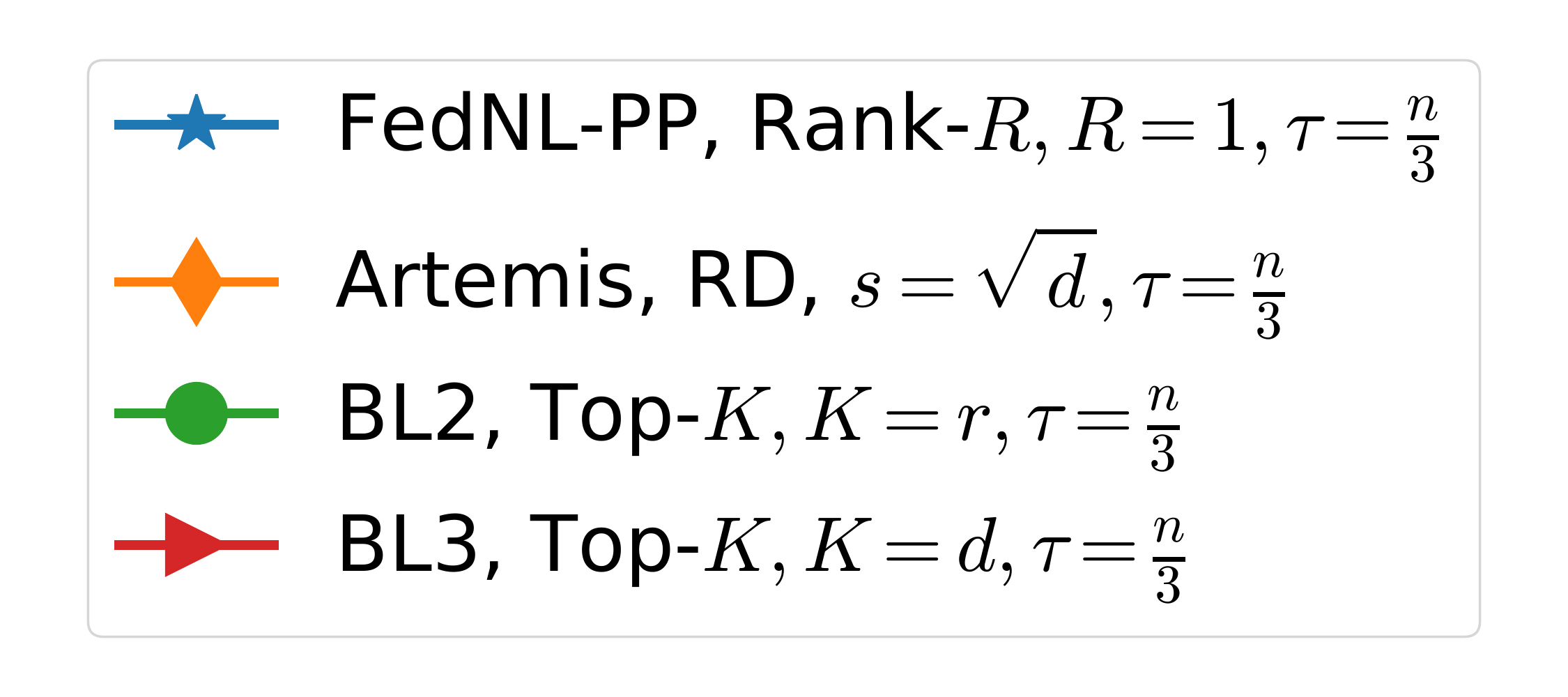} &
				\includegraphics[width = 0.15 \textwidth]{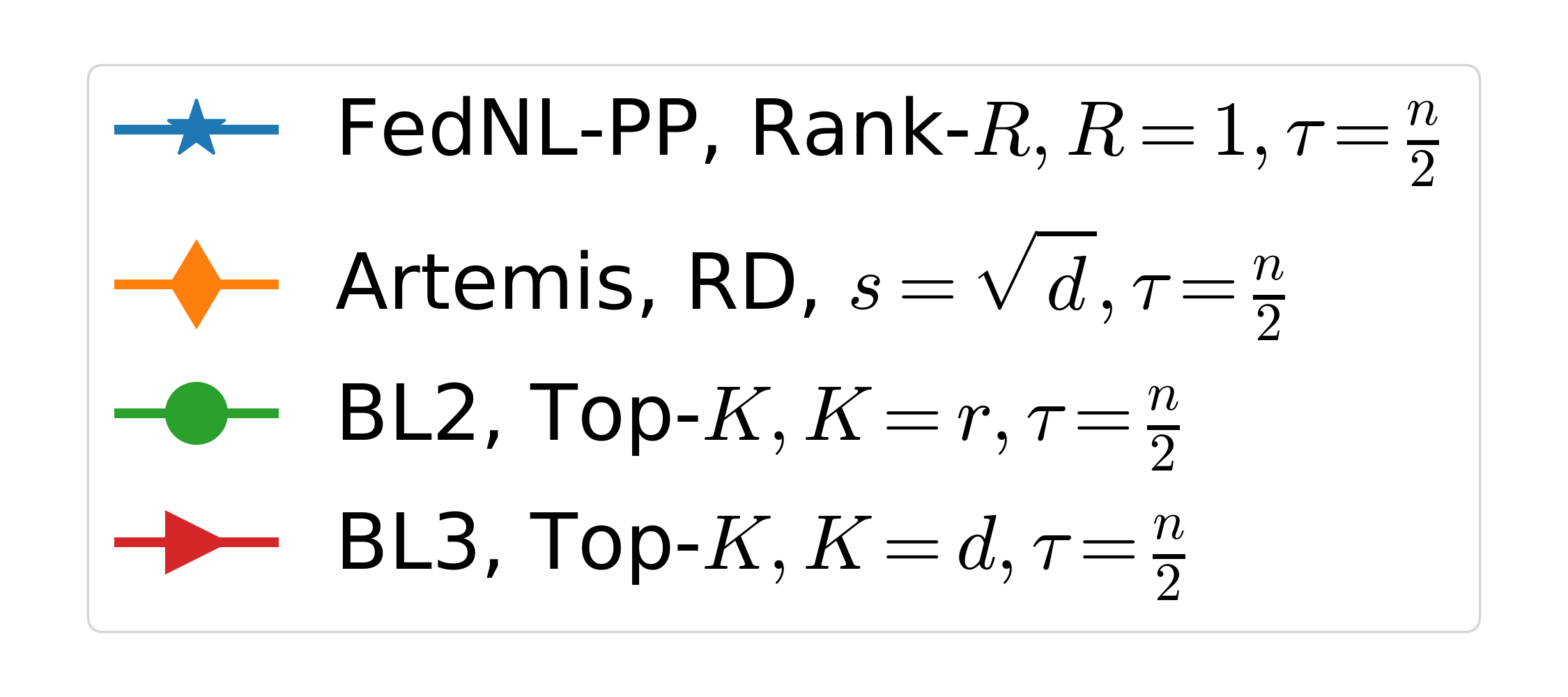}\\
				\includegraphics[width = 0.22 \textwidth]{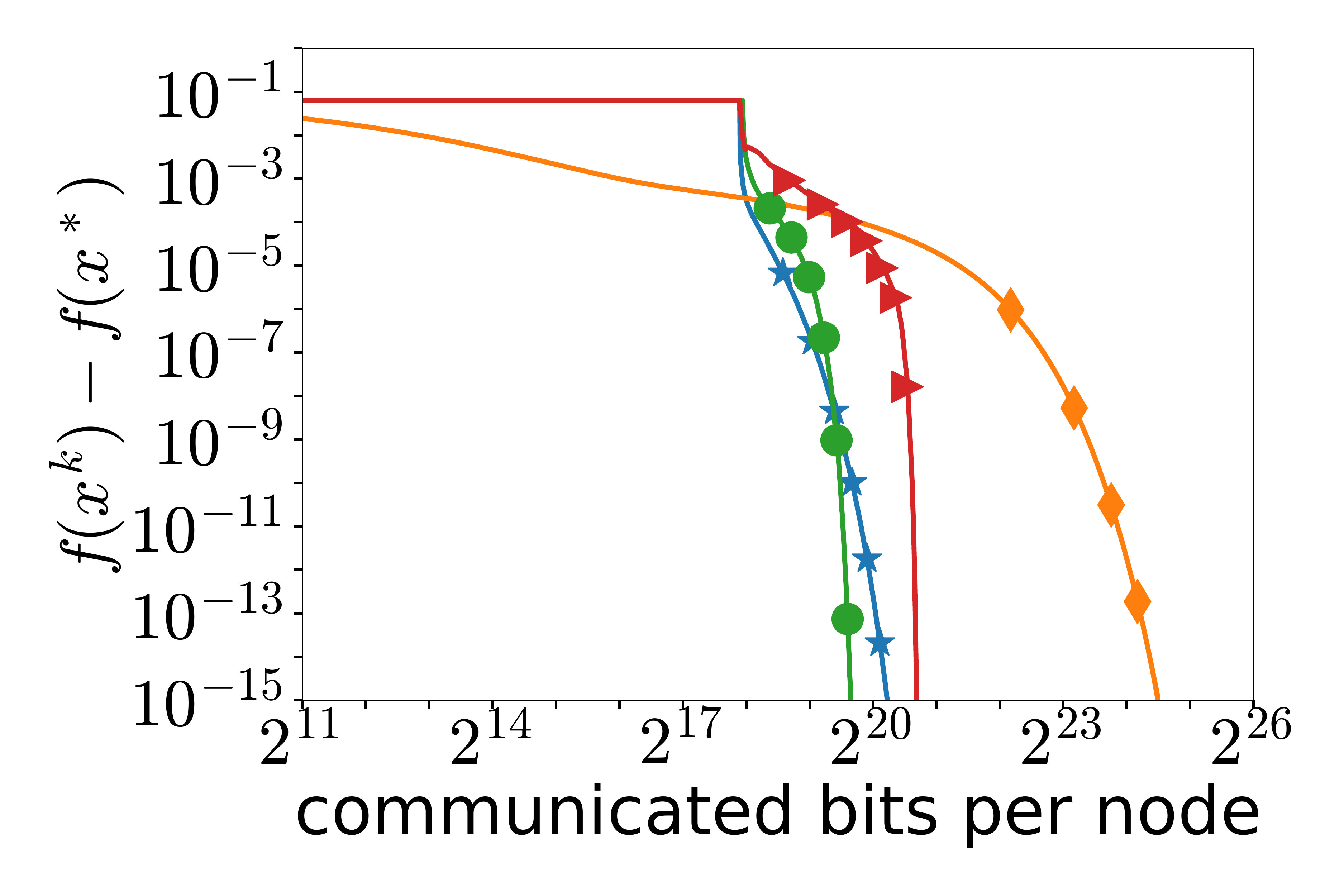}&
				\includegraphics[width = 0.22 \textwidth]{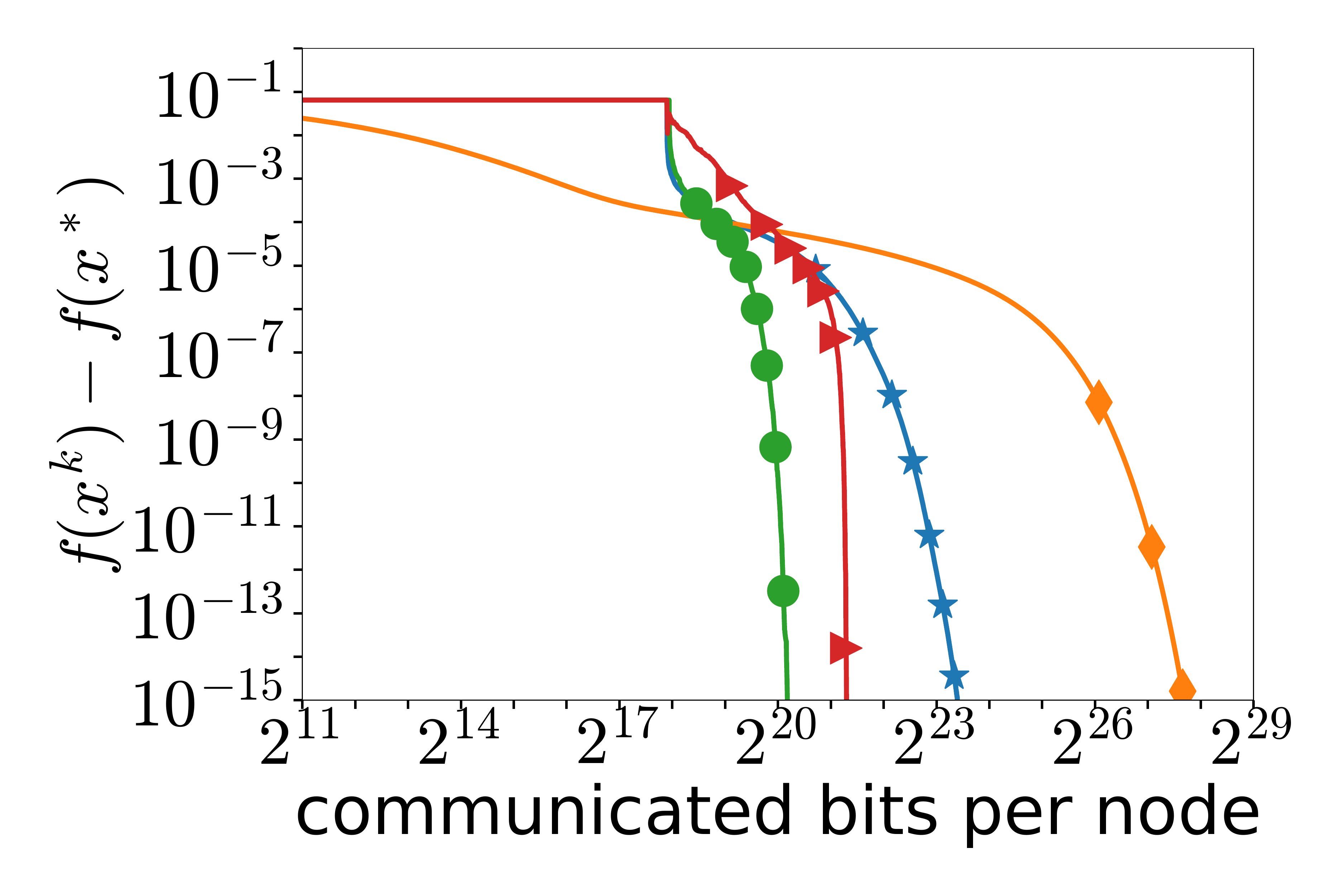} &
				\includegraphics[width = 0.22 \textwidth]{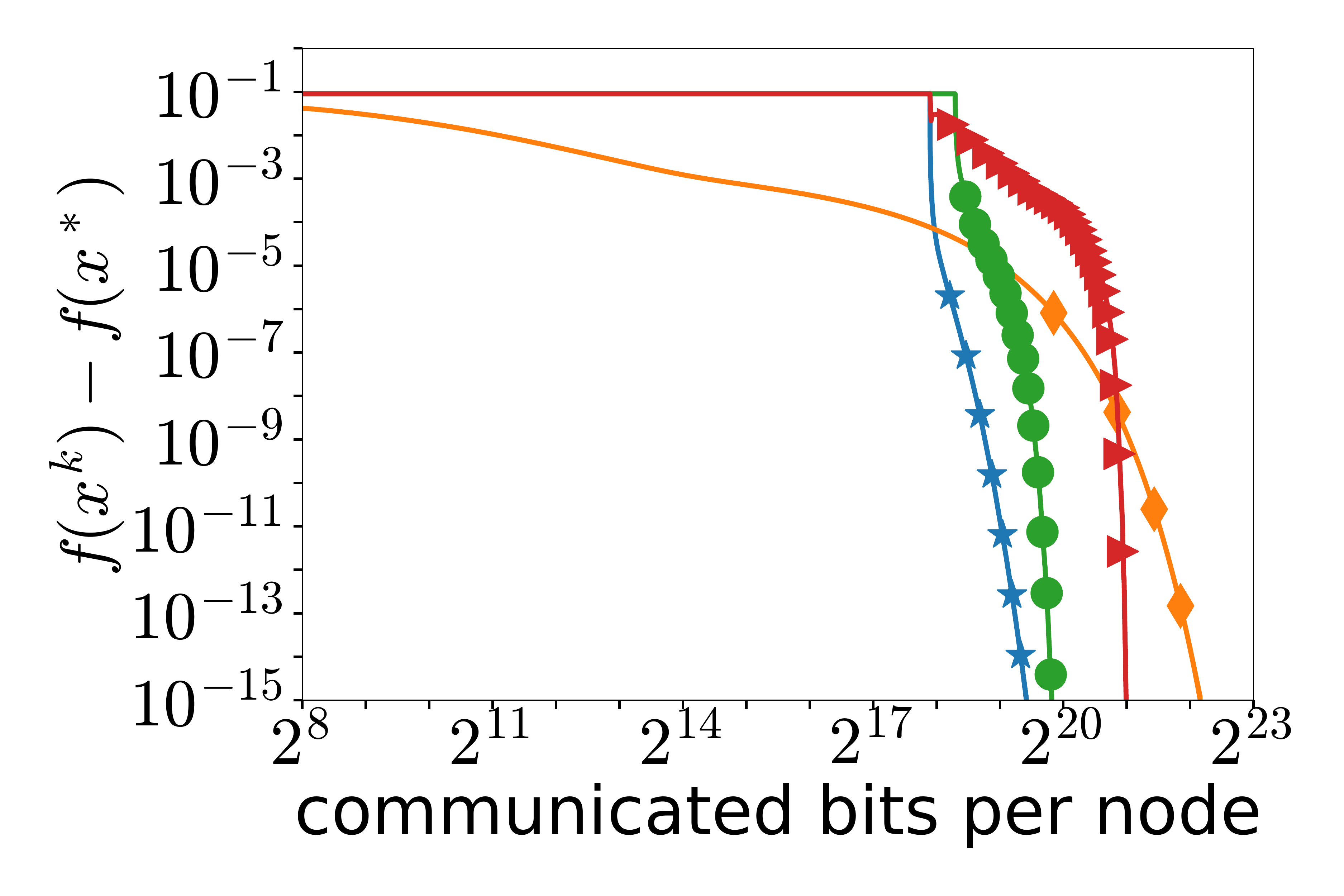} &
				\includegraphics[width = 0.22 \textwidth]{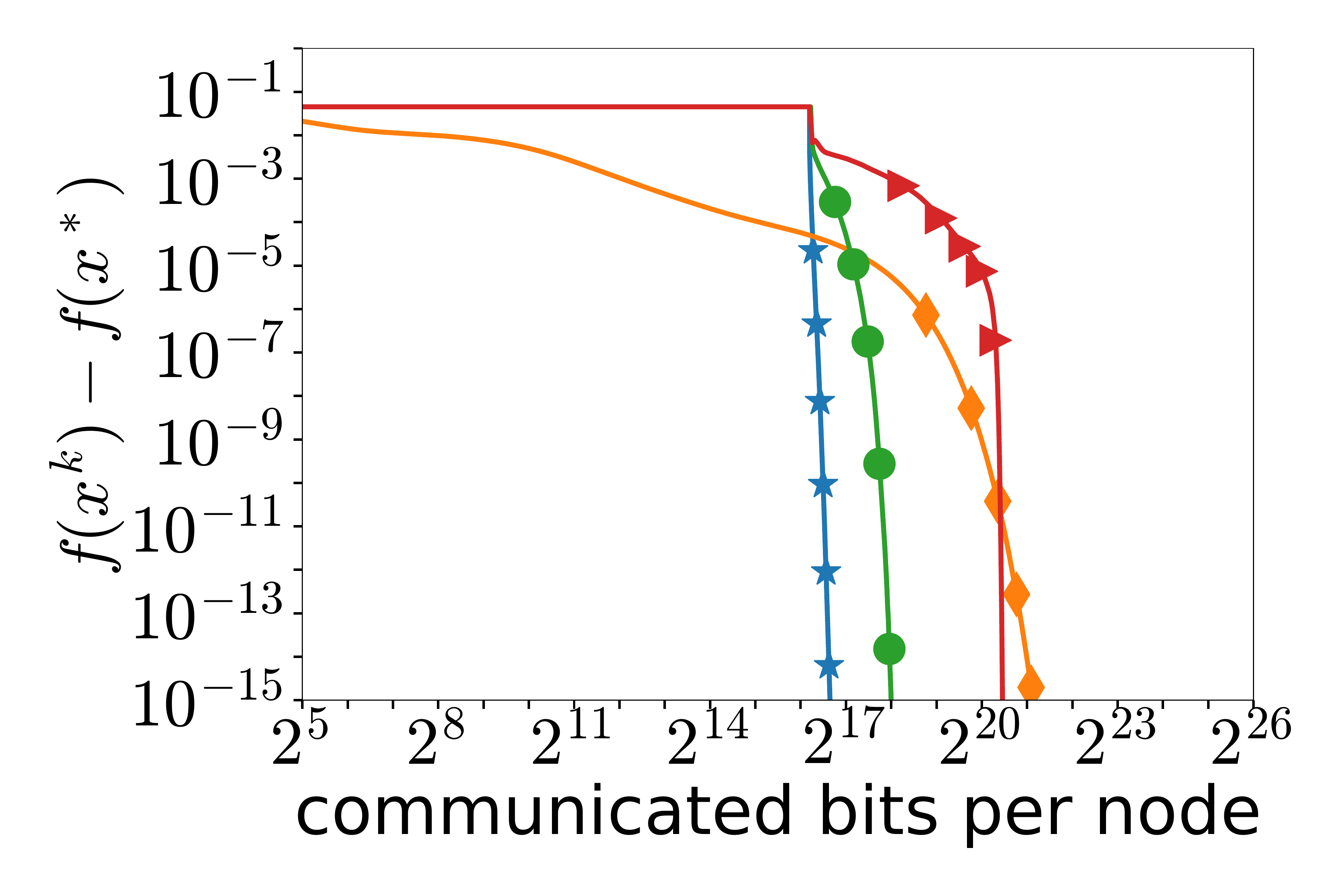}\\
				\dataname{a1a}, $\lambda=10^{-4}$ &
				\dataname{a1a}, $\lambda=10^{-5}$&
				\dataname{a9a}, $\lambda=10^{-4}$ &
				\dataname{phishing}, $\lambda=10^{-5}$\\
		\end{tabular}}
	\end{center}
	\caption{The comparison of \algname{FedNL-PP}, \algname{BL2}, \algname{BL3}, and \algname{Artemis} with partial device participation in terms of communication complexity.}
	\label{fig_apx:pp}
\end{figure}

In this section we study the effect of partial participation. For \algname{FedNL-PP} \citep{FedNL2021} we use stepsize $\alpha=1$ and Rank-$R$ ($R=1$) compression operator. The specific basis described in Section $2.3$ were used for \algname{BL2}. Besides, the parameters of this method are the following: compression operator $\cC^k_i$ is Top-$K$ with $K=r$, $p=1$. The basis for \algname{BL3} were chosen from the Example $4.15$. We use Top-$K$ compressor with $K=d$ for $\cC^k_i$, and set $p=1$ for this method. Both for \algname{BL2} and \algname{BL3} stepsizes are $\alpha=\eta=1$, model comressor $\cQ^k_i$ is identity. Random dithering with $s=\sqrt{d}$ levels was used for \algname{Artemis} \citep{philippenko2021bidirectional}. In different cases we set the number of active devices $\tau$ equal to various fractions of $n$. The results of the experiment are presented in Figure~\ref{fig_apx:pp}. According to the plots, \algname{BL2} and \algname{FedNL-PP} are the best methods, they outperform each depending on data set. \algname{BL3} also outperform \algname{FedNL-PP} on \dataname{a1a} ($\lambda=10^{-5}$)  data set. In almost all cases \algname{FedNL-PP} and \algname{BL2} outperform \algname{Artemis} be {\it many orders in magnitude}. We can conclude that specific for the problem basis could be beneficial.

\subsection{Bidirectional compression}

In our next test we compare \algname{FedNL-BC} \citep{FedNL2021}, \algname{BL1}, \algname{BL2}, \algname{BL3}, and \algname{DORE} \citep{liu2019double}. The parameters of \algname{FedNL-BC} are the following: matrix compression operator is Top-$K$, $K=\lfloor\frac{d}{2}\rfloor$; model compression operator is Top-$K$, $K=\lfloor\frac{d}{2}\rfloor$; stepsizes are $\alpha=\eta=1$; probability $p=1$. We use option $1$ (projection) to make Hessian approximation to be positive definite. Next, we use the basis described in Section $2.3$ for \algname{BL1} and \algname{BL2}. We use Top-$K$, $K=\lfloor\frac{r}{2}\rfloor$, for matrices and models compression, probability $p=\frac{r}{2d}$, and stepsizes $\alpha=\eta=1$. The basis for \algname{BL3} is described in Example $4.15$ in the main paper. Besides, this method has the following parameters: Top-$K$, $K=\lfloor\frac{d}{2}\rfloor$ for models and Hessians compression; stepsize $\alpha=\eta=1$; probability $p=\frac{1}{2}$. Finally, all devices are active for \algname{BL2} and \algname{BL3}, i.e. $\tau=n$. The results of this test can be found in Figure~\ref{fig_apx:bc}. 

\begin{figure}[ht]
	\begin{center}
		\begin{tabular}{cccc}
			\multicolumn{4}{c}{
				\includegraphics[width=0.8\linewidth]{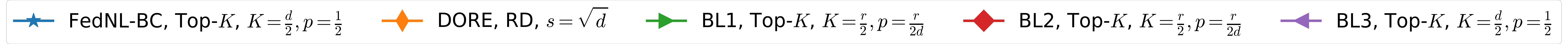}
			}\\
			\includegraphics[width = 0.22 \textwidth]{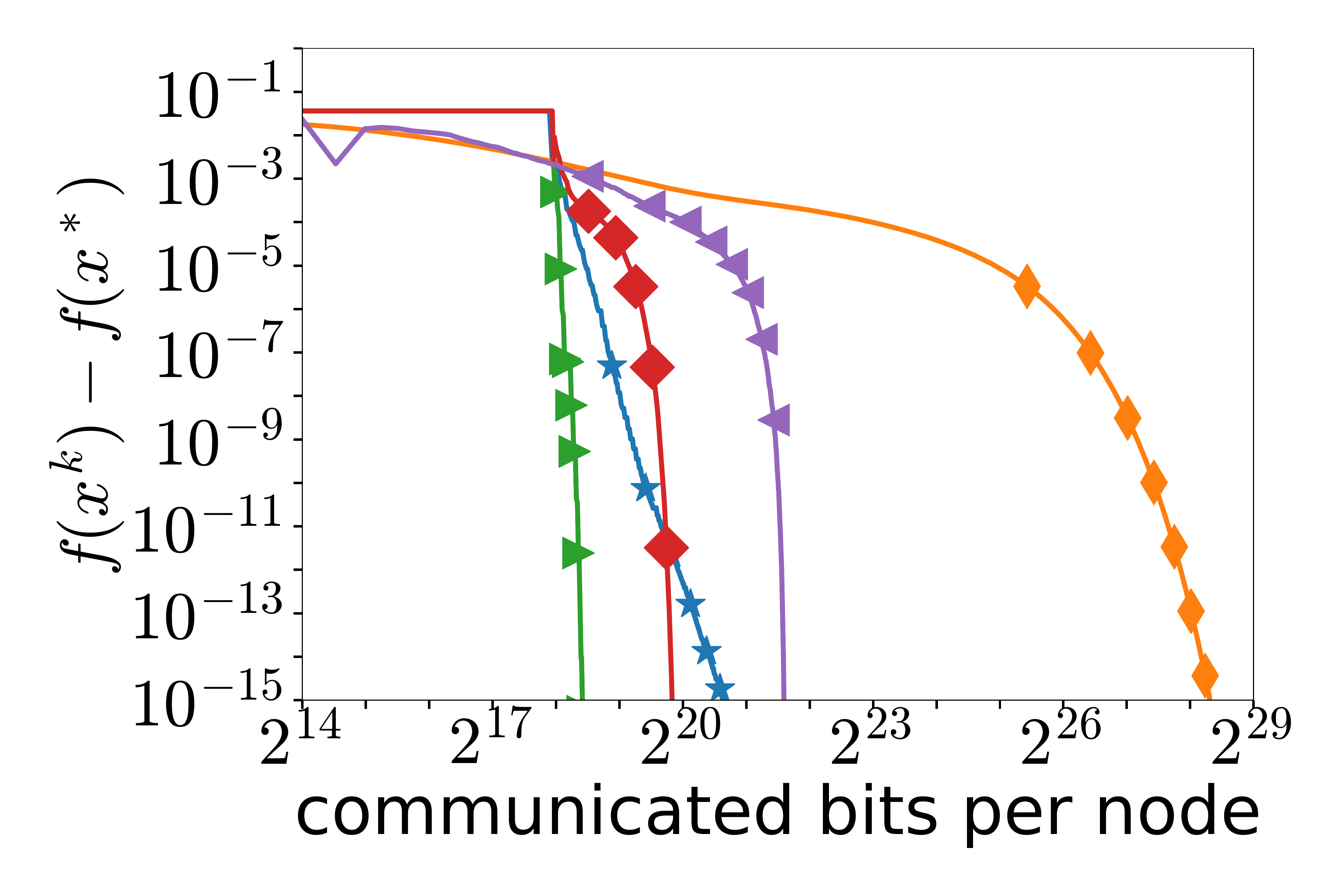} &
			\includegraphics[width = 0.22 \textwidth]{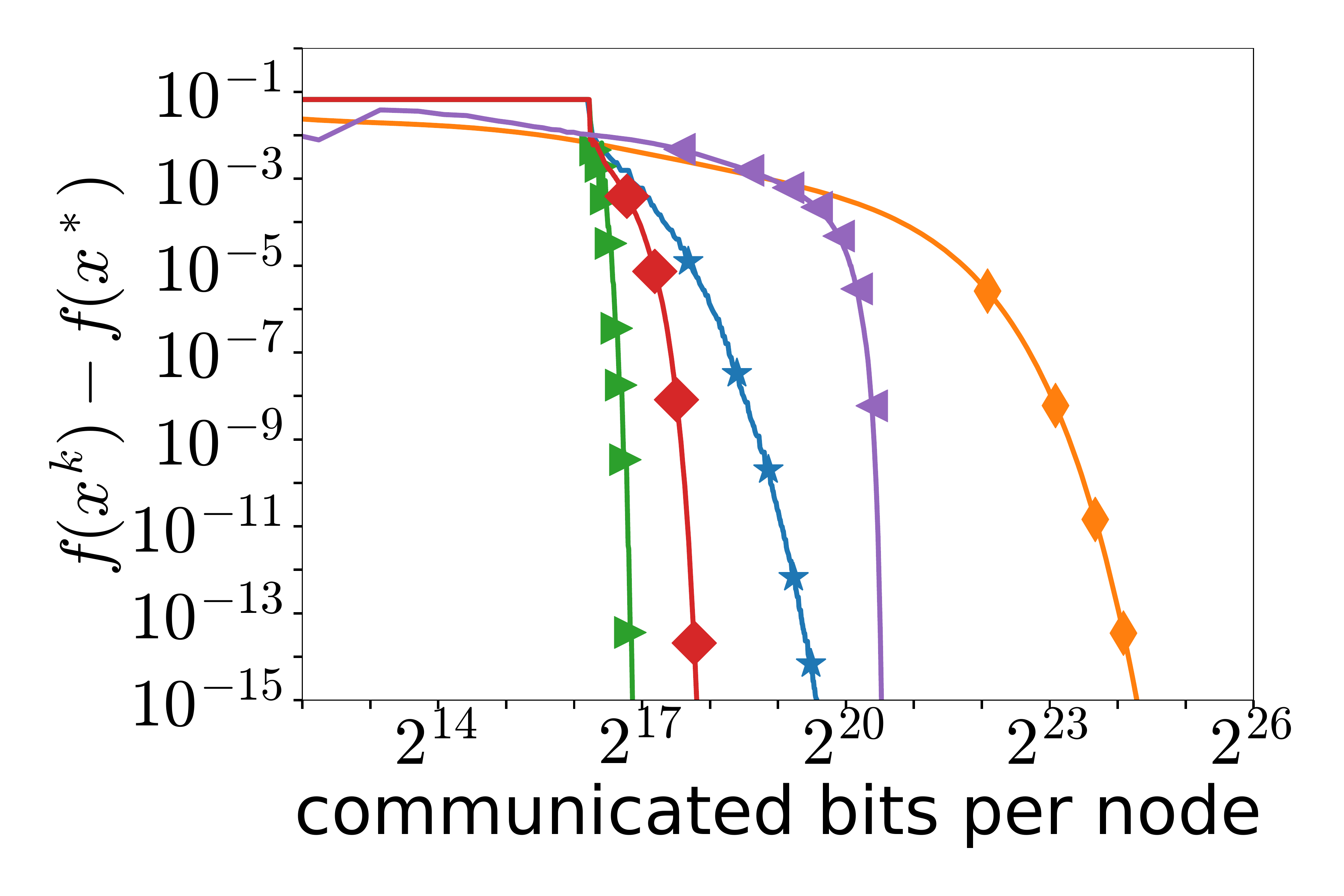}&
			\includegraphics[width = 0.22 \textwidth]{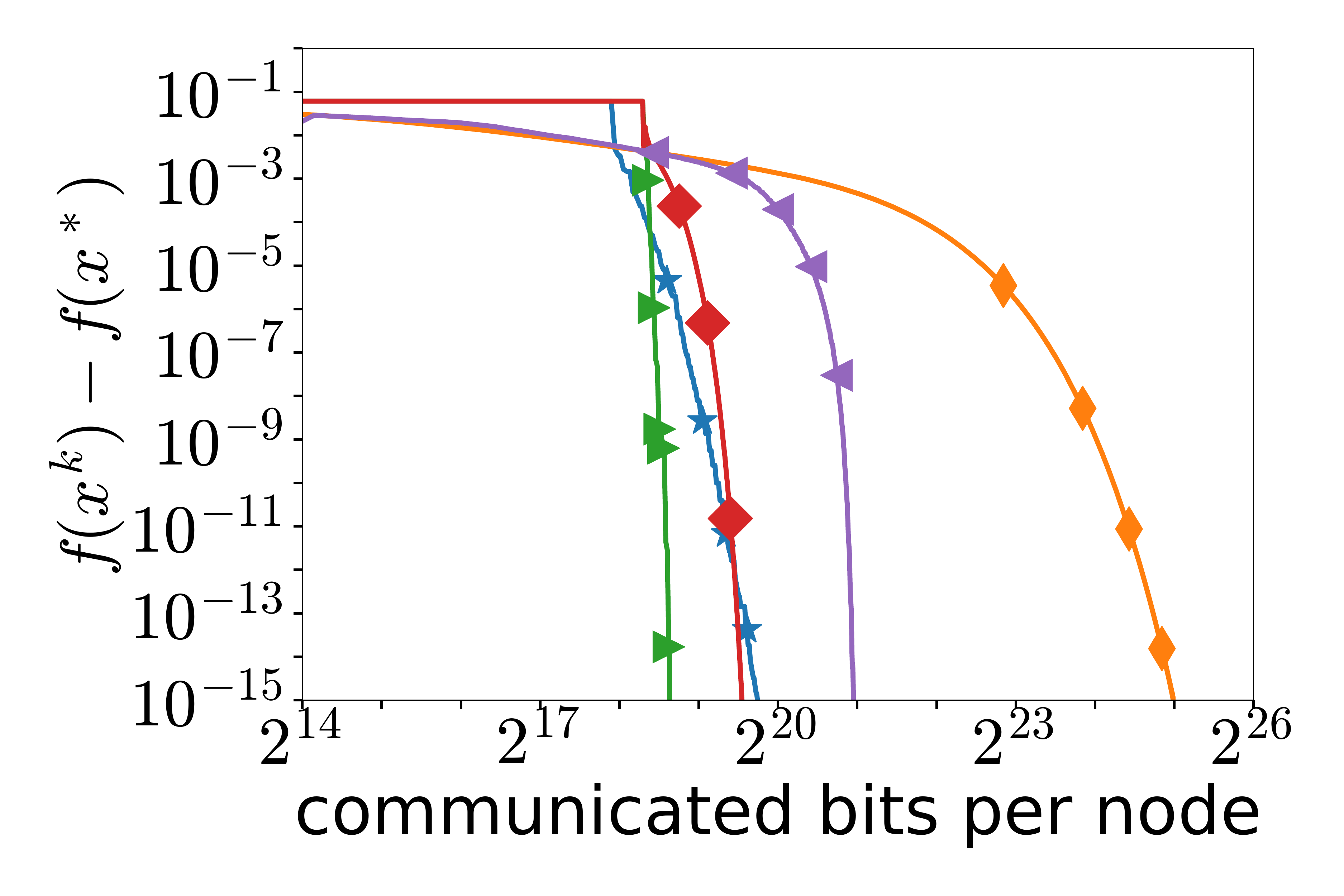}&
			\includegraphics[width = 0.22 \textwidth]{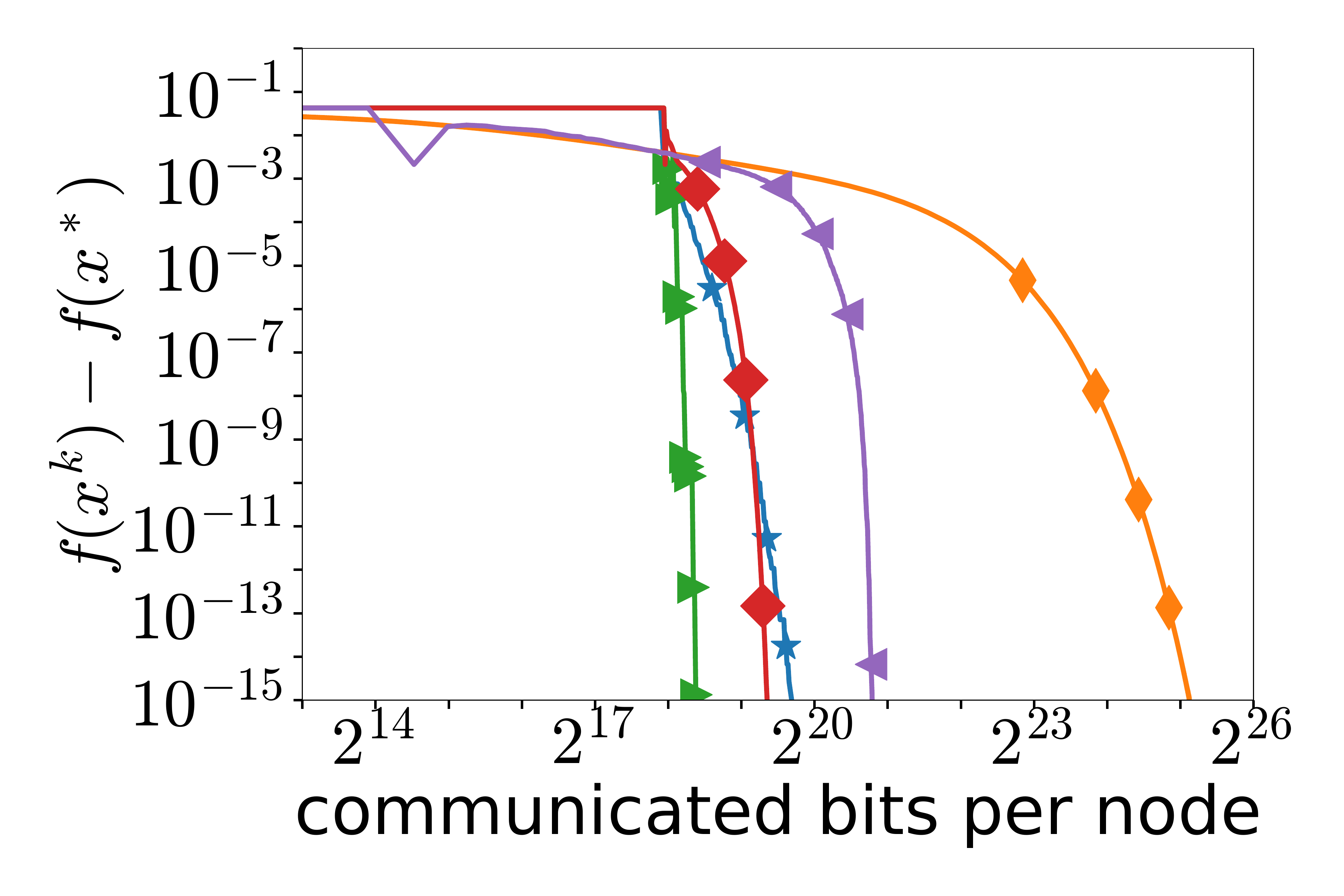} \\
			\dataname{a1a}, $\lambda=10^{-4}$ &
			\dataname{phishing}, $\lambda=10^{-4}$&
			\dataname{a9a}, $\lambda=10^{-3}$ &
			\dataname{a1a}, $\lambda=10^{-3}$\\
		\end{tabular}
	\end{center}
	\caption{The comparison of \algname{FedNL-BC}, \algname{BL1}, \algname{BL2}, \algname{BL3} and \algname{DORE} with bidirectional compression in terms of communication complexity.}
	\label{fig_apx:bc}
\end{figure}

We see that all second-order methods outperform \algname{DORE} in terms of communication complexity by {\it many orders in magnitude}. Moreover, we can conclude that specific to the problem basis is helpful since \algname{BL1} and \algname{BL2} outperform \algname{FedNL-BC}.

\subsection{Comparison of \algname{BL2} and \algname{BL3}}

Finally, we compare \algname{BL2} and \algname{BL3} with bidirectional compression and partial participation simultaneously. We set the number of active devices to $\frac{n}{2}$. For \algname{BL2} we use standard basis in the space of matrices, for \algname{BL3} the basis is one that was given in the example $4.15$. For both methods the compression operator is Top-$K$, $K=\lfloor pd \rfloor$, both for models and matrices. The gradient compressor is lazy Bernoulli compressor with parameter $p$. We set $p \in \{1, \nicefrac{1}{3}, \nicefrac{1}{5}\}$. In the
Figure~\ref{fig_apx:bl2_bl3} we plot the optimality gap $f(x^k) - f(x^*)$ versus the average number of communicated bits per node. 

\begin{figure}[ht]
	\begin{center}
		\begin{tabular}{cccc}
			\multicolumn{4}{c}{
				\includegraphics[width=0.9\linewidth]{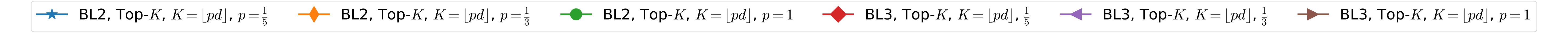}
			}\\
			\includegraphics[width = 0.22 \textwidth]{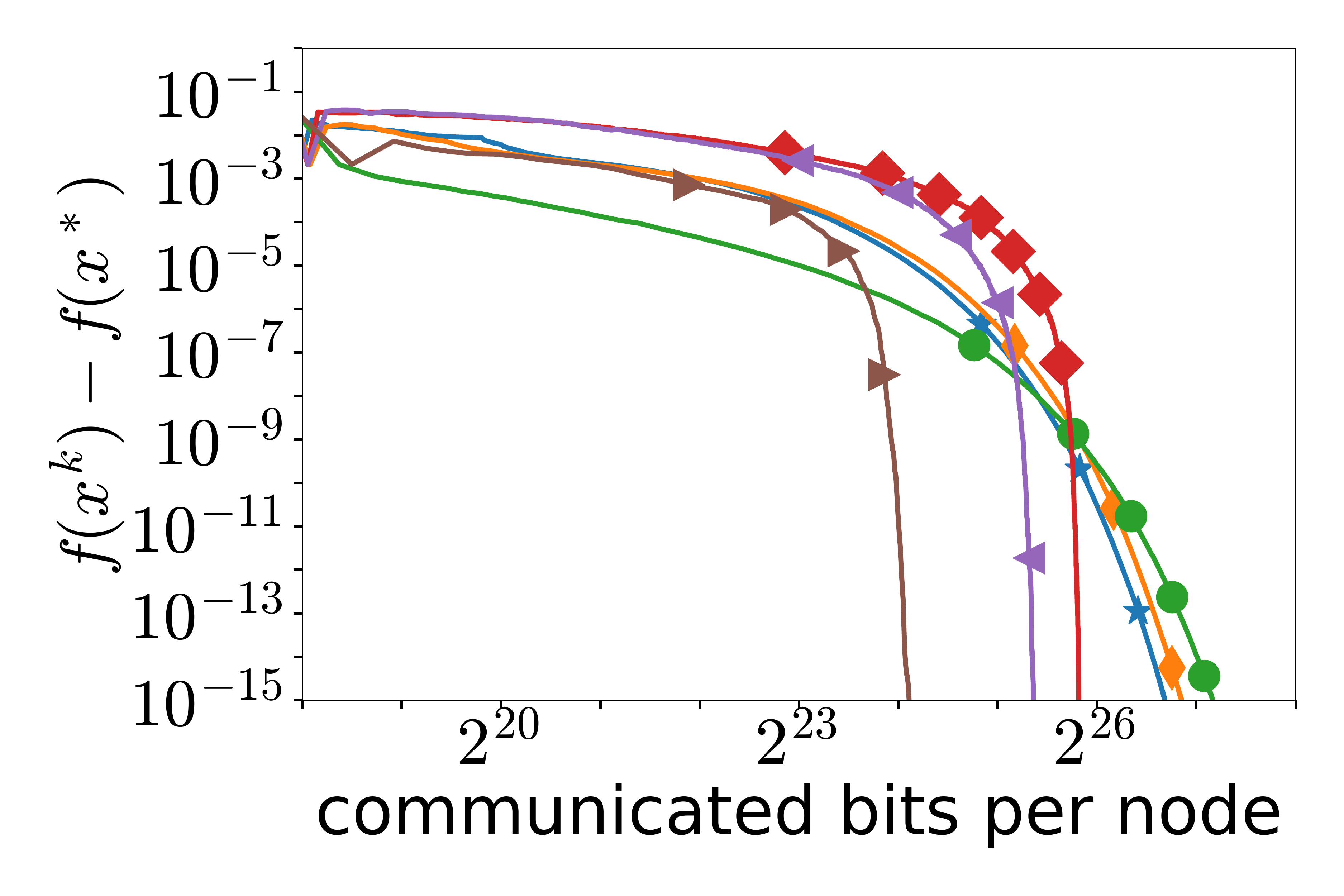} &
			\includegraphics[width = 0.22 \textwidth]{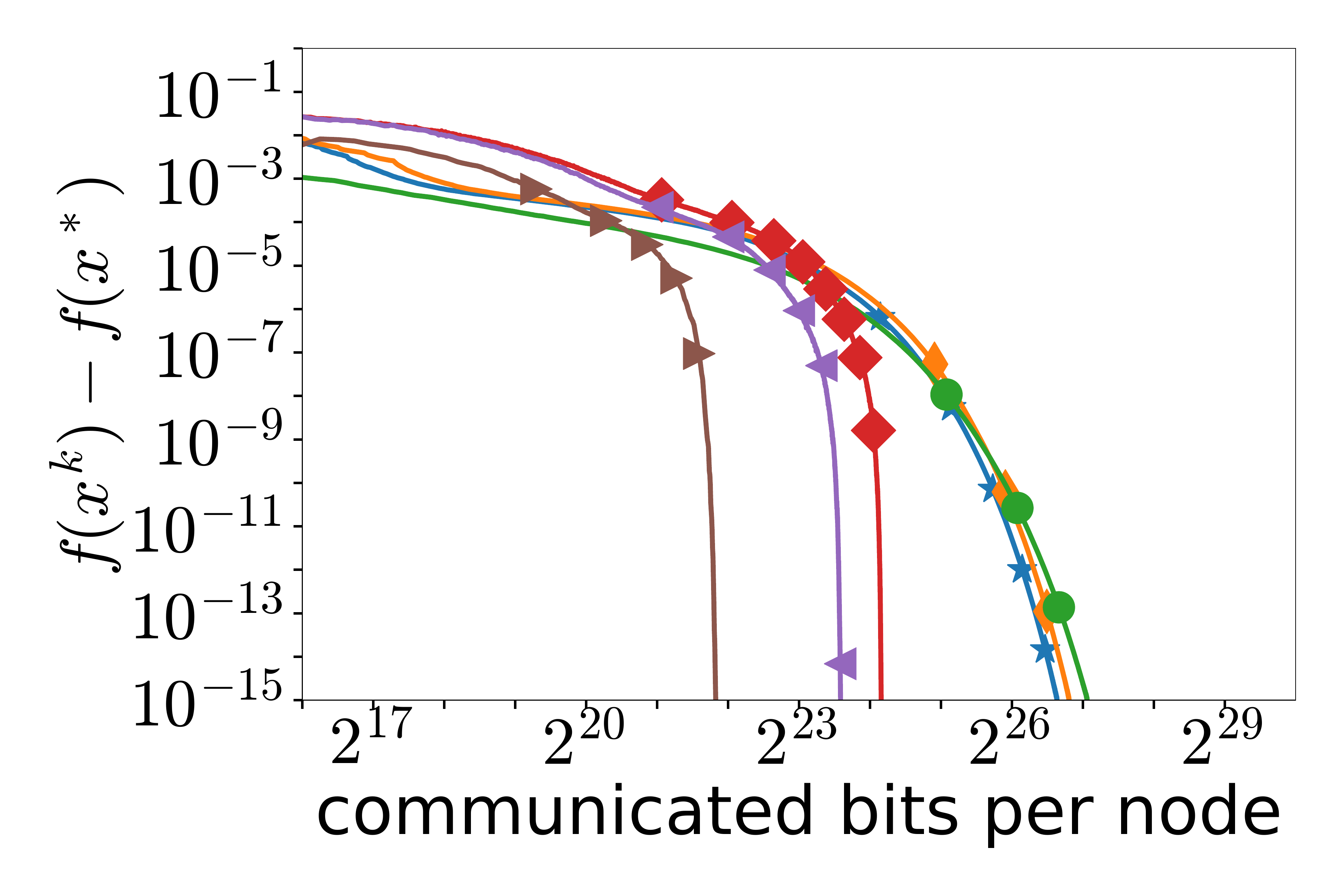}&
			\includegraphics[width = 0.22 \textwidth]{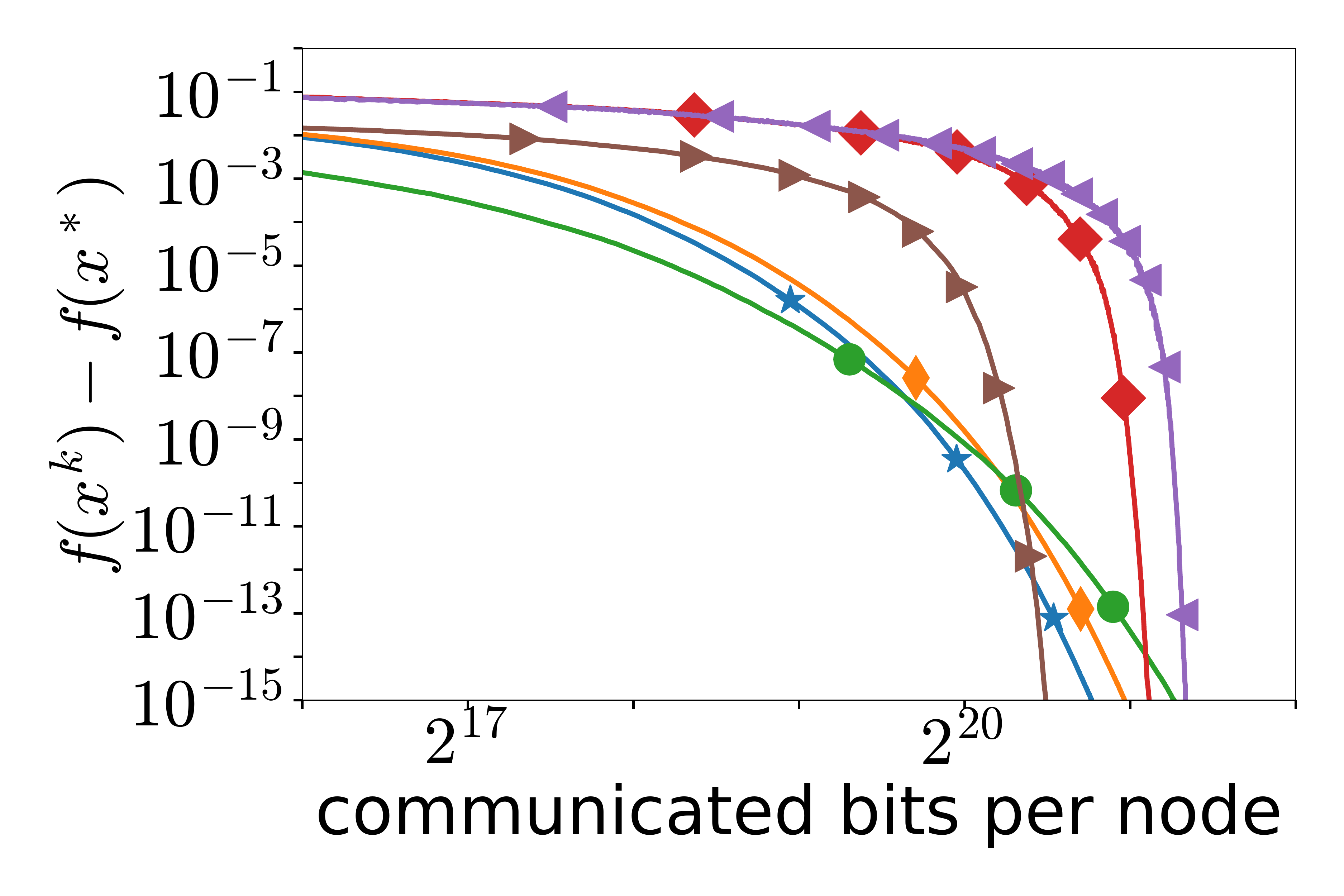}&
			\includegraphics[width = 0.22 \textwidth]{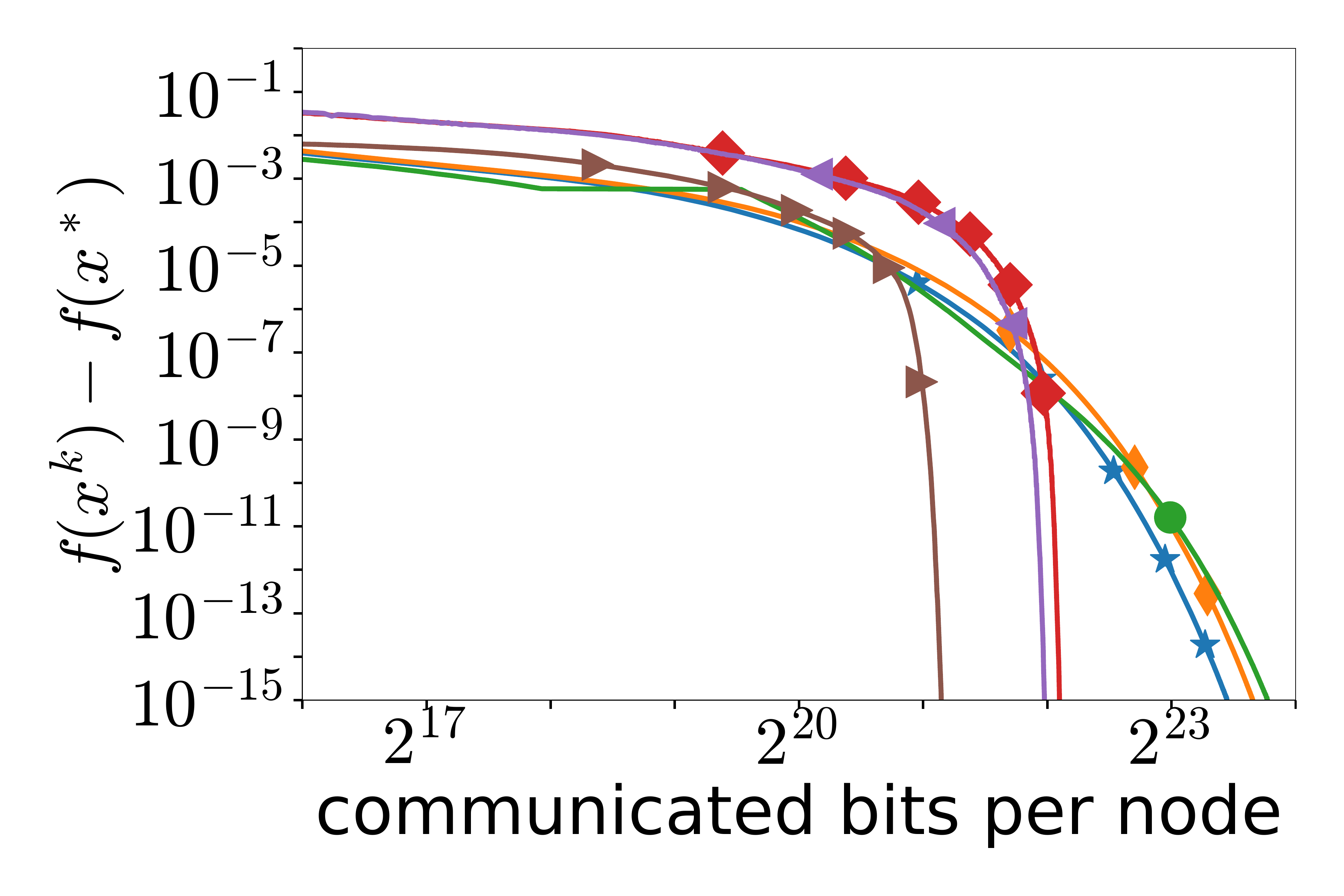} \\
			\dataname{a1a}, $\lambda=10^{-3}$ &
			\dataname{a1a}, $\lambda=10^{-4}$&
			\dataname{phishing}, $\lambda=10^{-3}$ &
			\dataname{phishing}, $\lambda=10^{-4}$\\
		\end{tabular}
	\end{center}
	\caption{The comparison of \algname{BL2} and \algname{BL3} with bidirectional compression and partial participation in terms of communication complexity.}
	\label{fig_apx:bl2_bl3}
\end{figure}

The first observation from the numerical results is that \algname{BL2} is less communication-efficient method than \algname{BL3}. However, if we use specific basis for \algname{BL2}, then it improves the performance of the method; in Figure~\ref{fig_apx:bc} \algname{BL2} is better than \algname{BL3}. Besides, we clearly see that bicompression improves the performance of \algname{BL2} in partial participation setting. However, this is not the case for \algname{BL3}.

\section{Proofs of Lemma \ref{lm:ABineq} and Proposition \ref{th:C1C2}}

\subsection{Proof of Lemma \ref{lm:ABineq}} 

(i) We have 
\begin{align*}
	&\quad \left\|\frac{(\mB + \mB^\top)}{2} - \mA \right\|_{\rm F}^2 - \|\mB-\mA\|_{\rm F}^2 \\ 
	& = \frac{1}{4} \|\mB + \mB^\top\|_{\rm F}^2 + \|\mA\|_{\rm F}^2 - \langle \mB+\mB^\top, \mA \rangle - \|\mB\|_{\rm F}^2 - \|\mA\|_{\rm F}^2 + 2\langle \mB, \mA \rangle \\ 
	& = \frac{1}{4}\|\mB\|_{\rm F}^2 + \frac{1}{4}\|\mB^\top\|_{\rm F}^2 + \frac{1}{2} \langle \mB, \mB^\top \rangle - \|\mB\|_{\rm F}^2 + \langle \mB - \mB^\top, \mA \rangle \\ 
	& = \frac{1}{2} \langle \mB, \mB^\top \rangle  - \frac{1}{2}\|\mB\|_{\rm F}^2 + \langle \mB - \mB^\top, \mA \rangle \\ 
	& \leq \langle \mB - \mB^\top, \mA \rangle \\ 
	& = 0, 
\end{align*}
where the first inequality comes from the Cauchy-Schwartz inequality, and the last equality comes from the fact that $\mA$ is symmetric. \\ 

\noindent (ii) From (i), for any $\mA \in \R^{d\times d}$ we have 
$$
\mathbb{E} \|{\tilde {\cal C}}(\mA) - \mA\|_{\rm F}^2 \leq \mathbb{E} \|{\cal C}(\mA) - \mA\|_{\rm F}^2 \leq (1-\delta) \|\mA\|_{\rm F}^2. 
$$

\subsection{Proof of Proposition \ref{th:C1C2}}
From the definition of ${\cal C}_1$, we have 

\begin{align*}
	\mathbb{E} [\|{\cal C}_1(\mA) - \mA \|_{\rm F}^2] & = \mathbb{E} \|{\cal C}_1(\mA)\|_{\rm F}^2 + \|\mA\|_{\rm F}^2 - 2\mathbb{E}[\langle {\cal C}_1(\mA), \mA \rangle] \\ 
	& = \|\mA\|_{\rm F}^2 + \sum_{i=1}^R \mathbb{E} \left[  \frac{\sigma_i^2 {\cal Q}_2^i(b_i v_i)^\top {\cal Q}_2^i(b_i v_i) {\cal Q}_1^i(a_i u_i)^\top  {\cal Q}_1^i(a_i u_i) }{a_i^2 b_i^2 (\omega_1 +1)^2(\omega_2+1)^2}  \right] \\ 
	& \quad + \sum_{i,j\in[R], i\neq j} \mathbb{E} \left[  \frac{\sigma_i \sigma_j {\cal Q}_2^i(b_i v_i)^\top  {\cal Q}_2^j(b_j v_j) {\cal Q}_1^j(a_j u_j)^\top {\cal Q}_1^i(a_i u_i) }{a_ia_jb_ib_j(\omega_1 +1)^2(\omega_2+1)^2}  \right] \\ 
	& \quad - 2 \left \langle \sum_{i=1}^R \frac{\sigma_i u_i v_i^\top}{(\omega_1+1)(\omega_2+1)}, \mA \right \rangle \\ 
	& = \|\mA\|_{\rm F}^2 + \sum_{i=1}^R  \frac{\sigma_i^2 \mathbb{E} \|{\cal Q}_2^i(b_i v_i)\|^2 \cdot \mathbb{E} \|{\cal Q}_1^i(a_i u_i)\|^2    }{a_i^2 b_i^2 (\omega_1 +1)^2(\omega_2+1)^2} - 2 \left \langle \sum_{i=1}^R \frac{\sigma_i u_i v_i^\top}{(\omega_1+1)(\omega_2+1)}, \mA \right \rangle, 
\end{align*}
where in the last two equalities, we use the independence of each ${\cal Q}_1^i, {\cal Q}_2^i$, and the fact that $u_j^\top u_i=0$ and $v_i^\top v_j=0$ for $i\neq j$. From the definition of unbiased compressors, we further have 
\begin{align*}
	\mathbb{E} [\|{\cal C}_1(\mA) - \mA \|_{\rm F}^2] & \leq  \|\mA\|_{\rm F}^2 + \sum_{i=1}^R \frac{\sigma_i^2 \|u_i\|^2 \|v_i\|^2}{(\omega_1+1)(\omega_2+1)} - 2 \left \langle \sum_{i=1}^R \frac{\sigma_i u_i v_i^\top}{(\omega_1+1)(\omega_2+1)}, \mA \right \rangle \\ 
	& = \left(  1 - \frac{1}{(\omega_1+1)(\omega_2+1)}  \right) \|\mA\|_{\rm F}^2 \\ 
	& \quad + \frac{1}{(\omega_1+1)(\omega_2+1)} \left(  \|\mA\|_{\rm F}^2 + \sum_{i=1}^R \sigma_i^2 \|u_i\|^2 \|v_i\|^2 - 2 \left\langle \sum_{i=1}^R \sigma_iu_iv_i^\top, \mA \right\rangle  \right) \\ 
	& = \left(  1 - \frac{1}{(\omega_1+1)(\omega_2+1)}  \right) \|\mA\|_{\rm F}^2 + \frac{1}{(\omega_1+1)(\omega_2+1)} \left\| \sum_{i=1}^R \sigma_iu_iv_i^\top - \mA \right\|_{\rm F}^2 \\
	& \leq \left(  1 - \frac{1}{(\omega_1+1)(\omega_2+1)}  \right) \|\mA\|_{\rm F}^2 + \frac{(1-\nicefrac{R}{d})}{(\omega_1+1)(\omega_2+1)} \|\mA\|_{\rm F}^2 \\ 
	& = \left(  1 -  \frac{R}{d(\omega_1+1)(\omega_2+1)} \right) \|\mA\|_{\rm F}^2,
\end{align*}
where in the last inequality we use the fact that Rank-R is a contraction compressor with parameter $\nicefrac{R}{d}$ \citep{FedNL2021}. 

For ${\cal C}_2$, the result follows from Lemma \ref{lm:ABineq} (ii).

\subsection{Linear Independence of Outer Products}

\begin{lemma}\label{lem:outer-prod-ind}
	Let vectors $\{v_1,v_2,\dots,v_{r}\}\subset\R^d$ are linearly independent. Then outer products $\{v_i v_j^\top \colon i,j=1,2,\dots,r\}$ are linearly independent matrices in $\R^{d\times d}$.
\end{lemma}
\begin{proof}
	Let $\{e_1,e_2,\dots,e_d\}$ be the standard basis in $\R^d$. Then, for all $i\in[r]$
	$$
	v_i = \sum_{t=1}^r v_{it}e_t.
	$$
	
	Denote $\mE_{tl} = e_te_l^\top$. Suppose linear combination of matrices $\{v_i v_j^\top \colon i,j=1,2,\dots,r\}$ with some coefficients $c_{ij}$ is zero matrix. After simple transformations, we get
	\begin{eqnarray*}
		\bm{0}
		=   \sum_{i,j=1}^{r} c_{ij}v_iv_j^\top
		= \sum_{i,j=1}^{r} c_{ij} \sum_{t,l=1}^d v_{it}v_{jl} \mE_{tl}
		= \sum_{t,l=1}^d \[ \sum_{i,j=1}^{r} c_{ij} v_{it}v_{jl} \] \mE_{tl},
	\end{eqnarray*}
	which implies that
	$$
	\sum_{i,j=1}^{r} c_{ij} v_{it}v_{jl} = 0, \quad \text{for all } t,l\in[d].
	$$
	
	Then notice that
	$$
	0
	= \sum_{i,j=1}^{r} c_{ij} v_{it}v_{jl}
	= \sum_{i=1}^{r}\[\sum_{j=1}^{r} c_{ij}v_{jl}\]v_{it}
	= \sum_{i=1}^{r}c'_{il} v_{it}
	$$
	holds for all $t\in[d]$, which implies that $\sum_{i=1}^{r}c'_{il}v_i = 0$ (where that last 0 is a vector of size $d$). Since $v_i$'s are linearly independent, we get $c'_{il} = 0$ for all $i\in[d]$ and $l\in[r]$. By definition $c'_{il} = \sum_{j=1}^{r} c_{ij}v_{jl}$, hence $\sum_{j=1}^{r} c_{ij}v_{j} = 0$. Again using linear independence of $v_i$'s, we get $c_{ij}=0$ for all $i,j\in[d]$. Therefore outer products $v_iv_j^\top$ are also independent.
\end{proof}

\newpage
\section{Proofs for \algname{BL1}} 

We denote $\mathbb{E}_k[\cdot]$ as the conditional expectation on $z^k$, $w^k$, and $\mH_i^k$.

\subsection{Proof of Lemma \ref{lm:EstM-1}}

If $\|\nabla^2 f_i(x) - \nabla^2 f_i(y)\|_{\rm F} \leq H_1\|x-y\|$ for any $x, y\in \R^d$, and $i \in [n]$, then from (\ref{eq:hmA}) we have 
\begin{align*}
	\|{h}^i(\nabla^2 f_i(x)) - {h}^i(\nabla^2 f_i(y))\|_{\rm F} & = \| vec({h}^i(\nabla^2 f_i(x))) - vec({h}^i(\nabla^2 f_i(y)))\| \\ 
	& \leq \|{\cal B}_i^{-1}\| \cdot \|vec(\nabla^2 f_i(x)) - vec(\nabla^2 f_i(y))\| \\ 
	& =  \| {\cal B}_i^{-1}\| \cdot \|\nabla^2 f_i(x) - \nabla^2 f_i(y)\|_{\rm F} \\ 
	& \leq   \| {\cal B}_i^{-1}\| H_1 \|x-y\|, 
\end{align*}
which implies that $M_1$ in Assumption \ref{as:BL1} satisfies $M_1 \leq   \max_i \{ \| {\cal B}_i^{-1}\| \} H_1$.  \\ 

\noindent If $|(\nabla^2 f_i(x))_{jl} - (\nabla^2 f_i(y))_{jl} | \leq \nu \|x-y\|$ for any $x, y\in \R^d$, $i \in [n]$, and $j, l\in [d]$, then from (\ref{eq:hmA}), every entry of ${ h}^i(\nabla^2 f_i(x)) - {h}^i(\nabla^2 f_i(y))$ will be bounded by $\nu \| {\cal B}_i^{-1}\|_{\infty} \|x-y\|$. Hence $M_2$ in Assumption~\ref{as:BL1} satisfies $M_2 \leq \nu \max_{i} \{\| {\cal B}_i^{-1}\|_{\infty}\}$.

\subsection{Lemmas}

The proofs of Lemma \ref{lm:zQcomp} and Lemma \ref{lm:Hk-BL1} are the same as that of Lemma B.1 in \citep{FedNL2021}. Thus we omit them. 

\begin{lemma}\label{lm:zQcomp}
	Let ${\cal Q}$ be a compressor and $\eta >0$. For any $x, y, z\in \R^d$, we have following results. 
\begin{itemize}		
	\item[(i)] If ${\cal Q}$ is an unbiased compressor with parameter $\omega_{\rm M}$ and $\eta\leq \nicefrac{1}{(\omega_{\rm M}+1)}$, then 
	$$
	\mathbb{E} \| z + \eta {\cal Q}(x-z) -y   \|^2 \leq (1 - \eta)\|z-y\|^2 + \eta\|x-y\|^2, 
	$$
	where $\mathbb{E}[\cdot]$ is the expectation with respect to ${\cal Q}$. 
	\item[(ii)] If ${\cal Q}$ is a contraction compressor with parameter $\delta_{\rm M}$ and $\eta=1$, then 
	$$
	\mathbb{E} \| z + \eta {\cal Q}(x-z) -y   \|^2 \leq \left(1 - \frac{\delta_{\rm M}}{4} \right)\|z-y\|^2 + \left(  \frac{6}{\delta_{\rm M}} - \frac{7}{2}  \right)\|x-y\|^2,
	$$
\end{itemize}	
\end{lemma}

\begin{lemma}\label{lm:Hk-BL1}
	Let $\cC$ be a compressor and $\alpha>0$. For any matrix $\mL \in \R^{d\times d}$ and $y, z\in \R^d$, we have the following results. 
\begin{itemize}		
	\item[(i)] If $\cC$ is an unbiased compressor with parameter $\omega$ and $\alpha \leq \nicefrac{1}{\omega+1}$, then 
	$$
	\mathbb{E}\| \mL + \alpha \cC(h^i(\nabla^2 f_i(y)) - \mL) - h^i(\nabla^2 f_i(z)) \|^2_{\rm F} \leq (1-\alpha) \| \mL - h^i(\nabla^2 f_i(z))\|^2_{\rm F} + \alpha M_1^2 \|y-z\|^2, 
	$$
	where $\mathbb{E}[\cdot]$ is the expectation with respect to $\cC$. 
	\item[(ii)] If $\cC$ is a contraction compressor with parameter $\delta$ and $\alpha=1$, then 
	$$
	\mathbb{E}\| \mL + \alpha \cC(h^i(\nabla^2 f_i(y)) - \mL) - h^i(\nabla^2 f_i(z)) \|^2_{\rm F} \leq \left(  1 - \frac{\delta}{4}  \right) \| \mL - h^i(\nabla^2 f_i(z))\|^2_{\rm F} + \left(  \frac{6}{\delta} - \frac{7}{2}  \right) M_1^2 \|y-z\|^2. 
	$$
\end{itemize}		
\end{lemma}

\begin{lemma}\label{lm:Hkbound-BL1}
\begin{itemize}	We consider four cases:	 
	\item[(i)] If Assumption \ref{as:Qunbiasedcomp-BL1} (ii) holds, $\|x^0-x^*\|^2 \leq \min\{  \frac{\mu^2}{4d^2H^2}, \frac{M}{d}  \}$,  $\|z^k-x^*\|^2 \leq \min\{  \frac{\mu^2}{4dH^2}, M  \}$, and ${\cal H}^k \leq \frac{\mu^2}{4dN_{\rm B}R^2}$ for $k\leq K$ and any $M>0$, then $\|z^{K+1}-x^*\|^2 \leq \min\{  \frac{\mu^2}{4dH^2}, M  \}$. 
	\item[(ii)] If Assumption \ref{as:Qcontractioncomp-BL1} holds, ${\cal H}^K \leq \frac{A_{\rm M}\mu^2}{4N_{\rm B}R^2B_{\rm M}}$, $\|z^k-x^*\|^2 \leq \min\{  \frac{A_{\rm M}\mu^2}{4H^2 B_{\rm M}}, M  \}$ for $k\leq K$ and any $M>0$, then $\|z^{K+1} - x^*\|^2 \leq  \min\{  \frac{A_{\rm M}\mu^2}{4H^2 B_{\rm M}}, M  \}$. 
	\item[(iii)] If Assumption \ref{as:Cunbiasedcomp-BL1}(ii) holds, and $\|z^k-x^*\|^2 \leq \frac{M}{d^2M_2^2} $ for $k\leq K$ and any $M>0$, then ${\cal H}^{K} \leq M$. 
	\item[(iv)] If Assumption \ref{as:Ccontractioncomp-BL1} holds, ${\cal H}^K \leq M$, and $\|z^K-x^*\|^2 \leq \frac{AM}{BM_1^2}$ for any any $M>0$, then ${\cal H}^{K+1} \leq M$. 
	\end{itemize}	
\end{lemma}

\begin{proof}

	(i) If $\xi^k=1$, from (\ref{eq:11-BL1}), (\ref{eq:22-BL1}), and (\ref{eq:44-BL1}), we have 
	\begin{align*}
		\|x^{k+1} - x^*\|^2 & \leq \frac{1}{\mu^2} \left(  \frac{H^2}{2} \|z^k-x^*\|^2 + 2N_{\rm B} R^2 {\cal H}^k  \right) \|z^k-x^*\|^2 \\ 
		& \leq \frac{1}{d} \|z^k-x^*\|^2 \\ 
		& \leq \min\left\{  \frac{\mu^2}{4d^2H^2}, \frac{M}{d}  \right\}, 
	\end{align*}
	for $0\leq k\leq K$. \\ 
	
	If $\xi^k=0$, since we also have $\|w^k - x^*\|^2 \leq \min\{  \frac{\mu^2}{4dH^2}, M  \}$, then from (\ref{eq:11-BL1}), (\ref{eq:33-BL1}), and (\ref{eq:44-BL1}), we can get the above inequality in the same way. \\ 
	Since $\|x^0-x^*\|^2 \leq \min\{  \frac{\mu^2}{4d^2H^2}, \frac{M}{d}  \}$, we know $\|x^k-x^*\|^2 \leq \min\{  \frac{\mu^2}{4d^2H^2}, \frac{M}{d}  \}$ for all $0\leq k\leq K+1$. Then from Assumption \ref{as:Qunbiasedcomp-BL1} (ii), we can get 
	\begin{align*}
		\|z^{K+1} - x^*\|^2 & \leq d \max_{j} |z^{K+1}_j - x^*_j|^2 \\ 
		& \leq d \max_{0\leq t \leq K+1} \|x^t - x^*\|^2 \\ 
		& \leq \min\left\{  \frac{\mu^2}{4dH^2}, M  \right\}. 
	\end{align*}
	
	(ii) First, from the update rule of $w^k$, we know $\|w^k - x^*\|^2 \leq \min\{  \frac{A_{\rm M}\mu^2}{4H^2 B_{\rm M}}, M  \}$ for $k\leq K$. If $\xi^K=1$, from (\ref{eq:11-BL1}), (\ref{eq:22-BL1}), and (\ref{eq:44-BL1}), we have 
	\begin{align*}
		\|x^{K+1} - x^*\|^2 & \leq \frac{1}{\mu^2} \left(  \frac{H^2}{2} \|z^K-x^*\|^2 + 2N_{\rm B} R^2 {\cal H}^K  \right) \|z^K-x^*\|^2 \\ 
		& \leq \left(  \frac{A_{\rm M}}{8B_{\rm M}} +   \frac{A_{\rm M}}{2B_{\rm M}} \right) \|z^{K}-x^*\|^2 \\ 
		& \leq \frac{A_{\rm M}}{B_{\rm M}} \min\left\{  \frac{A_{\rm M}\mu^2}{4H^2 B_{\rm M}}, M  \right\} . 
	\end{align*}
	If $\xi^K=0$, from $\|w^K - x^*\|^2 \leq \min\{  \frac{A_{\rm M}\mu^2}{4H^2 B_{\rm M}}, M  \}$ and (\ref{eq:33-BL1}), we can obtain the above inequality similarly. Then from Lemma \ref{lm:zQcomp} (ii), we arrive at 
	\begin{align*}
		\|z^{K+1} - z^*\|^2 & \leq (1-A_{\rm M}) \|z^K-x^*\|^2 + B_{\rm M}\|x^{K+1} - x^*\|^2 \\ 
		& \leq (1-A_{\rm M}) \min\left\{  \frac{A_{\rm M}\mu^2}{4H^2 B_{\rm M}}, M  \right\} + A_{\rm M} \min\left\{  \frac{A_{\rm M}\mu^2}{4H^2 B_{\rm M}}, M  \right\} \\ 
		& =  \min\left\{  \frac{A_{\rm M}\mu^2}{4H^2 B_{\rm M}}, M  \right\}. 
	\end{align*}
	(iii) From Assumption \ref{as:Cunbiasedcomp-BL1}(ii), we have 
	\begin{align*}
		{\cal H}^K & = \frac{1}{n} \sum_{i=1}^n \|\mL_i^K - \mL_i^*\|^2_{\rm F} \\ 
		& \leq  \frac{1}{n} \sum_{i=1}^n d^2 \max_{jl} \{ |(\mL_i^K)_{jl} - (\mL_i^*)_{jl}|^2 \} \\ 
		& \leq d^2 M_2^2 \max_{0\leq t \leq K} \|z^t-x^*\|^2 \\ 
		& \leq M. 
	\end{align*}
	
	(iv) From Assumption \ref{as:Ccontractioncomp-BL1} and Lemma \ref{lm:Hk-BL1} (ii), we have 
	\begin{align*}
		\|\mL_i^{K+1} - \mL_i^*\|^2_{\rm F} & \leq (1-A) \|\mL_i^K - \mL_i^*\|^2_{\rm F} + BM_1^2 \|z^K - x^*\|^2 \\ 
		& \leq  (1-A)  \|\mL_i^K - \mL_i^*\|^2_{\rm F} + AM, 
	\end{align*}
	which implies that 
	$$
	{\cal H}^{K+1} = \frac{1}{n} \sum_{i=1}^n \|\mL_i^{K+1} - \mL_i^*\|^2_{\rm F} \leq (1-A)M + AM \leq M. 
	$$
\end{proof}

\subsection{Proof of Theorem \ref{th:linear-BL1}}

First we have 
\begin{align}
	\|x^{k+1} - x^*\|^2 & = \|z^k - x^* - [\mH^k]_\mu^{-1} g^k \|^2 \nonumber \\ 
	& = \left\| [\mH^k]_\mu^{-1} \left(  [\mH^k]_\mu (z^k - x^*) - (g^k - \nabla f(x^*))  \right)   \right\|^2 \nonumber \\ 
	& \leq \frac{1}{\mu^2} \left\|   [\mH^k]_\mu (z^k - x^*) - (g^k - \nabla f(x^*))   \right\|^2, \label{eq:11-BL1}
\end{align}
where we use $\nabla f(x^*) = 0$ in the second equality, and $\|[\mH^k]_\mu^{-1}\| \leq \frac{1}{\mu}$ in the last inequality. 

If $\xi^k = 1$, then 
\begin{align}
	& \quad \left\|   [\mH^k]_\mu (z^k - x^*) - (g^k - \nabla f(x^*))   \right\|^2 \nonumber \\ 
	& = \left\|  \nabla f(z^k) - \nabla f(x^*) - \nabla^2 f(x^*) (z^k-x^*) + (\nabla^2 f(x^*) - [\mH^k]_\mu) (z^k-x^*)  \right\|^2 \nonumber \\ 
	& \leq 2\left\|  \nabla f(z^k) - \nabla f(x^*) - \nabla^2 f(x^*) (z^k-x^*) \right\|^2  + 2\left\| (\nabla^2 f(x^*) - [\mH^k]_\mu) (z^k-x^*)  \right\|^2 \nonumber \\ 
	& \leq \frac{H^2}{2} \|z^k - x^*\|^4 + 2\| [\mH^k]_\mu - \nabla^2 f(x^*)\|^2 \cdot \|z^k-x^*\|^2 \nonumber \\ 
	& \leq \frac{H^2}{2} \|z^k - x^*\|^4 + 2\| \mH^k - \nabla^2 f(x^*)\|_{\rm F}^2 \|z^k-x^*\|^2 \nonumber \\ 
	& = \frac{H^2}{2} \|z^k - x^*\|^4 + 2 \left\| \frac{1}{n} \mH_i^k - \frac{1}{n} \nabla^2 f_i(x^*) \right\|^2_{\rm F} \|z^k - x^*\|^2 \nonumber \\
	& \leq \frac{H^2}{2} \|z^k - x^*\|^4 +  \frac{2}{n} \sum_{i=1}^n \| \mH_i^k - \nabla^2 f_i(x^*) \|^2_{\rm F} \|z^k-x^*\|^2, \label{eq:22-BL1}
\end{align}
where in the second inequality, we use the Lipschitz continuity of the Hessian of $f$, and in the last inequality, we use the convexity of $\|\cdot\|^2_{\rm F}$. 

If $\xi^k = 0$, then 
\begin{align}
	& \quad \left\|   [\mH^k]_\mu (z^k - x^*) - (g^k - \nabla f(x^*))   \right\|^2 \nonumber \\ 
	& = \left\|  [\mH^k]_\mu(z^k-w^k) + \nabla f(w^k) - \nabla f(x^*) - [\mH^k]_\mu (z^k - x^*)  \right\|^2 \nonumber \\ 
	& = \left\|  [\mH^k]_\mu(x^* - w^k) +   \nabla f(w^k) - \nabla f(x^*)  \right\|^2 \nonumber \\ 
	& = \left\| \nabla f(w^k) - \nabla f(x^*) - \nabla^2 f(x^*) (w^k-x^*) + (\nabla^2 f(x^*) - [\mH^k]_\mu) (w^k-x^*)   \right\|^2 \nonumber \\ 
	& \leq \frac{H^2}{2}\|w^k-x^*\|^4 +  2\| \mH^k - \nabla^2 f(x^*)\|_{\rm F}^2 \|w^k-x^*\|^2 \nonumber \\ 
	& \leq \frac{H^2}{2}\|w^k-x^*\|^4 +  \frac{2}{n} \sum_{i=1}^n \| \mH_i^k - \nabla^2 f_i(x^*) \|^2_{\rm F} \|w^k-x^*\|^2. \label{eq:33-BL1}
\end{align}

From the above three inequalities, we can obtain 
\begin{align*}
	\mathbb{E}_k \|x^{k+1} - x^*\|^2 & \leq \frac{H^2p}{2\mu^2} \|z^k - x^*\|^4 + \frac{2p}{n \mu^2} \sum_{i=1}^n \| \mH_i^k - \nabla^2 f_i(x^*) \|^2_{\rm F} \|z^k-x^*\|^2 \\ 
	& \quad + \frac{H^2(1-p)}{2\mu^2}\|w^k-x^*\|^4 +  \frac{2(1-p)}{n \mu^2} \sum_{i=1}^n \| \mH_i^k - \nabla^2 f_i(x^*) \|^2_{\rm F} \|w^k-x^*\|^2. 
\end{align*}

From the definition of $(A_{\rm M}, B_{\rm M})$ and Lemma \ref{lm:zQcomp}, by choosing $z=z^k$, $x=x^{k+1}$, and $y=x^*$ in Lemma \ref{lm:zQcomp}, we can obtain 
\begin{align*}
	\mathbb{E}_k\|z^{k+1} - x^*\|^2 & = \mathbb{E}_k\|z^k + \eta {\cal Q}^k(x^{k+1} - z^k) - x^*\|^2 \\ 
	& \leq (1 - A_{\rm M}) \|z^k - x^*\|^2 + B_{\rm M} \mathbb{E}_k\|x^{k+1} - x^*\|^2. 
\end{align*}

Combining the above two inequalities, we arrive at 
\begin{align*}
	\mathbb{E}_k\|z^{k+1} - x^*\|^2 & \leq (1 - A_{\rm M}) \|z^k - x^*\|^2 + \frac{B_{\rm M}p}{\mu^2} \left(  \frac{H^2}{2}\|z^k-x^*\|^2 + \frac{2}{n} \sum_{i=1}^n \| \mH_i^k - \nabla^2 f_i(x^*) \|^2_{\rm F}  \right) \|z^k-x^*\|^2 \\ 
	& \quad +  \frac{B_{\rm M}(1-p)}{\mu^2} \left(  \frac{H^2}{2}\|w^k-x^*\|^2 + \frac{2}{n} \sum_{i=1}^n \| \mH_i^k - \nabla^2 f_i(x^*) \|^2_{\rm F}   \right) \|w^k-x^*\|^2. 
\end{align*}

From the update rule of $\mH_i^k$, we know $\mH_i^k = \sum_{jl} (\mL_i^k)_{jl} \mB_i^{jl}$. Denote $\mL_i^* = h^i(\nabla^2 f_i(x^*))$. Then we have 
\begin{align}
	\|\mH_i^k - \nabla^2 f_i(x^*) \|^2_{\rm F}  = \left\| \sum_{jl} (\mL_i^k - \mL_i^*)_{jl} \mB_i^{jl} \right\|^2_{\rm F} & \leq N_{\rm B} \sum_{jl} \| (\mL_i^k - \mL_i^*)_{jl} \mB_i^{jl} \|^2_{\rm F} \nonumber \\ 
	& \leq N_{\rm B} R^2 \|\mL_i^k - \mL_i^*\|^2_{\rm F}. \label{eq:44-BL1}
\end{align}

Define ${\cal H}^k \eqdef \frac{1}{n} \sum_{i=1}^n \|\mL_i^k - \mL_i^*\|^2_{\rm F}$. Then we have 
\begin{align*}
	\mathbb{E}_k\|z^{k+1} - x^*\|^2 & \leq (1 - A_{\rm M}) \|z^k - x^*\|^2 + \frac{B_{\rm M}p}{\mu^2} \left(  \frac{H^2}{2}\|z^k-x^*\|^2 +2N_{\rm B} R^2 {\cal H}^k  \right) \|z^k-x^*\|^2 \\ 
	& \quad +  \frac{B_{\rm M}(1-p)}{\mu^2} \left(  \frac{H^2}{2}\|w^k-x^*\|^2 + 2N_{\rm B} R^2 {\cal H}^k  \right) \|w^k-x^*\|^2. 
\end{align*}

Assume $\|z^k-x^*\|^2 \leq \frac{A_{\rm M}\mu^2}{4H^2B_{\rm M}}$ and ${\cal H}^k \leq \frac{A_{\rm M}\mu^2}{16 N_{\rm B}R^2B_{\rm M}}$ for $k\geq 0$. Then from the update rule of $w^k$, we know $\|w^k-x^*\|^2 \leq \frac{A_{\rm M}\mu^2}{4H^2B_{\rm M}}$ for $k\geq 0$. Thus we have 
\begin{equation}\label{eq:zk+1-BL1}
	\mathbb{E}_k\|z^{k+1} - x^*\|^2  \leq \left(1 - A_{\rm M} + \frac{A_{\rm M}p}{4} \right) \|z^k - x^*\|^2 + \frac{A_{\rm M}(1-p)}{4}\|w^k-x^*\|^2. 
\end{equation}

From the update rule of $w^k$, we have 
\begin{equation}\label{eq:wk+1-BL1}
	\mathbb{E}_k\|w^{k+1} - x^*\|^2 = p\|z^k-x^*\|^2 + (1-p) \|w^k-x^*\|^2. 
\end{equation}

Define $\Phi^k_1 \eqdef \|z^k-x^*\|^2 + \frac{A_{\rm M}(1-p)}{2p}\|w^k-x^*\|^2$. Then we can get 
\begin{eqnarray*}
	\mathbb{E}_k [\Phi^{k+1}_1] & = & \mathbb{E}_k \|z^{k+1} - x^*\|^2 +  \frac{A_{\rm M}(1-p)}{2p} \mathbb{E}_k\|w^{k+1} - x^*\|^2 \\ 
	& \overset{(\ref{eq:zk+1-BL1})}{\leq} & \left(1 - A_{\rm M} + \frac{A_{\rm M}p}{4} \right) \|z^k - x^*\|^2 + \frac{A_{\rm M}(1-p)}{4}\|w^k-x^*\|^2 +  \frac{A_{\rm M}(1-p)}{2p} \mathbb{E}_k\|w^{k+1} - x^*\|^2 \\ 
	& \overset{(\ref{eq:wk+1-BL1})}{\leq} & \left(  1 - \frac{A_{\rm M}}{2}  \right) \|z^k-x^*\|^2 + \left(  1 - \frac{p}{2}  \right)  \frac{A_{\rm M}(1-p)}{2p} \|w^k-x^*\|^2 \\ 
	& \leq & \left(  1 - \frac{\min\{  A_{\rm M}, p  \}}{2}  \right) \Phi^k_1. 
\end{eqnarray*}

By applying the tower property, we have 
$$
\mathbb{E} [\Phi^{k+1}_1]  \leq \left(  1 - \frac{\min\{  A_{\rm M}, p  \}}{2}  \right) \mathbb{E}[\Phi^k_1]. 
$$
Unrolling the recursion, we can get the result.

\subsection{Proof of Theorem \ref{th:superlinear-BL1}}

Since $\xi^k \equiv 1$, $\eta=1$, and ${\cal Q}^k(x) \equiv x$ for any $x\in \R^d$, it is easy to see that $z^k \equiv x^k$ for $k\geq 0$. In this case, we can view ${\cal Q}^k$ as an unbiased compressor with $\omega_{\rm M}=0$ or a contraction compressor with $\delta_{\rm M}=1$. Then from (\ref{eq:zk+1-BL1}), we have 
$$
\mathbb{E}_k \|x^{k+1}-x^*\|^2 \leq \left(  1 - \frac{A_{\rm M}}{2}  \right) \|x^k-x^*\|^2. 
$$

From Lemma \ref{lm:Hk-BL1}, we can obtain 
$$
\mathbb{E}_k[{\cal H}^{k+1}] \leq  (1-A) {\cal H}^k + BM_1^2 \|x^k-x^*\|^2. 
$$

Thus, 
\begin{align*}
	\mathbb{E}_k[\Phi_2^{k+1}] & = \mathbb{E}_k[{\cal H}^{k+1}] + \frac{4BM_1^2}{A_{\rm M}}\mathbb{E}_k\|x^{k+1} - x^*\|^2 \\ 
	& \leq (1-A) {\cal H}^k + BM_1^2 \|x^k-x^*\|^2 +  \frac{4BM_1^2}{A_{\rm M}}\left(  1 - \frac{A_{\rm M}}{2}  \right) \|x^k-x^*\|^2 \\ 
	& \leq \left(  1 - \frac{\min\{  4A, A_{\rm M}  \} }{4} \right) \Phi_2^k. 
\end{align*}

By applying the tower property, we have $\mathbb{E}[\Phi_2^{k+1}] \leq \theta_1 \mathbb{E}[\Phi_2^k]$. Unrolling the recursion, we have $\mathbb{E}[\Phi_2^k] \leq \theta_1^k \Phi_2^0$. 

Then we further have $\mathbb{E}[{\cal H}^k] \leq \theta_1^k \Phi_2^0$ and $\mathbb{E}\|x^k-x^*\|^2 \leq \frac{A_{\rm M}}{4BM_1^2} \theta_1^k \Phi_2^0$. From $z^k\equiv x^k$, (\ref{eq:11-BL1}), and (\ref{eq:22-BL1}), we can get 
$$
\|x^{k+1} - x^*\|^2 \leq \frac{1}{\mu^2} \left(  \frac{H^2}{2} \|x^k-x^*\|^2 + 2N_{\rm B}R^2 {\cal H}^k  \right) \|x^k-x^*\|^2. 
$$
Assume $x^k\neq x^*$ for all $k\geq 0$. Then we have 
$$
\frac{\|x^{k+1}-x^*\|^2}{\|x^k-x^*\|^2} \leq \frac{1}{\mu^2} \left(  \frac{H^2}{2} \|x^k-x^*\|^2 + 2N_{\rm B}R^2 {\cal H}^k  \right), 
$$
and by taking expectation, we arrive at 
\begin{align*}
	\mathbb{E} \left[   \frac{\|x^{k+1}-x^*\|^2}{\|x^k-x^*\|^2} \right] & \leq \frac{H^2}{2\mu^2} \mathbb{E}\|x^k-x^*\|^2 + \frac{2N_{\rm B}R^2}{\mu^2} \mathbb{E}[{\cal H}^k] \\ 
	& \leq \theta_1^k \left(  \frac{A_{\rm M}H^2}{8BM_1^2\mu^2} + \frac{2N_{\rm B}R^2}{\mu^2}  \right) \Phi_2^0. 
\end{align*}

\subsection{Proof of Theorem \ref{th:nbor-BL1}}

(i) Noticed that under Assumption \ref{as:Qunbiasedcomp-BL1}, we have $A_{\rm M} = B_{\rm M} = \eta$. We prove this by mathematical induction. First, since $z^0=x^0$, we know $\|z^0-x^*\|^2 \leq \min\{ \frac{\mu^2}{4dH^2},  \frac{\mu^2}{16d^3 N_{\rm B}R^2M_2^2}  \}$. Then from Lemma \ref{lm:Hkbound-BL1} (iii), we have ${\cal H}^0 \leq \frac{\mu^2}{16dN_{\rm B}R^2 }$. Next, assume 
$$
\|z^k-x^*\|^2 \leq \min\left\{ \frac{\mu^2}{4dH^2},  \frac{\mu^2}{16d^3 N_{\rm B}R^2M_2^2}  \right\} \quad  {\rm and} \quad  {\cal H}^k \leq \frac{\mu^2}{16dN_{\rm B}R^2 }, 
$$
for $k\leq K$. 
By choosing $M = \frac{\mu^2}{16d^3 N_{\rm B}R^2M_2^2}$ in Lemma \ref{lm:Hkbound-BL1} (i), we have $\|z^{K+1}-x^*\|^2 \leq \min\left\{ \frac{\mu^2}{4dH^2},  \frac{\mu^2}{16d^3 N_{\rm B}R^2M_2^2}  \right\} $. By further using Lemma \ref{lm:Hkbound-BL1} (iii), we can get ${\cal H}^{K+1} \leq \frac{\mu^2}{16dN_{\rm B}R^2 }$. \\ 

(ii) We prove the result by induction. Assume $\|z^k - x^*\|^2 \leq \min\left\{  \frac{A_{\rm M}\mu^2}{4H^2 B_{\rm M}},  \frac{AA_{\rm M}\mu^2}{16N_{\rm B}R^2 B_{\rm M}BM_1^2}  \right\}$ and ${\cal H}^k \leq \frac{A_{\rm M}\mu^2}{16N_{\rm B}R^2B_{\rm M}}$ for $k\leq K$. Then by Lemma \ref{lm:Hkbound-BL1} (iv), we have ${\cal H}^{K+1} \leq \frac{A_{\rm M}\mu^2}{16N_{\rm B}R^2B_{\rm M}}$. Moreover, by Lemma \ref{lm:Hkbound-BL1} (ii), we have $\|z^{K+1} - x^*\|^2 \leq \min\left\{  \frac{A_{\rm M}\mu^2}{4H^2 B_{\rm M}},  \frac{AA_{\rm M}\mu^2}{16N_{\rm B}R^2 B_{\rm M}BM_1^2}  \right\}$.

\newpage 
\section{Proofs for \algname{BL2}} 

We denote $\mathbb{E}_k[\cdot]$ as the conditonal expectation on $z_i^k$, $w_i^k$, $l_i^k$, $\mL_i^k$, and $\mH_i^k$.

\subsection{A lemma}

\begin{lemma}\label{lm:Hkbound-BL2}
\begin{itemize}	We consider four cases:	
	\item[(i)] If Assumption \ref{as:Qunbiasedcomp-BL1} (ii) holds, $\|x^0-x^*\|^2 \leq \min\{  \frac{\mu^2}{d^2(6H^2+24H_1^2)}, \frac{M}{d}  \}$,  $\|z_i^k-x^*\|^2 \leq \min\{  \frac{\mu^2}{d(6H^2+24H_1^2)}, M \}$, and ${\cal H}^k \leq \frac{\mu^2}{24dN_{\rm B}R^2}$ for $k\leq K$, $i\in [n]$, and any $M>0$, then $\|z_i^{K+1}-x^*\|^2 \leq \min\{  \frac{\mu^2}{d(6H^2+24H_1^2)}, M  \}$ for $i\in [n]$.  
	\item[(ii)] If Assumption \ref{as:Qcontractioncomp-BL1} holds, ${\cal H}^K \leq \frac{A_{\rm M}\mu^2}{24N_{\rm B}R^2B_{\rm M}}$, $\|z_i^k-x^*\|^2 \leq \min\{  \frac{A_{\rm M}\mu^2}{B_{\rm M}(6H^2 + 24H_1^2)}, M  \}$ for $k\leq K$, $i\in [n]$, and any $M>0$, then $\|z_i^{K+1} - x^*\|^2 \leq  \min\{  \frac{A_{\rm M}\mu^2}{B_{\rm M}(6H^2+24H_1^2)}, M  \}$ for $i\in [n]$. 
	\item[(iii)] If Assumption \ref{as:Cunbiasedcomp-BL1}(ii) holds, and $\|z_i^k-x^*\|^2 \leq \frac{M}{d^2M_2^2} $ for $k\leq K$, $i\in [n]$, and any $M>0$, then ${\cal H}^{K} \leq M$. 
	\item[(iv)] If Assumption \ref{as:Ccontractioncomp-BL1} holds, $\|\mL_i^K - \mL_i^*\|_{\rm F}^2 \leq M$, and $\|z_i^K-x^*\|^2 \leq \frac{AM}{BM_1^2}$ for $i\in [n]$ and any $M>0$, then $\|\mL_i^{K+1} - \mL_i^*\|_{\rm F}^2 \leq M$ for $i\in[n]$. 
\end{itemize}	
\end{lemma}

\begin{proof}
	
	(i) First, from the update rule of $w_i^k$, we know ${\cal Z}^k \leq \min\{  \frac{\mu^2}{d(6H^2+24H_1^2)}, M \}$ and ${\cal W}^k \leq \min\{  \frac{\mu^2}{d(6H^2+24H_1^2)}, M  \}$ for $k\leq K$. Then from (\ref{eq:xk+1-BL2}), we have 
	\begin{align*}
		\|x^{k+1} - x^*\|^2 & \leq \frac{3H^2}{4\mu^2}({\cal W}^k)^2  + \frac{12N_{\rm B}R^2}{\mu^2} {\cal H}^k {\cal W}^k +  \frac{3H_1^2}{\mu^2} {\cal Z}^k {\cal W}^k \\ 
		& \leq \frac{1}{d} {\cal W}^k \\ 
		& \leq \min \left\{  \frac{\mu^2}{d^2(6H^2+24H_1^2)}, \frac{M}{d}  \right\}, 
	\end{align*}
	for $0\leq k\leq K$. \\ 
	Since $\|x^0-x^*\|^2 \leq \min \{  \frac{\mu^2}{d^2(6H^2+24H_1^2)}, \frac{M}{d}  \}$, we know $\|x^k-x^*\|^2 \leq \min \{  \frac{\mu^2}{d^2(6H^2+24H_1^2)}, \frac{M}{d} \}$ for all $0\leq k\leq K+1$. Then for $i\in S^k$, from Assumption \ref{as:Qunbiasedcomp-BL1} (ii), we can get 
	\begin{align*}
		\|z_i^{K+1} - x^*\|^2 & \leq d \max_{j} |(z^{K+1}_i)_j - x^*_j|^2 \\ 
		& \leq d \max_{0\leq t \leq K+1} \|x^t - x^*\|^2 \\ 
		& \leq \min \left\{  \frac{\mu^2}{d(6H^2+24H_1^2)},  M  \right\}. 
	\end{align*}
	For $i \notin S^k$, we have 
	$$
	\|z_i^{K+1} - x^*\|^2 = \|z_i^{K} - x^*\|^2 \leq \min \left\{  \frac{\mu^2}{d(6H^2+24H_1^2)},  M  \right\}. 
	$$
	
	(ii) First, from the update rule of $w^k$, we know ${\cal Z}^k \leq \min\{  \frac{A_{\rm M}\mu^2}{B_{\rm M}(6H^2 + 24H_1^2)}, M  \}$ and ${\cal W}^k \leq \min\{  \frac{A_{\rm M}\mu^2}{B_{\rm M}(6H^2 + 24H_1^2)}, M  \}$ for $k\leq K$. Then from (\ref{eq:xk+1-BL2}), we have 
	\begin{align*}
		\|x^{K+1} - x^*\|^2 & \leq \frac{3H^2}{4\mu^2}({\cal W}^K)^2  + \frac{12N_{\rm B}R^2}{\mu^2} {\cal H}^K {\cal W}^K +  \frac{3H_1^2}{\mu^2} {\cal Z}^K {\cal W}^K \\ 
		& \leq \frac{A_{\rm M}}{B_{\rm M}} {\cal W}^K \\ 
		& \leq \frac{A_{\rm M}}{B_{\rm M}} \min \left\{  \frac{A_{\rm M}\mu^2}{B_{\rm M}(6H^2 + 24H_1^2)}, M  \right\} . 
	\end{align*}
	Then for $i\in S^k$, from Lemma \ref{lm:zQcomp} (ii), we arrive at 
	\begin{align*}
		\|z_i^{K+1} - z^*\|^2 & \leq (1-A_{\rm M}) \|z_i^K-x^*\|^2 + B_{\rm M}\|x^{K+1} - x^*\|^2 \\ 
		& \leq (1-A_{\rm M}) \min \left\{  \frac{A_{\rm M}\mu^2}{B_{\rm M}(6H^2 + 24H_1^2)}, M  \right\} + A_{\rm M} \min \left\{  \frac{A_{\rm M}\mu^2}{B_{\rm M}(6H^2 + 24H_1^2)}, M  \right\} \\ 
		& =  \min \left\{  \frac{A_{\rm M}\mu^2}{B_{\rm M}(6H^2 + 24H_1^2)}, M  \right\}. 
	\end{align*}
	For $i \notin S^k$, we have 
	$$
	\|z_i^{K+1} - x^*\|^2 = \|z_i^{K} - x^*\|^2 \leq \min \left\{  \frac{A_{\rm M}\mu^2}{B_{\rm M}(6H^2 + 24H_1^2)}, M  \right\}. 
	$$
	
	(iii) From Assumption \ref{as:Cunbiasedcomp-BL1}(ii), we have 
	\begin{align*}
		{\cal H}^K & = \frac{1}{n} \sum_{i=1}^n \|\mL_i^K - \mL_i^*\|^2_{\rm F} \\ 
		& \leq  \frac{1}{n} \sum_{i=1}^n d^2 \max_{jl} \{ |(\mL_i^K)_{jl} - (\mL_i^*)_{jl}|^2 \} \\ 
		& \leq d^2 M_2^2 \max_{i\in [n], 0\leq t \leq K} \|z_i^t-x^*\|^2 \\ 
		& \leq M. 
	\end{align*}
	
	(iv) For $i\in S^k$, from Assumption \ref{as:Ccontractioncomp-BL1} and Lemma \ref{lm:Hk-BL1} (ii), we have 
	\begin{align*}
		\|\mL_i^{K+1} - \mL_i^*\|^2_{\rm F} & \leq (1-A) \|\mL_i^K - \mL_i^*\|^2_{\rm F} + BM_1^2 \|z_i^K - x^*\|^2 \\ 
		& \leq  (1-A) M + AM\\ 
		& = M. 
	\end{align*}
	For $i \notin S^k$, we also have 
	$$
	\|\mL_i^{K+1} - \mL_i^*\|^2_{\rm F}  =  \|\mL_i^K - \mL_i^*\|^2_{\rm F} \leq M. 
	$$
\end{proof}

\subsection{Proof of Theorem \ref{th:BL2}}

First, from $[\mH_i^k]_s+l_i^k\mI \succeq \nabla^2 f_i(z_i^k)\succeq \mu \mI$, we know $[\mH^k]_s + l^k\mI = \frac{1}{n} \sum_{i=1}^n ([\mH_i^k]_s + l_i^k\mI) \succeq \mu\mI$. Then we have 
\begin{align*}
	\|x^{k+1} - x^*\| & = \left\| \left([\mH^k]_s + l^k\mI \right)^{-1} \left( g^k -  \left([\mH^k]_s + l^k\mI \right)x^* + \nabla f(x^*)  \right)  \right\|\\ 
	& \leq \frac{1}{\mu} \left\|  \frac{1}{n} \sum_{i=1}^n g_i^{k} - \frac{1}{n} \sum_{i=1}^n ([\mH_i^k]_s + l_i^k\mI)x^* + \frac{1}{n} \sum_{i=1}^n \nabla f_i(x^*)  \right\|\\ 
	& \leq \frac{1}{n\mu} \sum_{i=1}^n \| g_i^k - ([\mH_i^k]_s+l_i^k)x^* + \nabla f_i(x^*) \| \\ 
	& = \frac{1}{n\mu} \sum_{i=1}^n \| ([\mH_i^k]_s + l_i^k\mI) w_i^k - \nabla f_i(w_i^k) - ([\mH_i^k]_s+l_i^k\mI)x^* + \nabla f_i(x^*) \| \\ 
	& \leq \frac{1}{n\mu} \sum_{i=1}^n \|\nabla f_i(w_i^k)-\nabla f_i(x^*) - \nabla^2 f_i(x^*) (w_i^k-x^*)\| \\ 
	& \quad +  \frac{1}{n\mu} \sum_{i=1}^n \|([\mH_i^k]_s + l_i^k\mI - \nabla^2 f_i(x^*)) (w_i^k-x^*)\|, 
\end{align*}
where we use $\nabla f(x^*)=0$ in the first equality and $g_i^k = ([\mH_i^k]_s + l_i^k\mI) w_i^k - \nabla f_i(w_i^k)$ in the second equality. Since $\|\nabla^2 f_i(x) - \nabla^2 f_i(y)\| \leq H\|x-y\|$ for any $x, y\in \R^d$, we further have 
\begin{eqnarray*}
	\|x^{k+1} - x^*\| &\leq& \frac{H}{2n\mu} \sum_{i=1}^n \|w_i^k-x^*\|^2 + \frac{1}{n\mu} \sum_{i=1}^n \|([\mH_i^k]_s + l_i^k\mI - \nabla^2 f_i(x^*))\| \cdot (w_i^k-x^*)\| \\ 
	&\leq& \frac{H}{2\mu} {\cal W}^k + \frac{1}{n\mu} \sum_{i=1}^n \left(  \|[\mH_i^k]_s - \nabla^2 f_i(x^*)\|_{\rm F} + l_i^k  \right) \|w_i^k-x^*\|,
\end{eqnarray*}
where ${\cal W}^k \eqdef \frac{1}{n} \sum_{i=1}^n \|w_i^k-x^*\|^2$. 

Since $\|\nabla^2 f_i(x) - \nabla^2 f_i(y)\|_{\rm F} \leq H_1\|x-y\|$ for any $x, y\in \R^d$, we can get 
\begin{align*}
	l_i^k & = \|[\mH_i^k]_s-\nabla^2 f_i(z_i^k)\|_{\rm F} \\ 
	& \leq \|[\mH_i^k]_s - \nabla^2 f_i(x^*)\|_{\rm F} + \|\nabla^2 f_i(z_i^k)- \nabla^2 f_i(x^*)\|_{\rm F} \\ 
	& \leq \|[\mH_i^k]_s - \nabla^2 f_i(x^*)\|_{\rm F} + H_1 \|z_i^k-x^*\|. 
\end{align*}

Thus, 
\begin{eqnarray*}
	\|x^{k+1} - x^*\| &\leq& \frac{H}{2\mu} {\cal W}^k + \frac{1}{n\mu} \sum_{i=1}^n \left(  \|[\mH_i^k]_s - \nabla^2 f_i(x^*)\| + \|[\mH_i^k]_s - \nabla^2 f_i(x^*)\|_{\rm F} + H_1 \|z_i^k-x^*\|  \right) \|w_i^k-x^*\| \\ 
	&\leq&  \frac{H}{2\mu} {\cal W}^k + \frac{2}{n\mu} \sum_{i=1}^n \|[\mH_i^k]_s - \nabla^2 f_i(x^*)\|_{\rm F}  \|w_i^k-x^*\| + \frac{H_1}{n\mu} \sum_{i=1}^n \|z_i^k-x^*\| \|w_i^k-x^*\| \\ 
	&\overset{Lemma~ \ref{lm:ABineq}}{\leq}& \frac{H}{2\mu} {\cal W}^k + \frac{2}{n\mu} \sum_{i=1}^n \|\mH_i^k - \nabla^2 f_i(x^*)\|_{\rm F}  \|w_i^k-x^*\| + \frac{H_1}{n\mu} \sum_{i=1}^n \|z_i^k-x^*\| \|w_i^k-x^*\|    \\
	&\leq& \frac{H}{2\mu} {\cal W}^k  + \frac{2}{n\mu} \left(  \sum_{i=1}^n \|\mH_i^k-\nabla^2 f_i(x^*)\|_{\rm F}^2  \right)^{\frac{1}{2}} \left(  n {\cal W}^k  \right)^{\frac{1}{2}} + \frac{H_1}{n\mu} \left(  n {\cal Z}^k  \right)^{\frac{1}{2}} \left(  n {\cal W}^k  \right)^{\frac{1}{2}}, 
\end{eqnarray*}
where we use the Cauchy-Schwarz inequality in the last inequality and ${\cal Z}^k \eqdef \frac{1}{n} \sum_{i=1}^n \|z_i^k-x^*\|^2$. Since $\mH_i^k = \sum_{jl} (\mL_i^k)_{jl} \mB_i^{jl}$, same as (\ref{eq:44-BL1}), we have 
\begin{equation}\label{eq:HikD-BL2}
	\|\mH_i^k - \nabla^2 f_i(x^*)\|_{\rm F}^2 \leq N_{\rm B} R^2 \|\mL_i^k - \mL_i^*\|^2_{\rm F}, 
\end{equation}
where $\mL_i^* = h^i(\nabla^2 f_i(x^*))$ and $N_{\rm B}$ is defined in \ref{eq:N_B-BL1}. Then from the convexity of $\|\cdot\|^2$, we further bound $\|x^{k+1}-x^*\|^2$ as 
\begin{align}
	\|x^{k+1}-x^*\|^2 & \leq \frac{3H^2}{4\mu^2}({\cal W}^k)^2 + \frac{12{\cal W}^k}{n\mu^2} \sum_{i=1}^n \|\mH_i^k-\nabla^2 f_i(x^*)\|^2_{\rm F} + \frac{3H_1^2}{\mu^2} {\cal Z}^k {\cal W}^k \nonumber \\ 
	& \overset{(\ref{eq:HikD-BL2})}{\leq} \frac{3H^2}{4\mu^2}({\cal W}^k)^2 + \frac{12N_{\rm B}R^2{\cal W}^k}{n\mu^2} \sum_{i=1}^n \|\mL_i^k-\mL_i^*\|_{\rm F}^2 + \frac{3H_1^2}{\mu^2} {\cal Z}^k {\cal W}^k \nonumber \\ 
	& = \frac{3H^2}{4\mu^2}({\cal W}^k)^2  + \frac{12N_{\rm B}R^2}{\mu^2} {\cal H}^k {\cal W}^k +  \frac{3H_1^2}{\mu^2} {\cal Z}^k {\cal W}^k, \label{eq:xk+1-BL2}
\end{align}
where ${\cal H}^k = \frac{1}{n} \sum_{i=1}^n \|\mL_i^k-\mL_i^*\|_{\rm F}^2$. 

For $i\in S^k$, we have $z_i^{k+1} = z_i^k + \eta {\cal Q}_i^k (x^{k+1}-z_i^k)$. Then from the definition of $(A_{\rm M}, B_{\rm M})$ and Lemma \ref{lm:zQcomp}, by choosing $z=z_i^k$, $x=x^{k+1}$, and $y=x^*$ in Lemma \ref{lm:zQcomp}, we can obtain 
\begin{align*}
	\mathbb{E}_k[\|z_i^{k+1} - x^*\|^2 \ | \  i\in S^k ]& = \mathbb{E}_k [\|z_i^k + \eta {\cal Q}^k(x^{k+1} - z_i^k) - x^*\|^2  \ | \  i\in S^k ]\\ 
	& \leq (1 - A_{\rm M}) \|z_i^k - x^*\|^2 + B_{\rm M} \mathbb{E}_k\|x^{k+1} - x^*\|^2 \\ 
	& = (1 - A_{\rm M}) \|z_i^k - x^*\|^2 + B_{\rm M} \|x^{k+1} - x^*\|^2. 
\end{align*}

Noticing that $\mathbb{P}[i \in S^k] = \nicefrac{\tau}{n}$ and $z_i^{k+1}=z_i^k$ for $i\notin S^k$, we further have 
\begin{align*}
	\mathbb{E}_k\|z_i^{k+1} - x^*\|^2 & = \frac{\tau}{n} \mathbb{E}_k[\|z_i^{k+1} - x^*\|^2 \ | \  i\in S^k ] + \left(  1 - \frac{\tau}{n}  \right) \mathbb{E}_k[\|z_i^{k+1} - x^*\|^2 \ | \  i\notin S^k ] \\ 
	& \leq \frac{\tau}{n} (1 - A_{\rm M}) \|z_i^k - x^*\|^2 + \frac{\tau B_{\rm M}}{n} \|x^{k+1} - x^*\|^2 + \left(  1 - \frac{\tau}{n}  \right) \|z_i^k-x^*\|^2 \\ 
	& = \left(  1 - \frac{\tau A_{\rm M}}{n}  \right) \|z_i^k-x^*\|^2 + \frac{\tau B_{\rm M}}{n} \|x^{k+1} - x^*\|^2, 
\end{align*}
which implies that 
\begin{align}
	\mathbb{E}_k [{\cal Z}^{k+1}] & = \frac{1}{n} \sum_{i=1}^n \mathbb{E}_k \|z_i^{k+1}-x^*\|^2  \nonumber \\ 
	& \leq \frac{1}{n} \sum_{i=1}^n \left(  1 - \frac{\tau A_{\rm M}}{n}  \right) \|z_i^k-x^*\|^2 + \frac{\tau B_{\rm M}}{n} \|x^{k+1} - x^*\|^2  \nonumber \\ 
	& =  \left(  1 - \frac{\tau A_{\rm M}}{n}  \right) {\cal Z}^k +  \frac{\tau B_{\rm M}}{n} \|x^{k+1} - x^*\|^2. \label{eq:Zk+1-BL2}
\end{align}

For $i\in S^k$, from the update rule of $w_i^{k+1}$, we have 
\begin{align*}
	\mathbb{E}_k[\|w_i^{k+1} - x^*\|^2 \ | \  i\in S^k ]& = p \mathbb{E}_k [\|z_i^{k+1}-x^*\|^2] + (1-p) \|w_i^k-x^*\|^2. 
\end{align*}

For $i\notin S^k$, we have $w_i^{k+1} = w_i^k$. Thus, 
\begin{align*}
	\mathbb{E}_k \|w_i^{k+1} - x^*\|^2 & = \frac{\tau}{n} \mathbb{E}_k[\|w_i^{k+1} - x^*\|^2 \ | \  i\in S^k ] + \left(  1 - \frac{\tau}{n}  \right) \mathbb{E}_k[\|w_i^{k+1} - x^*\|^2 \ | \  i\notin S^k ] \\ 
	& = \left(  1 - \frac{\tau p}{n}  \right) \|w_i^k-x^*\|^2 + \frac{\tau p}{n} \mathbb{E}_k \|z_i^{k+1} - x^*\|^2, 
\end{align*}
which yields that 
\begin{align}
	\mathbb{E}_k [{\cal W}^{k+1}] & = \frac{1}{n} \sum_{i=1}^n \mathbb{E}_k\|w_i^{k+1}-x^*\|^2 \nonumber \\ 
	& =  \frac{1}{n} \sum_{i=1}^n \left(  1 - \frac{\tau p}{n}  \right) \|w_i^k-x^*\|^2 + \frac{1}{n} \sum_{i=1}^n \frac{\tau p}{n} \mathbb{E}_k \|z_i^{k+1} - x^*\|^2 \nonumber \\ 
	& = \left(  1 - \frac{\tau p}{n}  \right) {\cal W}^k + \frac{\tau p}{n} \mathbb{E}_k [{\cal Z}^{k+1}] \nonumber \\ 
	& \overset{(\ref{eq:Zk+1-BL2})}{\leq}  \left(  1 - \frac{\tau p}{n}  \right) {\cal W}^k + \frac{\tau p}{n} \left(  1 - \frac{\tau A_{\rm M}}{n}  \right) {\cal Z}^k + \frac{\tau^2B_{\rm M}p}{n^2} \|x^{k+1}-x^*\|^2. \label{eq:Wk+1-BL2}
\end{align}

Let $\Phi_3^k \eqdef {\cal W}^k + \frac{2p}{A_{\rm M}}\left(  1 - \frac{\tau A_{\rm M}}{n}  \right) {\cal Z}^k$ for $k\geq 0$. Then from the above inequality we have 
\begin{align*}
	\mathbb{E}_k[\Phi_3^{k+1}] & \leq \left(  1 - \frac{\tau p}{n}  \right) {\cal W}^k + \frac{\tau p}{n} \left(  1 - \frac{\tau A_{\rm M}}{n}  \right) {\cal Z}^k + \frac{\tau^2B_{\rm M}p}{n^2} \|x^{k+1}-x^*\|^2 +  \frac{2p}{A_{\rm M}}\left(  1 - \frac{\tau A_{\rm M}}{n}  \right) \mathbb{E}_k[{\cal Z}^{k+1}] \\ 
	& \overset{(\ref{eq:Zk+1-BL2})}{\leq} \left(  1 - \frac{\tau p}{n}  \right) {\cal W}^k + \frac{2p}{A_{\rm M}}\left(  1 - \frac{\tau A_{\rm M}}{n}  \right) \left(  1 - \frac{\tau A_{\rm M}}{2n}  \right) {\cal Z}^k + \frac{2\tau p B_{\rm M}}{n A_{\rm M}} \left(  1 - \frac{\tau A_{\rm M}}{2n}  \right) \|x^{k+1}-x^*\|^2 \\ 
	& \overset{\ref{eq:xk+1-BL2}}{\leq} \left(  1 - \frac{\tau p}{n} + \frac{2\tau pB_{\rm M}}{nA_{\rm M}} \left(  \frac{3H^2}{4\mu^2}{\cal W}^k  + \frac{12N_{\rm B}R^2}{\mu^2} {\cal H}^k  +  \frac{3H_1^2}{\mu^2} {\cal Z}^k   \right)  \right) {\cal W}^k \\ 
	& \quad + \frac{2p}{A_{\rm M}}\left(  1 - \frac{\tau A_{\rm M}}{n}  \right) \left(  1 - \frac{\tau A_{\rm M}}{2n}  \right) {\cal Z}^k. 
\end{align*}

If $\|z_i^k-x^*\|^2 \leq \frac{A_{\rm M}\mu^2}{(6H^2 + 24H_1^2)B_{\rm M}}$ and ${\cal H}^k \leq \frac{A_{\rm M}\mu^2}{96N_{\rm B}R^2 B_{\rm M}}$ for all $k\geq 0$, then we have 
$$
\frac{3H^2}{4\mu^2}{\cal W}^k  + \frac{12N_{\rm B}R^2}{\mu^2} {\cal H}^k  +  \frac{3H_1^2}{\mu^2} {\cal Z}^k \leq \frac{A_{\rm M}}{4B_{\rm M}}, 
$$
which implies that 
\begin{align*}
	\mathbb{E}_k[\Phi_3^{k+1}] & \leq \left(  1 - \frac{\tau p}{2n}  \right) {\cal W}^k + \frac{2p}{A_{\rm M}}\left(  1 - \frac{\tau A_{\rm M}}{n}  \right) \left(  1 - \frac{\tau A_{\rm M}}{2n}  \right) {\cal Z}^k \\ 
	& \leq \left(  1 - \frac{\tau \min\{  p, A_{\rm M}  \}}{2n}  \right) \Phi_3^k. 
\end{align*}

By applying the tower property, we have 
$$
\mathbb{E} [\Phi_3^{k+1}] \leq \left(  1 - \frac{\tau \min\{  p, A_{\rm M}  \}}{2n}  \right) \mathbb{E}[\Phi_3^k]. 
$$

Unrolling the recursion, we can obtain the result.

\subsection{Proof of Theorem \ref{th:superlinear-BL2}}

Since $\xi^k \equiv 1$, $\eta=1$, $S^k \equiv [n]$, and ${\cal Q}_i^k(x) \equiv x$ for any $x\in \R^d$, it is easy to see that $w_i^k = z_i^k \equiv x^k$ for all $i\in[n]$ and $k\geq 0$. In this case, we can view ${\cal Q}_i^k$ as an unbiased compressor with $\omega_{\rm M}=0$ or a contraction compressor with $\delta_{\rm M}=1$. Then from (\ref{eq:Zk+1-BL2}), we have 
\begin{align*}
	\mathbb{E}_k \|x^{k+1}-x^*\|^2 & \leq \left(  1 - A_{\rm M}  \right) \|x^k-x^*\|^2 + B_{\rm M} \|x^{k+1}-x^*\|^2 \\ 
	& \overset{(\ref{eq:xk+1-BL2})}{\leq} \left(  1 - A_{\rm M}  \right) \|x^k-x^*\|^2 + \frac{1}{4}A_{\rm M} \|x^k-x^*\|^2 \\ 
	& = \left(  1 - \frac{3A_{\rm M}}{4}  \right) \|x^k-x^*\|^2. 
\end{align*}

From Lemma \ref{lm:Hk-BL1}, we can obtain 
$$
\mathbb{E}_k[{\cal H}^{k+1}] \leq  (1-A) {\cal H}^k + BM_1^2 \|x^k-x^*\|^2. 
$$

Thus, 
\begin{align*}
	\mathbb{E}_k[\Phi_4^{k+1}] & = \mathbb{E}_k[{\cal H}^{k+1}] + \frac{4BM_1^2}{A_{\rm M}}\mathbb{E}_k\|x^{k+1} - x^*\|^2 \\ 
	& \leq (1-A) {\cal H}^k + BM_1^2 \|x^k-x^*\|^2 +  \frac{4BM_1^2}{A_{\rm M}}\left(  1 - \frac{3A_{\rm M}}{4}  \right) \|x^k-x^*\|^2 \\ 
	& \leq \left(  1 - \frac{\min\{  2A, A_{\rm M}  \} }{2} \right) \Phi_4^k. 
\end{align*}

By applying the tower property, we have $\mathbb{E}[\Phi_4^{k+1}] \leq \theta_2 \mathbb{E}[\Phi_4^k]$. Unrolling the recursion, we have $\mathbb{E}[\Phi_4^k] \leq \theta_2^k \Phi_4^0$. 

Then we further have $\mathbb{E}[{\cal H}^k] \leq \theta_2^k \Phi_4^0$ and $\mathbb{E}\|x^k-x^*\|^2 \leq \frac{A_{\rm M}}{4BM_1^2} \theta_2^k \Phi_4^0$. From $w_i^k=z_i^k\equiv x^k$ and (\ref{eq:xk+1-BL2}), we can get 
$$
\|x^{k+1} - x^*\|^2 \leq  \left(  \frac{3H^2+12H_1^2}{4\mu^2} \|x^k-x^*\|^2 + \frac{12N_{\rm B}R^2}{\mu^2} {\cal H}^k  \right) \|x^k-x^*\|^2. 
$$
Assume $x^k\neq x^*$ for all $k\geq 0$. Then we have 
$$
\frac{\|x^{k+1}-x^*\|^2}{\|x^k-x^*\|^2} \leq  \frac{3H^2+12H_1^2}{4\mu^2} \|x^k-x^*\|^2 + \frac{12N_{\rm B}R^2}{\mu^2} {\cal H}^k, 
$$
and by taking expectation, we arrive at 
\begin{align*}
	\mathbb{E} \left[   \frac{\|x^{k+1}-x^*\|^2}{\|x^k-x^*\|^2} \right] & \leq \frac{3H^2+12H_1^2}{4\mu^2} \mathbb{E}\|x^k-x^*\|^2 + \frac{12N_{\rm B}R^2}{\mu^2} \mathbb{E}[{\cal H}^k] \\ 
	& \leq \theta_2^k \left(  \frac{A_{\rm M}(3H^2+12H_1^2)}{16BM_1^2\mu^2} + \frac{12N_{\rm B}R^2}{\mu^2}  \right) \Phi_4^0. 
\end{align*}

\subsection{Proof of Theorem \ref{th:nbor-BL2}}

(i) Noticed that under Assumption \ref{as:Qunbiasedcomp-BL1}, we have $A_{\rm M} = B_{\rm M} = \eta$. We prove this by mathematical induction. First, since $z_i^0=x^0$, we know $\|z_i^0-x^*\|^2 \leq \min\{ \frac{\mu^2}{d(6H^2+24H_1^2)},  \frac{\mu^2}{96d^3 N_{\rm B}R^2M_2^2}  \}$ for $i\in [n]$. Then from Lemma \ref{lm:Hkbound-BL2} (iii), we have ${\cal H}^0 \leq \frac{\mu^2}{96dN_{\rm B}R^2 }$. Next, assume 
$$
\|z_i^k-x^*\|^2 \leq \min\left\{ \frac{\mu^2}{d(6H^2+24H_1^2)},  \frac{\mu^2}{96d^3 N_{\rm B}R^2M_2^2}  \right\} \ {\rm for} \ i\in[n] \quad  {\rm and} \quad  {\cal H}^k \leq \frac{\mu^2}{96dN_{\rm B}R^2 }, 
$$
for $k\leq K$. 
By choosing $M = \frac{\mu^2}{96d^3 N_{\rm B}R^2M_2^2}$ in Lemma \ref{lm:Hkbound-BL2} (i), we have $$\|z_i^{K+1}-x^*\|^2 \leq \min\left\{ \frac{\mu^2}{d(6H^2+24H_1^2)},  \frac{\mu^2}{96d^3 N_{\rm B}R^2M_2^2}  \right\}, $$ for $i\in[n]$. By further using Lemma \ref{lm:Hkbound-BL2} (iii), we can get ${\cal H}^{K+1} \leq \frac{\mu^2}{96dN_{\rm B}R^2 }$. \\ 

(ii) We prove the result by induction. Assume $\|z_i^k - x^*\|^2 \leq \min\left\{  \frac{A_{\rm M}\mu^2}{B_{\rm M}(6H^2+24H_1^2)},  \frac{AA_{\rm M}\mu^2}{96N_{\rm B}R^2 B_{\rm M}BM_1^2}  \right\}$ and $\|\mL_i^k - \mL_i^*\|^2_{\rm F} \leq \frac{A_{\rm M}\mu^2}{96N_{\rm B}R^2B_{\rm M}}$ for all $i\in[n]$ and $k\leq K$. Then by Lemma \ref{lm:Hkbound-BL2} (iv), we have $\|\mL_i^{K+1} - \mL_i^*\|^2_{\rm F} \leq \frac{A_{\rm M}\mu^2}{96N_{\rm B}R^2B_{\rm M}}$. Moreover, by Lemma \ref{lm:Hkbound-BL2} (ii), we have $\|z_i^{K+1} - x^*\|^2 \leq \min\left\{  \frac{A_{\rm M}\mu^2}{B_{\rm M}(6H^2+24H_1^2)},  \frac{AA_{\rm M}\mu^2}{96N_{\rm B}R^2 B_{\rm M}BM_1^2}  \right\}$ for $i\in[n]$.

\newpage 
\section{Proofs for \algname{BL3}}

We denote $\mathbb{E}_k[\cdot]$ as the conditonal expectation on $z_i^k$, $w_i^k$, $\mL_i^k$, $\gamma_i^k$, $\beta_i^k$, $\mA_i^k$ and $\mC_i^k$.

\subsection{Proof of Lemma \ref{lm:EstM-2}}

\noindent (i) If $\|\nabla^2 f_i(x) - \nabla^2 f_i(y)\|_{\rm F} \leq H_1\|x-y\|$ for any $x, y\in \R^d$, and $i \in [n]$, then from (\ref{eq:hmAsym}) we have 
\begin{align*}
	\|{\tilde h}^i(\nabla^2 f_i(x)) - {\tilde h}^i(\nabla^2 f_i(y))\|_{\rm F} & \leq \| svec({\tilde h}^i(\nabla^2 f_i(x))) - svec({\tilde h}^i(\nabla^2 f_i(y)))\| \\ 
	& \leq \|({\tilde {\cal B}}_i)^{-1}\| \cdot \|svec(\nabla^2 f_i(x)) - svec(\nabla^2 f_i(y))\| \\ 
	& \leq \sqrt{2} \|({\tilde {\cal B}}_i)^{-1}\| \cdot \|\nabla^2 f_i(x) - \nabla^2 f_i(y)\|_{\rm F} \\ 
	& \leq  \sqrt{2} \|({\tilde {\cal B}}_i)^{-1}\| H_1 \|x-y\|, 
\end{align*}
which implies that $M_4$ in Assumption \ref{as:BL3} satisfies $M_4 \leq  \sqrt{2} \max_i \{ \|({\tilde {\cal B}}_i)^{-1}\| \} H_1$. \\  

\noindent If $|(\nabla^2 f_i(x))_{jl} - (\nabla^2 f_i(y))_{jl} | \leq \nu \|x-y\|$ for any $x, y\in \R^d$, $i \in [n]$, and $j, l\in [d]$, then from (\ref{eq:hmAsym}), every entry of ${\tilde h}^i(\nabla^2 f_i(x)) - {\tilde h}^i(\nabla^2 f_i(y))$ will be bounded by $2\nu \|({\tilde {\cal B}}_i)^{-1}\|_{\infty} \|x-y\|$. Hence $M_5$ in Assumption \ref{as:BL3} satisfies $M_5 \leq 2\nu \max_{i} \{\|({\tilde {\cal B}}_i)^{-1}\|_{\infty}\}$. 

\vskip 2mm 

(ii) If $|(\nabla^2 f_i(x))_{jl}| \leq \gamma$ for any $x\in \R^d$, $i \in [n]$, and $j, l\in [d]$, then from (\ref{eq:hmAsym}), every entry of ${\tilde h}^i(\nabla^2 f_i(x))$ will be bounded by $2\gamma \|({\tilde {\cal B}}_i)^{-1}\|_{\infty}$, i.e., $\max_{jl}\{ |{\tilde h}^i(\nabla^2 f_i(x))_{jl}|\} \leq 2\gamma \|({\tilde {\cal B}}_i)^{-1}\|_{\infty}$. In particular, under Assumption \ref{as:Cunbiasedcomp-BL1} (ii), $(\mL_i^k)_{jl}$ is a convex combination of $\{ {\tilde h}^i(\nabla^2 f_i(z_i^t))_{jl} \}_{t\leq k}$, and thus $M_3$ in Assumption \ref{as:BL3} satisfies $M_3 \leq 2\gamma \max_{i} \{\|({\tilde {\cal B}}_i)^{-1}\|_{\infty} \}$. 

\vskip 2mm 

(iii) First, from $\|\nabla^2 f_i(x)\|_{\rm F} \leq {\tilde \gamma}$ and (\ref{eq:hmAsym}), we have 
\begin{align*}
	\|{\tilde h}^i(\nabla^2 f_i(x))\|_{\rm F} & \leq \|svec({\tilde h}^i(\nabla^2 f_i(x)))\| \\ 
	& \leq \|({\tilde {\cal B}}_i)^{-1}\| \cdot \|svec(\nabla^2 f_i(x))\| \\ 
	& \leq \sqrt{2}  \|({\tilde {\cal B}}_i)^{-1}\| \cdot \|\nabla^2 f_i(x)\|_{\rm F} \\ 
	& \leq \sqrt{2}  \|({\tilde {\cal B}}_i)^{-1}\| {\tilde \gamma}, 
\end{align*}
for any $x\in \R^d$ and $i\in [n]$. Assume $\|\mL_i^K\|_{\rm F}^2 \leq \frac{2B}{A} \|({\tilde {\cal B}}_i)^{-1}\|^2 {\tilde \gamma}^2$. Then under Assumption \ref{as:Ccontractioncomp-BL1}, same as Lemma \ref{lm:Hk-BL3} (ii), we have 
\begin{align*}
	\|\mL_i^{K+1} \|_{\rm F}^2 & = \|\mL_i^K + {\cal C}_i^K\left(  {\tilde h}^i(\nabla^2 f_i(z_i^{K+1})) - \mL_i^K  \right)\|_{\rm F}^2 \\ 
	& \leq \left(  1 - \frac{\delta}{4}  \right) \|\mL_i^K\|_{\rm F}^2 + \left(  \frac{6}{\delta} - \frac{7}{2}  \right) \| {\tilde h}^i(\nabla^2 f_i(z_i^{K+1}))\|_{\rm F}^2 \\ 
	& = (1-A)  \|\mL_i^K\|_{\rm F}^2 + B  \| {\tilde h}^i(\nabla^2 f_i(z_i^{K+1}))\|_{\rm F}^2 \\ 
	& \leq (1-A) \cdot \frac{2B}{A} \|({\tilde {\cal B}}_i)^{-1}\|^2 {\tilde \gamma}^2 + B \cdot 2 \|({\tilde {\cal B}}_i)^{-1}\|^2 {\tilde \gamma}^2 \\ 
	& = \frac{2B}{A} \|({\tilde {\cal B}}_i)^{-1}\|^2 {\tilde \gamma}^2. 
\end{align*}
Since $\|\mL_i^0\|_{\rm F}^2 \leq \frac{2B}{A} \|({\tilde {\cal B}}_i)^{-1}\|^2 {\tilde \gamma}^2$, by mathematical induction, we can get $\|\mL_i^k\|_{\rm F}^2 \leq \frac{2B}{A} \|({\tilde {\cal B}}_i)^{-1}\|^2 {\tilde \gamma}^2$ for $k\geq 0$. At last, from $\max_{jl}\{  |(\mL_i^k)_{jl}|  \} \leq \|\mL_i^k\|_{\rm F} \leq \frac{\sqrt{2B} }{\sqrt{A}} \|({\tilde {\cal B}}_i)^{-1}\| {\tilde \gamma}$, we can obtain $M_3 \leq \frac{\sqrt{2B} {\tilde \gamma}}{\sqrt{A}} \max_{i} \{ \|({\tilde {\cal B}}_i)^{-1}\| \}$.

\subsection{Lemmas}

The proof of Lemma \ref{lm:Hk-BL3} is the same as that of Lemma B.1 in \citep{FedNL2021}. Hence we omit it. 

\begin{lemma}\label{lm:Hk-BL3}
	Let $\cC$ be a compressor and $\alpha>0$. For any matrix $\mL \in \R^{d\times d}$ and $y, z\in \R^d$, we have the following results. 
\begin{itemize}	
	\item[(i)] If $\cC$ is an unbiased compressor with parameter $\omega$ and $\alpha \leq \nicefrac{1}{\omega+1}$, then 
	$$
	\mathbb{E}\| \mL + \alpha \cC({\tilde h}^i(\nabla^2 f_i(y)) - \mL) - {\tilde h}^i(\nabla^2 f_i(z)) \|^2_{\rm F} \leq (1-\alpha) \| \mL - {\tilde h}^i(\nabla^2 f_i(z))\|^2_{\rm F} + \alpha M_4^2 \|y-z\|^2, 
	$$
	where $\mathbb{E}[\cdot]$ is the expectation with respect to $\cC$. 
	\item[(ii)] If $\cC$ is a contraction compressor with parameter $\delta$ and $\alpha=1$, then 
	$$
	\mathbb{E}\| \mL + \alpha \cC({\tilde h}^i(\nabla^2 f_i(y)) - \mL) - {\tilde h}^i(\nabla^2 f_i(z)) \|^2_{\rm F} \leq \left(  1 - \frac{\delta}{4}  \right) \| \mL - {\tilde h}^i(\nabla^2 f_i(z))\|^2_{\rm F} + \left(  \frac{6}{\delta} - \frac{7}{2}  \right) M_4^2 \|y-z\|^2. 
	$$
\end{itemize}		
\end{lemma}

The constants $c_1$ and $c_2$ in the following lemma are defined in (\ref{eq:c1c2}). 
\begin{lemma}\label{lm:Hkbound-BL3}
\begin{itemize}		We consider four cases:
	\item[(i)] If Assumption \ref{as:Qunbiasedcomp-BL1} (ii) holds, $\|x^0-x^*\|^2 \leq \min\{  \frac{\mu^2}{4d^2(H^2+4c_1)}, \frac{M}{d}  \}$,  $\|z_i^k-x^*\|^2 \leq \min\{  \frac{\mu^2}{4d(H^2+4c_1)}, M  \}$, and ${\cal H}^k \leq \frac{\mu^2}{4dc_2}$ for $k\leq K$, $i\in [n]$, and any $M>0$, then $\|z_i^{K+1}-x^*\|^2 \leq \min\{  \frac{\mu^2}{4d(H^2+4c_1)}, M  \}$ for $i\in [n]$.  
	\item[(ii)] If Assumption \ref{as:Qcontractioncomp-BL1} holds, ${\cal H}^K \leq \frac{A_{\rm M}\mu^2}{4c_2B_{\rm M}}$, $\|z_i^k-x^*\|^2 \leq \min\{  \frac{A_{\rm M}\mu^2}{4B_{\rm M}(H^2 + 4c_1)}, M  \}$ for $k\leq K$, $i\in [n]$, and any $M>0$, then $\|z_i^{K+1} - x^*\|^2 \leq  \min\{  \frac{A_{\rm M}\mu^2}{4B_{\rm M}(H^2+4c_1)}, M  \}$ for $i\in [n]$. 
	\item[(iii)] If Assumption \ref{as:Cunbiasedcomp-BL1}(ii) holds, and $\|z_i^k-x^*\|^2 \leq \frac{M}{d^2M_5^2} $ for $k\leq K$, $i\in [n]$, and any $M>0$, then ${\cal H}^{K} \leq M$. 
	\item[(iv)] If Assumption \ref{as:Ccontractioncomp-BL1} holds, $\|\mL_i^K - \mL_i^*\|_{\rm F}^2 \leq M$, and $\|z_i^K-x^*\|^2 \leq \frac{AM}{BM_4^2}$ for $i\in [n]$ and any $M>0$, then $\|\mL_i^{K+1} - \mL_i^*\|_{\rm F}^2 \leq M$ for $i\in[n]$. 
\end{itemize}		
\end{lemma}

\begin{proof}
	
	(i) First, from the update rule of $w_i^k$, we know ${\cal Z}^k \leq \min\{  \frac{\mu^2}{4d(H^2+4c_1)}, M  \}$ and ${\cal W}^k \leq \min\{  \frac{\mu^2}{4d(H^2+4c_1)}, M  \}$ for $k\leq K$. Then for Option 1, from (\ref{eq:xk+1-BL3}), we have 
	\begin{align*}
		\|x^{k+1} - x^*\|^2 & \leq \frac{H^2}{2\mu^2} ({\cal W}^k)^2 + \frac{2c_1}{\mu^2} {\cal Z}^{k-1}{\cal W}^k + \frac{2c_2}{\mu^2} {\cal H}^k {\cal W}^k \\ 
		& \leq \frac{1}{d} {\cal W}^k \\ 
		& \leq \min \left\{  \frac{\mu^2}{4d^2(H^2+4c_1)}, \frac{M}{d}  \right\}, 
	\end{align*}
	for $0\leq k\leq K$. For Option 2, we can get the same bound for $\|x^{k+1} - x^*\|^2$ as above from (\ref{eq:xk+1-BL3-2}). \\ 
	Since $\|x^0-x^*\|^2 \leq \min \{  \frac{\mu^2}{4d^2(H^2+4c_1)}, \frac{M}{d}  \}$, we know $\|x^k-x^*\|^2 \leq \min \{  \frac{\mu^2}{4d^2(H^2+4c_1)}, \frac{M}{d} \}$ for all $0\leq k\leq K+1$. Then for $i\in S^k$, from Assumption \ref{as:Qunbiasedcomp-BL1} (ii), we can get 
	\begin{align*}
		\|z_i^{K+1} - x^*\|^2 & \leq d \max_{j} |(z^{K+1}_i)_j - x^*_j|^2 \\ 
		& \leq d \max_{0\leq t \leq K+1} \|x^t - x^*\|^2 \\ 
		& \leq \min \left\{  \frac{\mu^2}{4d(H^2+4c_1)},  M  \right\}. 
	\end{align*}
	For $i \notin S^k$, we have 
	$$
	\|z_i^{K+1} - x^*\|^2 = \|z_i^{K} - x^*\|^2 \leq \min \left\{  \frac{\mu^2}{4d(H^2+4c_1)},  M  \right\}. 
	$$
	
	(ii) First, from the update rule of $w^k$, we know ${\cal Z}^k \leq \min\{  \frac{A_{\rm M}\mu^2}{4B_{\rm M}(H^2 + 4c_1)}, M  \}$ and ${\cal W}^k \leq \min\{  \frac{A_{\rm M}\mu^2}{4B_{\rm M}(H^2 + 4c_1)}, M  \}$ for $k\leq K$. Then for Option 1, from (\ref{eq:xk+1-BL3}), we have 
	\begin{align*}
		\|x^{K+1} - x^*\|^2 & \leq \frac{H^2}{2\mu^2} ({\cal W}^K)^2 + \frac{2c_1}{\mu^2} {\cal Z}^{K-1}{\cal W}^K + \frac{2c_2}{\mu^2} {\cal H}^K {\cal W}^K \\ 
		& \leq \frac{A_{\rm M}}{B_{\rm M}} {\cal W}^K \\ 
		& \leq \frac{A_{\rm M}}{B_{\rm M}} \min \left\{  \frac{A_{\rm M}\mu^2}{4B_{\rm M}(H^2 + 4c_1)}, M  \right\} . 
	\end{align*}
	For Option 2, we can get the same bound for $\|x^{K+1} - x^*\|^2$ as above from (\ref{eq:xk+1-BL3-2}).
	Then for $i\in S^k$, from Lemma \ref{lm:zQcomp} (ii), we arrive at 
	\begin{align*}
		\|z_i^{K+1} - z^*\|^2 & \leq (1-A_{\rm M}) \|z_i^K-x^*\|^2 + B_{\rm M}\|x^{K+1} - x^*\|^2 \\ 
		& \leq (1-A_{\rm M}) \min \left\{  \frac{A_{\rm M}\mu^2}{4B_{\rm M}(H^2 + 4c_1)}, M  \right\} + A_{\rm M} \min \left\{  \frac{A_{\rm M}\mu^2}{4B_{\rm M}(H^2 + 4c_1)}, M  \right\} \\ 
		& =  \min \left\{  \frac{A_{\rm M}\mu^2}{4B_{\rm M}(H^2 + 4c_1)}, M  \right\}. 
	\end{align*}
	For $i \notin S^k$, we have 
	$$
	\|z_i^{K+1} - x^*\|^2 = \|z_i^{K} - x^*\|^2 \leq \min \left\{  \frac{A_{\rm M}\mu^2}{4B_{\rm M}(H^2 + 4c_1)}, M  \right\}. 
	$$
	
	(iii) From Assumption \ref{as:Cunbiasedcomp-BL1}(ii), we have 
	\begin{align*}
		{\cal H}^K & = \frac{1}{n} \sum_{i=1}^n \|\mL_i^K - \mL_i^*\|^2_{\rm F} \\ 
		& \leq  \frac{1}{n} \sum_{i=1}^n d^2 \max_{jl} \{ |(\mL_i^K)_{jl} - (\mL_i^*)_{jl}|^2 \} \\ 
		& \leq d^2 M_5^2 \max_{i\in [n], 0\leq t \leq K} \|z_i^t-x^*\|^2 \\ 
		& \leq M. 
	\end{align*}
	
	(iv) For $i\in S^k$, from Assumption \ref{as:Ccontractioncomp-BL1} and Lemma \ref{lm:Hk-BL3} (ii), we have 
	\begin{align*}
		\|\mL_i^{K+1} - \mL_i^*\|^2_{\rm F} & \leq (1-A) \|\mL_i^K - \mL_i^*\|^2_{\rm F} + BM_4^2 \|z_i^K - x^*\|^2 \\ 
		& \leq  (1-A) M + AM\\ 
		& = M. 
	\end{align*}
	For $i \notin S^k$, we also have 
	$$
	\|\mL_i^{K+1} - \mL_i^*\|^2_{\rm F}  =  \|\mL_i^K - \mL_i^*\|^2_{\rm F} \leq M. 
	$$
\end{proof}

\subsection{Proof of Theorem \ref{th:BL3}}

Define $\mH_i^k \eqdef \beta^k \mA_i^k - \mC_i^k$ for $i\in[n]$ and $k\geq 0$. First, it is easy to verify that $\mA^k = \frac{1}{n} \sum_{i=1}^n \mA_i^k$, $\mC^k = \frac{1}{n} \sum_{i=1}^n \mC_i^k$, $\mH^k = \frac{1}{n} \sum_{i=1}^n \mH_i^k$, $g_1^k = \frac{1}{n} \sum_{i=1}^n g_{i,1}^k$, and $g_2^k = \frac{1}{n} \sum_{i=1}^n g_{i,2}^k$ for $k\geq 0$. Then we have 
\begin{align*}
	g^k & = \beta^k g_1^k - g_2^k \\ 
	& = \frac{1}{n} \sum_{i=1}^n \left(  \beta^k g_{i,1}^k - g_{i,2}^k  \right) \\ 
	& = \frac{1}{n} \sum_{i=1}^n \left(  \beta^k \mA_i^k w_i^k - \mC_i^k w_i^k - \nabla f_i(w_i^k) \right) \\ 
	& = \frac{1}{n} \sum_{i=1}^n \left(  \mH_i^k w_i^k - \nabla f_i(w_i^k)  \right). 
\end{align*}

Thus, from 
$$
x^{k+1} = \left(  \mH^k  \right)^{-1} g^k =  \left(  \mH^k  \right)^{-1} \left[  \frac{1}{n} \sum_{i=1}^n \left(  \mH^k_i w_i^k - \nabla f_i(w_i^k)   \right)  \right] , 
$$
and 
$$
x^* = \left(  \mH^k  \right)^{-1} \left[  \mH^k x^* - \nabla f(x^*)  \right] =  \left(  \mH^k  \right)^{-1} \left[  \frac{1}{n} \sum_{i=1}^n \left(  \mH^k_i x^* - \nabla f_i(x^*)   \right)  \right], 
$$
we can obtain 
$$
x^{k+1} - x^* =  \left(  \mH^k  \right)^{-1} \left[  \frac{1}{n} \sum_{i=1}^n \left(  \mH^k_i (w_i^k - x^*) - (\nabla f_i(w_i^k) - \nabla f_i(x^*) )  \right)  \right]. 
$$
Then from the triangle inequality and the fact that $\mH^k \succeq \mu \mI$, we have 
\begin{align*}
	\|x^{k+1} - x^*\| &\leq \frac{1}{\mu n} \sum_{i=1}^n \left\| \nabla f_i(w_i^k) - \nabla f_i(x^*) - \mH^k_i (w_i^k - x^*)    \right\| \\ 
	& \leq \frac{1}{\mu n} \sum_{i=1}^n \left\| \nabla f_i(w_i^k) - \nabla f_i(x^*) - \nabla^2 f_i(x^*) (w_i^k - x^*)  \right\| + \frac{1}{\mu n} \sum_{i=1}^n \left\| (\mH_i^k - \nabla^2 f_i(x^*) (w_i^k - x^*)) \right\| \\ 
	& \leq \frac{H}{2\mu n} \sum_{i=1}^n \|w_i^k - x^*\|^2 + \frac{1}{\mu n} \sum_{i=1}^n \|\mH_i^k - \nabla^2 f_i(x^*)\| \cdot \|w_i^k - x^*\| \\ 
	& = \frac{H}{2\mu} {\cal W}^k + \frac{1}{\mu n} \sum_{i=1}^n \|\mH_i^k - \nabla^2 f_i(x^*)\| \cdot \|w_i^k - x^*\|. 
\end{align*}

We further use Young's inequality to bound $\|x^{k+1} - x^*\|^2$ as 
\begin{align*}
	\|x^{k+1} - x^*\|^2 & \leq \frac{H^2}{2\mu^2} ({\cal W}^k)^2 + \frac{2}{\mu^2 n^2} \left(   \sum_{i=1}^n \|\mH_i^k - \nabla^2 f_i(x^*)\| \cdot \|w_i^k - x^*\|  \right)^2 \\ 
	& \leq \frac{H^2}{2\mu^2} ({\cal W}^k)^2 + \frac{2}{\mu^2} \left(  \frac{1}{n} \sum_{i=1}^n \|\mH_i^k - \nabla^2 f_i(x^*)\|^2  \right) {\cal W}^k, 
\end{align*}
where we use Cauchy-Schwarz inequality in the last inequality. Next we estimate $\|\mH_i^k - \nabla^2 f_i(x^*)\|^2$. Denote $\mL_i^* = {\tilde h}^i(\nabla^2 f_i(x^*))$ and assume $\max_{jl} \{\|\mB_i^{jl}\|_{\rm F}\} \leq R$ for $i\in [n]$. Then 

\begin{align*}
	\|\mH_i^k - \nabla^2 f_i(x^*)\|^2 & = \| \beta^k \mA_i^k - \mC_i^k - \nabla^2 f_i(x^*)\|^2 \\ 
	& = \left\| \sum_{jl} \left[  \beta^k ((\mL_i^k)_{jl} + 2\gamma_i^k) - 2\gamma_i^k - (\mL_i^*)_{jl}  \right] \mB^{jl} \right\|^2 \\ 
	& \leq N R^2 \sum_{jl} \left|  \beta^k((\mL_i^k)_{jl} + 2\gamma_i^k) - ((\mL_i^*)_{jl} + 2\gamma_i^k)  \right|^2. 
\end{align*}

Assuming $\max_{jl} \{|(\mL_i^k)_{jl}|\} \leq M_3$ for $i\in [n]$, we have 
\begin{align*}
	\left|  \beta^k((\mL_i^k)_{jl} + 2\gamma_i^k) - ((\mL_i^*)_{jl} + 2\gamma_i^k)  \right|^2 & = \left| ( \beta^k - 1)((\mL_i^k)_{jl} + 2\gamma_i^k) + (\mL_i^k - \mL_i^*)_{jl}  \right|^2 \\ 
	& \leq 2 \left|  (\beta^k-1) ((\mL_i^k)_{jl} + 2\gamma_i^k)  \right|^2 + 2\left|  (\mL_i^k - \mL_i^*)_{jl}  \right|^2 \\ 
	& \leq 2(M_3 + 2\max\{c, M_3\})^2 \left|  \beta^k - 1  \right|^2 + 2\left|  (\mL_i^k - \mL_i^*)_{jl}  \right|^2. 
\end{align*}

For Option 1 in Algorithm~\ref{alg:BL3}, we have $\beta_i^{{k}} = \max_{jl} \frac{{\tilde h}^i(\nabla^2 f_i(z_i^{k-1}))_{jl} + 2 \gamma_i^{k}}{(\mL_i^{k})_{jl} + 2 \gamma_i^{k} }$, where we define $z_i^{-1} = z_i^0$. For any $j,l\in [d]$, we have 
\begin{align*}
	\left|   \frac{{\tilde h}^i (\nabla^2 f_i(z_i^{k-1}))_{jl} + 2 \gamma_i^{k}}{(\mL_i^{k})_{jl} + 2 \gamma_i^{k} } - 1  \right|^2 &= \left|  \frac{({\tilde h}^i(\nabla^2 f_i(z_i^{k-1})) - \mL_i^k)_{jl}}{(\mL_i^k + 2\gamma_i^k \mI)_{jl}}  \right|^2 \\
	& \leq \frac{1}{c^2} \left| ({\tilde h}^i(\nabla^2 f_i(z_i^{k-1})) - \mL_i^k)_{jl}   \right|^2 \\ 
	& \leq \frac{2}{c^2} \left| ({\tilde h}^i(\nabla^2 f_i(z_i^{k-1})) - \mL_i^*)_{jl}  \right|^2 + \frac{2}{c^2}\left| (\mL_i^k - \mL_i^*)_{jl}   \right|^2 \\ 
	& \leq \frac{2M_5^2}{c^2} \|z_i^{k-1} - x^*\|^2 + \frac{2}{c^2} \left| (\mL_i^k - \mL_i^*)_{jl}   \right|^2, 
\end{align*}
where we use $(\mL_i^k + 2\gamma_i^k \mI)_{jl} \geq (\mL_i^k)_{jl} + |(\mL_i^k)_{jl}| + c \geq c$ in the first inequality, in the second inequality, we use the Young's inequality, and the last inequality comes from $\max_{jl} \{|({\tilde h}^i(\nabla^2 f_i(x)) - \mL_i^*)_{jl} | \} \leq M_5 \|x-x^*\|$. Then from the definition of $\beta^k$, we arrive at 
\begin{align*}
	\left|  \beta^k - 1  \right|^2 &\leq \max_{j l}\left\{  \frac{2M_5^2}{c^2} \|z_i^{k-1} - x^*\|^2 + \frac{2}{c^2} \left| (\mL_i^k - \mL_i^*)_{jl}   \right|^2  \right\} \\ 
	& \leq \frac{2M_5^2}{c^2} \|z_i^{k-1} - x^*\|^2 + \frac{2}{c^2} \| \mL_i^k - \mL_i^* \|^2. 
\end{align*}

For Option 2 in Algorithm~\ref{alg:BL3}, we can have the following bound in the same way. 
$$
\left|  \beta^k - 1  \right|^2 \leq \frac{2M_5^2}{c^2} \|z_i^{k} - x^*\|^2 + \frac{2}{c^2} \| \mL_i^k - \mL_i^* \|^2. 
$$

For Option 1, from the above inequalities, we can get 
\begin{align*}
	\|\mH_i^k - \nabla^2 f_i(x^*)\|^2 &\leq \frac{4N^2 R^2 M_5^2 (M_3 + 2\max\{c, M_3\})^2}{c^2} \|z_i^{k-1} - x^*\|^2 \\ 
	& \quad + 2NR^2 \left(  1 + \frac{2N(M_3 + 2\max\{c, M_3\})^2}{c^2}  \right) \|\mL_i^k - \mL_i^*\|_{\rm F}^2, 
\end{align*}
which implies that 
\begin{align*}
	\frac{1}{n} \sum_{i=1}^n \|\mH_i^k - \nabla^2 f_i(x^*)\|^2 &\leq \frac{4N^2 R^2 M_5^2 (M_3 + 2\max\{c, M_3\})^2}{c^2} {\cal Z}^{k-1} \\ 
	& \quad + 2NR^2 \left(  1 + \frac{2N(M_3 + 2\max\{c, M_3\})^2}{c^2}  \right) {\cal H}^k, 
\end{align*}
where ${\cal H}^k \eqdef \frac{1}{n}\sum_{i=1}^n \|\mL_i^k - \mL_i^*\|_{\rm F}^2$. 
Similarly, for Option 2, we can get 
\begin{align*}
	\frac{1}{n} \sum_{i=1}^n \|\mH_i^k - \nabla^2 f_i(x^*)\|^2 &\leq \frac{4N^2 R^2 M_5^2 (M_3 + 2\max\{c, M_3\})^2}{c^2} {\cal Z}^{k} \\ 
	& \quad + {2NR^2} \left(  1 + \frac{2N(M_3 + 2\max\{c, M_3\})^2}{c^2}  \right) {\cal H}^k. 
\end{align*}

Let 
\begin{equation}\label{eq:c1c2}
	c_1 \eqdef \frac{4N^2 R^2 M_5^2 (M_3 + 2\max\{c, M_3\})^2}{c^2}, \quad c_2 \eqdef {2NR^2} \left(  1 + \frac{2N(M_3 + 2\max\{c, M_3\})^2}{c^2}  \right)
\end{equation}

Then for Option 1, we have 
\begin{align}
	\|x^{k+1} - x^*\|^2 & \leq \frac{H^2}{2\mu^2} ({\cal W}^k)^2 + \frac{2}{\mu^2} {\cal W}^k \left(  c_1 {\cal Z}^{k-1} + c_2 {\cal H}^k  \right)  \nonumber \\ 
	& = \frac{H^2}{2\mu^2} ({\cal W}^k)^2 + \frac{2c_1}{\mu^2} {\cal Z}^{k-1}{\cal W}^k + \frac{2c_2}{\mu^2} {\cal H}^k {\cal W}^k, \label{eq:xk+1-BL3}
\end{align}
and for Option 2, we have 
\begin{align}
	\|x^{k+1} - x^*\|^2 & \leq \frac{H^2}{2\mu^2} ({\cal W}^k)^2 + \frac{2c_1}{\mu^2} {\cal Z}^{k}{\cal W}^k + \frac{2c_2}{\mu^2} {\cal H}^k {\cal W}^k. \label{eq:xk+1-BL3-2}
\end{align}

From the update rule of $w_i^k$ and $z_i^k$, the results in (\ref{eq:Zk+1-BL2}) and (\ref{eq:Wk+1-BL2}) also hold for Algorithm~\ref{alg:BL3}. Then for Option 1, we have 
\begin{align*}
	\mathbb{E}_k[\Phi_5^{k+1}] & \overset{(\ref{eq:Wk+1-BL2})}{\leq}  \left(  1 - \frac{\tau p}{n}  \right) {\cal W}^k + \frac{\tau p}{n} \left(  1 - \frac{\tau A_{\rm M}}{n}  \right) {\cal Z}^k + \frac{\tau^2B_{\rm M}p}{n^2} \|x^{k+1}-x^*\|^2 +  \frac{2p}{A_{\rm M}}\left(  1 - \frac{\tau A_{\rm M}}{n}  \right) \mathbb{E}_k[{\cal Z}^{k+1}] \\ 
	& \overset{(\ref{eq:Zk+1-BL2})}{\leq} \left(  1 - \frac{\tau p}{n}  \right) {\cal W}^k + \frac{2p}{A_{\rm M}}\left(  1 - \frac{\tau A_{\rm M}}{n}  \right) \left(  1 - \frac{\tau A_{\rm M}}{2n}  \right) {\cal Z}^k + \frac{2\tau p B_{\rm M}}{n A_{\rm M}} \left(  1 - \frac{\tau A_{\rm M}}{2n}  \right) \|x^{k+1}-x^*\|^2 \\ 
	& \overset{\ref{eq:xk+1-BL3}}{\leq} \left(  1 - \frac{\tau p}{n} + \frac{2\tau pB_{\rm M}}{nA_{\rm M}} \left(  \frac{H^2}{2\mu^2}{\cal W}^k  + \frac{2c_2}{\mu^2} {\cal H}^k  +  \frac{2c_1}{\mu^2} {\cal Z}^{k-1}   \right)  \right) {\cal W}^k \\ 
	& \quad + \frac{2p}{A_{\rm M}}\left(  1 - \frac{\tau A_{\rm M}}{n}  \right) \left(  1 - \frac{\tau A_{\rm M}}{2n}  \right) {\cal Z}^k. 
\end{align*}

If $\|z_i^k-x^*\|^2 \leq \frac{A_{\rm M}\mu^2}{4(H^2 + 4c_1)B_{\rm M}}$ and ${\cal H}^k \leq \frac{A_{\rm M}\mu^2}{16c_2 B_{\rm M}}$ for all $k\geq 0$, then we have 
$$
\frac{H^2}{2\mu^2}{\cal W}^k  + \frac{2c_2}{\mu^2} {\cal H}^k  +  \frac{2c_1}{\mu^2} {\cal Z}^{k-1}  \leq \frac{A_{\rm M}}{4B_{\rm M}}, 
$$
which implies that 
\begin{align*}
	\mathbb{E}_k[\Phi_5^{k+1}] & \leq \left(  1 - \frac{\tau p}{2n}  \right) {\cal W}^k + \frac{2p}{A_{\rm M}}\left(  1 - \frac{\tau A_{\rm M}}{n}  \right) \left(  1 - \frac{\tau A_{\rm M}}{2n}  \right) {\cal Z}^k \\ 
	& \leq \left(  1 - \frac{\tau \min\{  p, A_{\rm M}  \}}{2n}  \right) \Phi_5^k. 
\end{align*}

By applying the tower property, we have 
$$
\mathbb{E} [\Phi_5^{k+1}] \leq \left(  1 - \frac{\tau \min\{  p, A_{\rm M}  \}}{2n}  \right) \mathbb{E}[\Phi_5^k]. 
$$

Unrolling the recursion, we can obtain the result. For Option 2, we can have the same result.

\subsection{Proof of Theorem \ref{th:superlinear-BL3}}

Since $\xi^k \equiv 1$, $\eta=1$, $S^k \equiv [n]$, and ${\cal Q}_i^k(x) \equiv x$ for any $x\in \R^d$, it is easy to see that $w_i^k = z_i^k \equiv x^k$ for all $i\in[n]$ and $k\geq 0$. In this case, we can view ${\cal Q}_i^k$ as an unbiased compressor with $\omega_{\rm M}=0$ or a contraction compressor with $\delta_{\rm M}=1$. Since (\ref{eq:Zk+1-BL2}) also holds for Algorithm~\ref{alg:BL3}, for Option 1, we have 
\begin{eqnarray*}
	\mathbb{E}_k \|x^{k+1}-x^*\|^2 & \leq & \left(  1 - A_{\rm M}  \right) \|x^k-x^*\|^2 + B_{\rm M} \|x^{k+1}-x^*\|^2 \\ 
	& \overset{(\ref{eq:xk+1-BL3})}{\leq} & \left(  1 - A_{\rm M}  \right) \|x^k-x^*\|^2 + \frac{1}{4}A_{\rm M} \|x^k-x^*\|^2 \\ 
	& =& \left(  1 - \frac{3A_{\rm M}}{4}  \right) \|x^k-x^*\|^2. 
\end{eqnarray*}
For Option 2, we can get the same bound for $\mathbb{E}_k\|x^{k+1} - x^*\|^2$ as above from (\ref{eq:xk+1-BL3-2}).

From Lemma \ref{lm:Hk-BL3}, we can obtain 
$$
\mathbb{E}_k[{\cal H}^{k+1}] \leq  (1-A) {\cal H}^k + BM_4^2 \|x^k-x^*\|^2. 
$$

Thus, 
\begin{align*}
	\mathbb{E}_k[\Phi_6^{k+1}] & = \mathbb{E}_k[{\cal H}^{k+1}] + \frac{4BM_4^2}{A_{\rm M}}\mathbb{E}_k\|x^{k+1} - x^*\|^2 \\ 
	& \leq (1-A) {\cal H}^k + BM_4^2 \|x^k-x^*\|^2 +  \frac{4BM_4^2}{A_{\rm M}}\left(  1 - \frac{3A_{\rm M}}{4}  \right) \|x^k-x^*\|^2 \\ 
	& \leq \left(  1 - \frac{\min\{  2A, A_{\rm M}  \} }{2} \right) \Phi_6^k. 
\end{align*}

By applying the tower property, we have $\mathbb{E}[\Phi_6^{k+1}] \leq \theta_2 \mathbb{E}[\Phi_6^k]$. Unrolling the recursion, we have $\mathbb{E}[\Phi_6^k] \leq \theta_3^k \Phi_6^0$. 

Then we further have $\mathbb{E}[{\cal H}^k] \leq \theta_3^k \Phi_6^0$ and $\mathbb{E}\|x^k-x^*\|^2 \leq \frac{A_{\rm M}}{4BM_4^2} \theta_3^k \Phi_6^0$. For Option 1, from $w_i^k=z_i^k\equiv x^k$ and (\ref{eq:xk+1-BL3}), we can get 
$$
\|x^{k+1} - x^*\|^2 \leq  \left(  \frac{H^2}{2\mu^2} \|x^k-x^*\|^2 + \frac{2c_1}{\mu^2} \|x^{k-1}-x^*\|^2 + \frac{2c_2}{\mu^2} {\cal H}^k  \right) \|x^k-x^*\|^2. 
$$
Assume $x^k\neq x^*$ for all $k\geq 0$. Then we have 
$$
\frac{\|x^{k+1}-x^*\|^2}{\|x^k-x^*\|^2} \leq \frac{H^2}{2\mu^2} \|x^k-x^*\|^2 + \frac{2c_1}{\mu^2} \|x^{k-1}-x^*\|^2 + \frac{2c_2}{\mu^2} {\cal H}^k,  
$$
and by taking expectation, we arrive at 
\begin{align*}
	\mathbb{E} \left[   \frac{\|x^{k+1}-x^*\|^2}{\|x^k-x^*\|^2} \right] & \leq \frac{H^2}{2\mu^2} \mathbb{E}\|x^k-x^*\|^2 + \frac{2c_1}{\mu^2} \mathbb{E}\|x^{k-1}-x^*\|^2 + \frac{2c_2}{\mu^2} \mathbb{E}[{\cal H}^k] \\ 
	& \leq \theta_3^k \left(  \frac{A_{\rm M}(H^2 \theta_3+4c_1)}{8BM_4^2\mu^2 \theta_3} + \frac{2c_2}{\mu^2}  \right) \Phi_6^0. 
\end{align*}

\noindent For Option 2, from $w_i^k=z_i^k\equiv x^k$ and (\ref{eq:xk+1-BL3-2}), we can get 
$$
\|x^{k+1} - x^*\|^2 \leq  \left(  \frac{H^2+4c_1}{2\mu^2} \|x^k-x^*\|^2 + \frac{2c_2}{\mu^2} {\cal H}^k  \right) \|x^k-x^*\|^2. 
$$
Assume $x^k\neq x^*$ for all $k\geq 0$. Then we have 
$$
\frac{\|x^{k+1}-x^*\|^2}{\|x^k-x^*\|^2} \leq   \frac{H^2+4c_1}{2\mu^2} \|x^k-x^*\|^2 + \frac{2c_2}{\mu^2} {\cal H}^k, 
$$
and by taking expectation, we arrive at 
\begin{align*}
	\mathbb{E} \left[   \frac{\|x^{k+1}-x^*\|^2}{\|x^k-x^*\|^2} \right] & \leq  \frac{H^2+4c_1}{2\mu^2} \mathbb{E}\|x^k-x^*\|^2 + \frac{2c_2}{\mu^2} \mathbb{E}[{\cal H}^k] \\ 
	& \leq \theta_3^k \left(  \frac{A_{\rm M}(H^2 + 4c_1)}{8BM_4^2\mu^2} + \frac{2c_2}{\mu^2}  \right) \Phi_6^0. 
\end{align*}

\subsection{Proof of Theorem \ref{th:nbor-BL3}}

(i) Notice that under Assumption \ref{as:Qunbiasedcomp-BL1}, we have $A_{\rm M} = B_{\rm M} = \eta$. We prove this by mathematical induction. First, since $z_i^0=x^0$, we know $\|z_i^0-x^*\|^2 \leq \min\{ \frac{\mu^2}{4d(H^2+4c_1)},  \frac{\mu^2}{16d^3 c_2M_5^2}  \}$ for $i\in [n]$. Then from Lemma \ref{lm:Hkbound-BL3} (iii), we have ${\cal H}^0 \leq \frac{\mu^2}{16dc_2 }$. Next, assume 
$$
\|z_i^k-x^*\|^2 \leq \min\left\{ \frac{\mu^2}{4d(H^2+4c_1)},  \frac{\mu^2}{16d^3 c_2M_5^2}  \right\} \ {\rm for} \ i\in[n] \quad  {\rm and} \quad  {\cal H}^k \leq \frac{\mu^2}{16dc_2 }, 
$$
for $k\leq K$. 
By choosing $M = \frac{\mu^2}{16d^3 c_2M_5^2}$ in Lemma \ref{lm:Hkbound-BL3} (i), we have $$\|z_i^{K+1}-x^*\|^2 \leq \min\left\{ \frac{\mu^2}{4d(H^2+4c_1)},  \frac{\mu^2}{16d^3 c_2M_5^2}  \right\}, $$ for $i\in[n]$. By further using Lemma \ref{lm:Hkbound-BL3} (iii), we can get ${\cal H}^{K+1} \leq \frac{\mu^2}{16dc_2 }$. \\ 

(ii) We prove the result by induction. Assume $\|z_i^k - x^*\|^2 \leq \min\left\{  \frac{A_{\rm M}\mu^2}{4B_{\rm M}(H^2+4c_1)},  \frac{AA_{\rm M}\mu^2}{16c_2 B_{\rm M}BM_4^2}  \right\}$ and $\|\mL_i^k - \mL_i^*\|^2_{\rm F} \leq \frac{A_{\rm M}\mu^2}{16c_2B_{\rm M}}$ for all $i\in[n]$ and $k\leq K$. Then by Lemma \ref{lm:Hkbound-BL3} (iv), we have $\|\mL_i^{K+1} - \mL_i^*\|^2_{\rm F} \leq \frac{A_{\rm M}\mu^2}{16c_2B_{\rm M}}$. Moreover, by Lemma \ref{lm:Hkbound-BL3} (ii), we have $\|z_i^{K+1} - x^*\|^2 \leq \min\left\{  \frac{A_{\rm M}\mu^2}{4B_{\rm M}(H^2+4c_1)},  \frac{AA_{\rm M}\mu^2}{16c_2 B_{\rm M}BM_4^2}  \right\}$ for $i\in[n]$.

\end{document}